\documentclass{article} 

\newcommand{\chart}{\phi}

\usepackage[main, nonatbib, final]{neurips_2026}

\usepackage[
  backend=biber,
  style=numeric-comp,
  sorting=none,
  giveninits=true,
  isbn=false,
  url=false,
  doi=false,
  eprint=false,
    defernumbers=true
]{biblatex}

\let\cite\autocite

\DeclareNameAlias{sortname}{family-given}
\DeclareNameAlias{default}{family-given}

\DeclareFieldFormat[article]{title}{#1}
\DeclareFieldFormat[inproceedings]{title}{#1}

\DeclareFieldFormat[article]{journaltitle}{\uline{#1}}
\DeclareFieldFormat[inproceedings]{booktitle}{\uline{#1}}

\DeclareFieldFormat[article]{volume}{\textbf{#1}}
\DeclareFieldFormat[article]{number}{(#1)}
\DeclareFieldFormat[article]{pages}{#1}

\newbibmacro*{author+date}{%
  \usebibmacro{author/translator+others}%
  \setunit{\addspace}%
  \printfield{year}%
  \newunit\newblock
}

\renewbibmacro*{journal+issuetitle}{%
  \printfield{journaltitle}%
  \setunit*{\addspace}%
  \printfield{volume}%
  \printfield{number}%
}

\DeclareBibliographyDriver{article}{%
  \usebibmacro{bibindex}%
  \usebibmacro{begentry}%
  \usebibmacro{author+date}%
  \setunit{\labelnamepunct}\newblock
  \usebibmacro{title}%
  \newunit
  \usebibmacro{journal+issuetitle}%
  \newunit
  \printfield{pages}%
  \newunit
  \usebibmacro{finentry}%
}

\DeclareBibliographyDriver{inproceedings}{%
  \usebibmacro{bibindex}%
  \usebibmacro{begentry}%
  \usebibmacro{author+date}%
  \setunit{\labelnamepunct}\newblock
  \usebibmacro{title}%
  \newunit
  \printfield{booktitle}%
  \newunit
  \printfield{eventtitle}%
  \newunit
  \usebibmacro{byeditor+others}%
  \newunit
  \printfield{volume}%
  \newunit
  \printfield{series}%
  \newunit
  \printfield{pages}%
  \newunit
  \usebibmacro{finentry}%
}

\DeclareFieldFormat{labelnumberwidth}{#1.}
\DeclareFieldFormat{labelnumber}{#1}

\defbibenvironment{bibliography}
  {\list
     {\printfield[labelnumberwidth]{labelnumber}}
     {\setlength{\labelwidth}{\labelnumberwidth}%
      \setlength{\leftmargin}{\labelwidth}%
      \setlength{\labelsep}{0.1em}%
      \addtolength{\leftmargin}{\labelsep}%
      \setlength{\itemsep}{\bibitemsep}%
      \setlength{\parsep}{0pt}%
      \setlength{\itemindent}{0pt}%
      \setlength{\listparindent}{0pt}%
      \addtolength{\leftskip}{-1.0em}}%
   }
  {\endlist}
  {\item}

\defbibfilter{supprefs}{
  segment=2 and not segment=1
}

\usepackage[colorlinks=true, linkcolor=blue, citecolor=blue, urlcolor=black]{hyperref}
\usepackage{url}
\usepackage{lipsum}
\usepackage{dsfont}
\usepackage{xcolor} 

\definecolor{darkBlue}{RGB}{40, 80, 150}

\usepackage{wrapfig}
\usepackage{sidecap}
\usepackage[export]{adjustbox}

\usepackage{microtype}
\usepackage{graphicx}
\usepackage{subfigure}
\usepackage{booktabs} 

\usepackage{hyperref}

\usepackage{amsmath}
\usepackage{amssymb}
\usepackage{mathtools}
\usepackage{amsthm}

\usepackage[capitalize,noabbrev]{cleveref}

\usepackage[textsize=tiny]{todonotes}

\usepackage{stfloats}

\usepackage{amsthm}
\usepackage{amsmath}
\usepackage{extpfeil}
\usepackage{enumitem}
\usepackage{soul}

\usepackage{algorithm}
\usepackage{algpseudocode}

\usepackage[normalem]{ulem}

\newtheoremstyle{iclrtheorem}
  {3pt}{3pt}               
  {\itshape}{}             
  {}               
  {.}                      
  {.5em}                   
  {\underline{\bf\thmname{#1}\thmnumber{ #2}}\thmnote{ (#3)}}

\newtheoremstyle{iclrdefinition}
  {3pt}{3pt}
  {\normalfont}{} 
  {}
  {.}
  {.5em}
  {\underline{\bf\thmname{#1}\thmnumber{ #2}}\thmnote{ (#3)}}
  
\theoremstyle{iclrtheorem}
\newtheorem{theorem}{Theorem}[section]
\newtheorem{lemma}[theorem]{Lemma}
\newtheorem{proposition}[theorem]{Proposition}
\newtheorem{corollary}[theorem]{Corollary}

\theoremstyle{iclrdefinition}
\newtheorem{definition}[theorem]{Definition}

\newtheorem{remark}[theorem]{Remark}

\newcommand{\PROOF}{{\small\normalfont [Proof]}}

\newcommand{\sr}{\mathrm{SR}}

\newcommand{\msa}{\mathrm{MSA}}

\allowdisplaybreaks

\newcommand{\parsec}[1]{{{\normalfont\bfseries #1}}}

\usepackage{xcolor}

\newif\ifshowcorrections
 \showcorrectionsfalse    

\newcommand{\rd}[1]{%
  \ifshowcorrections
    {\color{red}#1}%
  \else
    #1%
  \fi
}

\title{\fontsize{16.5}{20}\selectfont Geometry-Aware Similarity Metrics for Neural Representations on Riemannian and Statistical Manifolds}

\author{%
  N Alex Cayco-Gajic\thanks{These authors contributed equally. \rd{Author order is alphabetical.}}\\
  \rd{Département d'Etudes Cognitives}
    \And
  Arthur Pellegrino\footnotemark[1]\\
  \rd{Département d'Informatique}
  \vspace{-15pt}
  \AND
  \vspace{-20pt}
  ~\\Ecole Normale Supérieure, PSL University
}

\begin{document}

\begin{refsegment}

\maketitle

\setlength{\abovedisplayskip}{5pt}
\setlength{\belowdisplayskip}{5pt}
\setlength{\abovedisplayshortskip}{2pt}
\setlength{\belowdisplayshortskip}{2pt}

\setlength{\intextsep}{6pt}     
\setlength{\columnsep}{8pt}

\vspace{-10pt}

\vspace{-5pt}

\begin{abstract}
\vspace{-5pt}
Similarity measures \rd{such as centred kernel alignment (CKA) or representational similarity analysis (RSA)} are widely used to compare the representational geometries \rd{employed} by different neural networks to solve the same task. Yet, because existing methods compare the \emph{extrinsic} geometry of representations in state space, rather than their \emph{intrinsic} geometry, they may fail to capture subtle yet crucial distinctions between fundamentally different computations. Here, we introduce metric similarity analysis (MSA), a novel method which leverages tools from Riemannian geometry to compare the intrinsic geometry of neural representations under the manifold hypothesis. We show that MSA can be used to \emph{i)} disentangle neural computations of deep networks in rich vs. lazy learning regimes, \emph{ii)} compare the dynamics of recurrent neural networks and state-space models, and \emph{iii)} investigate the statistical manifolds of diffusion models. Overall, we introduce a mathematically grounded and broadly applicable framework to understand the mechanisms behind neural computations by contrasting their intrinsic geometries.
\end{abstract}

\vspace{-10pt}

\section{Introduction}

\rd{Growing} interest in mechanistic interpretability has inspired efforts to reverse-engineer neural networks by analysing \rd{the structure of} their \rd{hidden-layer} activations. Subtle variations in initialisation \cite{li2015convergent,wang2018towards}, architecture \cite{maheswaranathan2019universality}, or learning regime \cite{jacot2018neural,chizat2019lazy, saxe2019mathematical,paccolat2021geometric,woodworth2020kernel} can affect \rd{how neural networks represent data features in their hidden layers. Moreover, these} different representations can support distinct computation\rd{al solutions to a given task}, with varying generalisation performance\cite{woodworth2020kernel, johnston2023abstract, li2024representations, shang2025unraveling,choufeature}. \rd{Developing principled methods to compare representations stemming from different architectural and training choices is therefore crucial for understanding the mechanisms underlying neural computation.}

Towards this end, various similarity \rd{metrics} have been proposed to compare the representations of neural networks \cite{klabunde2025similarity}. Widely used methods include linear regression, Procrustes, centred kernel alignment (CKA), canonical correlation analysis (CCA), representational similarity analysis (RSA), and their variants \cite{kriegeskorte2008representational,raghu2017svcca,kornblith2019similarity,grave2019unsupervised,williams2021generalized,boix2022gulp,harvey2024representational,lin2024topology}. The core idea behind these methods is to define a distance measure that quantifies how similar or dissimilar \rd{the} neural activation patterns \rd{of two networks} are by comparing their geometries in the hidden-layer state space.

Yet, by focusing on state-space geometry, these methods do not explicitly leverage the ``manifold hypothesis'': the widely accepted view that input data lie on an intrinsically low-dimensional manifold \cite{tenenbaum2000global}. \rd{Under this hypothesis}, neural network computations can be understood in terms of how they transform a low-dimensional manifold of data inputs. For example, successive layers in a deep network progressively warp the input manifold to extract task-relevant features \cite{hauser2017principles}, \rd{and recent causal studies have linked representational warping to task performance via LLM steering \cite{wurgaft2026manifold}.}  Such warping can be studied formally using tools from Riemannian geometry, both in \rd{deep} neural networks \cite{poole2016exponential,hauser2017principles,kaul2019riemannian,brandon2025emergent}, \rd{where the input is warped across layers,} and in dynamical systems such as state-space and diffusion models \cite{pellegrino2025rnns,karczewski2025spacetime}, where the input manifold is instead warped through time. 

In this Riemannian view, the extrinsic geometry of a representation (e.g., its embedding) reflects less about internal computations than its \emph{intrinsic geometry}, which characterises \rd{pairwise relationships and distances} between points \emph{within} the manifold. For example, the same two-dimensional input manifold may be embedded in \rd{an} ambient \rd{hidden-layer} space as a plane or a Swiss roll (Fig.\,\ref{fig:intrinsic-extrinsic},\,\textit{left}).

\begin{figure}[H]
\noindent
\begin{minipage}[c]{0.4575\linewidth}
    \vspace{0pt}
    \caption{\textbf{\textit{Intrinsic vs. extrinsic similarity.}} \rd{The Euclidean plane and Swiss roll can have high intrinsic similarity (MSA) yet low extrinsic similarity (Procrustes), and conversely for a Euclidean vs. non-Euclidean plane. We argue that the intrinsic geometry of representations is more insightful for comparing computations.}} \label{fig:intrinsic-extrinsic}
\end{minipage}%
\hfill
\begin{minipage}[c]{0.52\linewidth}
        \includegraphics[width=\linewidth]{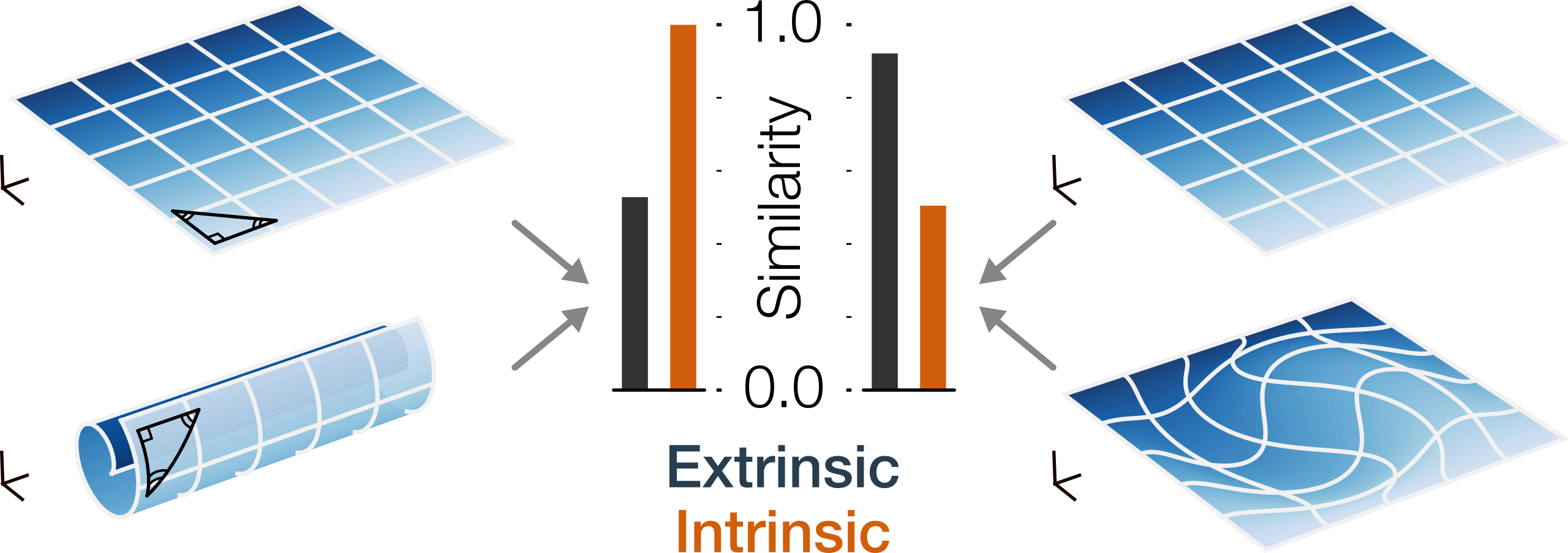}
\end{minipage}
\end{figure}

In this case, Procrustes or CKA will report high dissimilarity, despite all angles and distances \emph{along} the manifolds being identical. Conversely, there exist embeddings that are extrinsically very similar, but whose intrinsic geometries are mismatched (Fig.\,\ref{fig:intrinsic-extrinsic},\,\textit{right}). \rd{Importantly, the difference between intrinsic and extrinsic similarity can be arbitrarily large (App.\,\ref{app:arbitrarilylarge}).} In the following sections, we will show that neural networks can likewise exhibit substantial discrepancies between their intrinsic and extrinsic geometries. This illustrative example reveals the need for a principled Riemannian approach to characterise the similarity prof representations' intrinsic geometries.

To address this gap, \textbf{we introduce metric similarity analysis (MSA)}, a new framework to compare the intrinsic geometry of neural network computations through the lens of Riemannian geometry. We highlight MSA as a mathematically grounded method that satisfies the conditions to be a distance function defining a topological metric space over Riemannian metrics.

\ul{Our theoretical contributions are:} 
\begin{enumerate}[leftmargin=1.2em, topsep=0pt, itemsep=0pt, parsep=0pt]

\item \textbf{\textit{A Riemannian similarity measure.}} MSA leverages the pullback metric to define a similarity between the intrinsic geometries of Riemannian manifolds. We prove several mathematical results regarding MSA, including invariance to state-space rotations and choice of intrinsic coordinates.

\item \textbf{\textit{A novel psuedodistance between symmetric positive definite (SPD) matrices.}} As part of the development of MSA, we propose the \emph{spectral ratio}, a new distance on the SPD cone that is well suited for similarity analysis of Riemannian metrics.

\end{enumerate}

\ul{We apply this theory to:}
\begin{itemize}[leftmargin=1.2em, topsep=0pt, itemsep=0pt, parsep=0pt]
\item \textbf{\textit{Rich and lazy learning.}} We show that, unlike MSA, common methods report high similarity between rich and lazy deep networks, two learning regimes known to yield different computations.

\item \textbf{\textit{Dynamical systems.}} We train state-space models (SSMs) and recurrent neural networks (RNNs) on a working memory task, and show that MSA uncovers differences in their internal dynamics.

\item \textbf{\textit{Statistical manifolds.}} In large-scale text-to-image diffusion models, we show how MSA can be used to compare the information geometries of diffusion under different levels of guidance. 
\end{itemize}

Thus, we introduce a new framework anchored in Riemannian geometry, with which we show that identifying distinct neural computations can require uncovering minute differences in intrinsic representational geometry.

\vspace{-10pt}

\subsection*{Related work}

\vspace{-5pt}

\textbf{Distances between representations.} Similarity measures compare the hidden-layer activations of two networks $\varphi^{1}:\mathbb{R}^{n_\text{in}}\to\mathbb{R}^{n_1}$ and $\varphi^{2}:\mathbb{R}^{n_\text{in}}\to\mathbb{R}^{n_2}$ by defining a distance:
$d(\varphi^1, \varphi^2) \geq 0$
(or inversely, a similarity; see \cite{klabunde2025similarity} for a review). Methods such as CKA, CCA, RSA and Procrustes first sample the activations of each network given a finite set of $k$ inputs: $M_1=[\varphi^1(\mathbf{x}_1), ..., \varphi^1(\mathbf{x}_k)]$ and $M_2=[\varphi^2(\mathbf{x}_1), ..., \varphi^2(\mathbf{x}_k)]$ for $\mathbf{x}_i\in\mathbb{R}^{n_\text{in}}$.  
For example, Procrustes finds the minimum distance over rotations of one point cloud to another: $d_{\text{Proc}}(M_1,M_2)=\text{min}_{Q\in O(n)}||M_1-QM_2||_F^2$. 
In this work, we instead directly define a distance between the functions using the pullback metrics of $\varphi^1$ and $\varphi^2$. \rd{We further prove that nonlinear variants of CKA and RSA \cite{williams2024equivalence,lin2024topology} can be formulated in terms of the pullback metric (App. \ref{app:kcka}), establishing a new mathematical relationship between these methods and MSA. However, unlike MSA, these methods do not respect the intrinsic geometry of Riemannian metrics and therefore fail to distinguish important geometric differences between pullback metrics. }

\textbf{Distances between nonlinear dynamics.} 
\rd{A number of methods have recently been developed which extend beyond static geometry to compare nonlinear dynamics\cite{chen2024dform,sagodi2025dynamical,ostrow2023beyond,zhangkoopstd,huang2025inputdsa,gosztolai2025marble}. These methods generally rely on the notion of topological conjugacy: two dynamical systems are considered equivalent if there exists a homeomorphism $h$ between their state spaces such that their flows, $f, g$ are conjugate ($h \circ f = g \circ h$). Such homeomorphisms may be learned explicitly by aligning either vector fields \cite{chen2024dform}, trajectories \cite{sagodi2025dynamical} or, in the case of Dynamical Similarity Analysis (DSA), implicitly via their Koopman eigenspectra\cite{ostrow2023beyond}. MSA quantifies dynamical similarity through the pullback metric, which is derived from the Jacobian of the system and therefore reflects local velocity field geometry. Importantly, topological conjugacy does not necessarily preserve geometry (App. \ref{app:van_der_pol}). 
Thus, MSA complements existing methods for assessing dynamical similarity by quantifying geometric rather than purely topological properties.
}

\textbf{Riemannian geometry of neural networks.}
Riemannian geometry has found numerous applications in machine learning, including: metric learning \cite{gruffaz2025riemannian}, natural gradients \cite{amari1998natural}, mechanistic interpretability \cite{hauser2017principles,brandon2025emergent, pellegrino2025rnns,wurgaft2026manifold}, and characterising the latent space of deep generative models \cite{arvanitidis2017latent, shao2018riemannian, park2023understanding}. One recent paper used information geometry to compare the learning dynamics of probabilistic models in the space of their \rd{output distributions\cite{mao2024training}, but not hidden layer representations. Relatedly, a line of work uses information geometry to identify local stimulus directions that maximally differentiate neural encoding models \cite{zhou2023comparing,feather2025discriminating,azeglio2026multi}. However, MSA is the first method to define a global distance between neural representations based on how they warp the Riemannian geometry of the input manifold.} 

\vspace{-5pt}

\section{Theoretical framework}\label{section:framework}

\vspace{-5pt}

In this section we develop a mathematical framework for comparing the intrinsic geometries of neural representations. We begin by formulating a neural network as a mapping from an input manifold to a representational manifold embedded in an ambient hidden-layer state space. We then review the pullback metric as an exact characterisation of the intrinsic geometry defined by this mapping. To compare two networks' pullback metrics, we introduce the spectral ratio (SR) distance between SPD matrices. We then provide a formal definition of MSA as the integrated SR distance over the input manifold, and demonstrate that it satisfies the conditions to induce a topological metric space over Riemannian metrics. Finally, we prove several properties of MSA (e.g. coordinate invariance), which make it fundamentally distinct from other distance measures and more suited for similarity analysis.

\textbf{Neural network model.} We work in the general setting of a network mapping inputs on a manifold $\mathcal{M}$ to target outputs. Symbolically, such a network can be summarised as:
\[
    \mathcal{M} \xlongrightarrow{\psi} \mathbb{R}^{n_\text{in}} \xlongrightarrow{\varphi} \mathbb{R}^{n} \xlongrightarrow{\zeta} \mathbb{R}^{n_\text{out}}
\]
where $\psi$ is the embedding of the input manifold (e.g. in pixel space), $\varphi$ the network up to the hidden layer whose representation we wish to study, and $\zeta$ the subsequent layers and decoder. Many architectures, such as fully-connected or transformer networks, fall within this setting. In Sections \ref{section:dynamics} and \ref{section:diffusion} we will show how MSA can be extended to dynamical systems and statistical \rd{models}. 

\textbf{The pullback metric.} The intrinsic geometry of a manifold is determined by a choice of inner product over its tangent spaces. To characterise our neural network, we will choose a particular intrinsic geometry on $\mathcal{M}$ that captures how it is warped by $\varphi \circ \psi$. This can be done using the \emph{pullback metric}, which defines the inner product between the tangent vectors of $\mathcal{M}$ in terms of their corresponding dot product in $\mathbb{R}^n$ when mapped to the hidden layer \cite{lee2003smooth,hauser2017principles}. Letting $(\cdot)_*$ denote the pushforward, the pullback metric is the following map:
\[
    g_p:T_p\mathcal{M}\times T_p\mathcal{M} \xlongrightarrow{\psi_*} \mathbb{R}^{n_\text{in}}\times \mathbb{R}^{n_\text{in}} \xlongrightarrow{\varphi_*} \mathbb{R}^{n}\times \mathbb{R}^{n} \xlongrightarrow{\langle \cdot, \cdot \rangle} \mathbb{R}
\]

\begin{wrapfigure}{r}{0.52\textwidth}
\vspace{-\intextsep}
\setlength{\abovecaptionskip}{1pt}
\setlength{\intextsep}{0pt}
  \renewcommand{\figurename}{Fig.}
    \centering\includegraphics[width=0.93\linewidth]{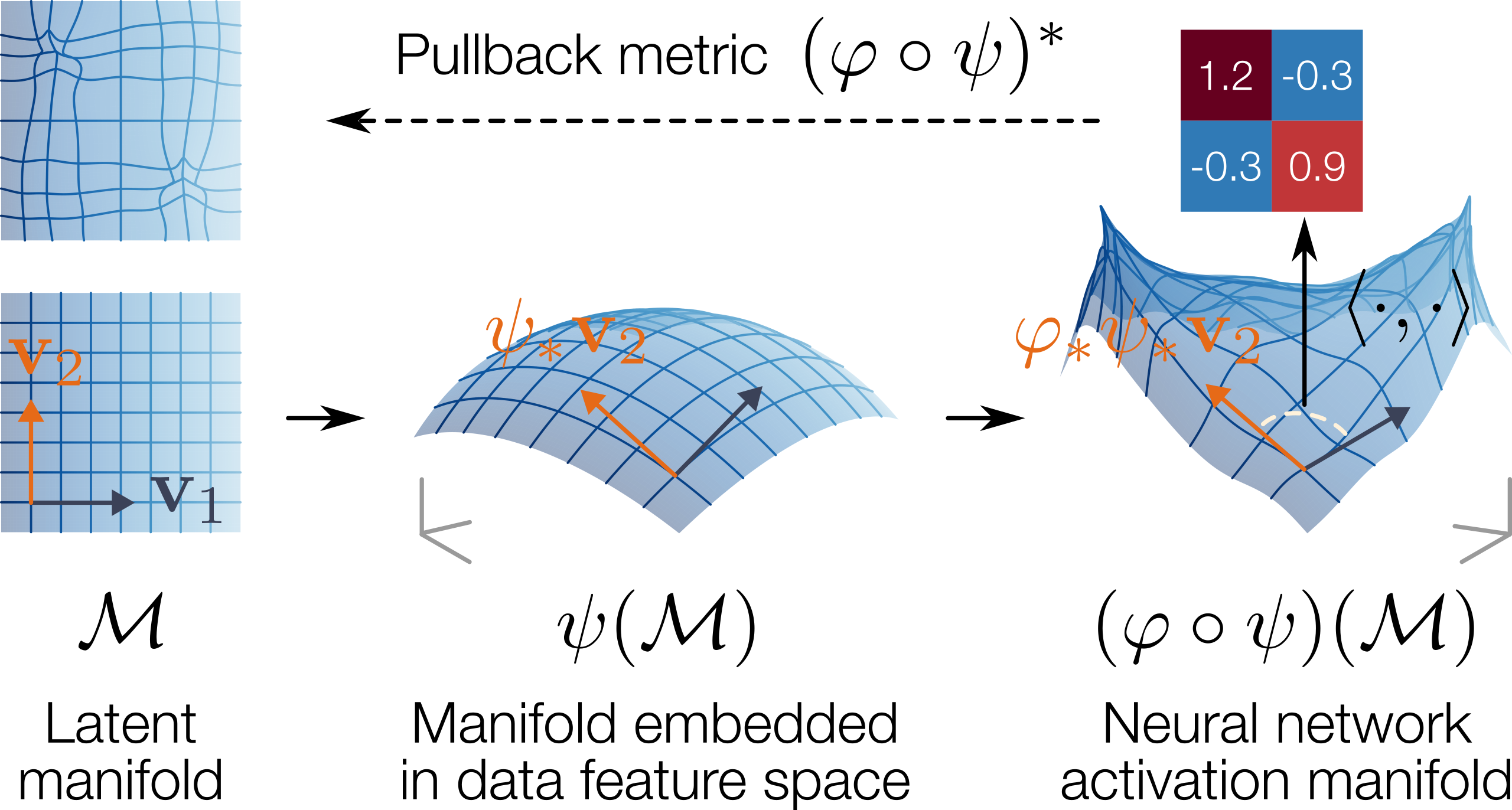}
    \caption{\textbf{\textit{The pullback metric of representations.}}}
    \label{fig:pullback}
    \vspace{-10pt}
\end{wrapfigure}

where $T_p\mathcal{M}$ is the tangent space at point $p\in\mathcal{M}$. Intuitively, the pushforward maps tangent vectors \rd{$\mathbf{v}_1,\mathbf{v}_2\in T_p\mathcal{M}$} of the abstract input manifold $\mathcal{M}$ to tangent vectors of the hidden layer manifold $(\varphi\circ\psi)(\mathcal{M})$, where their inner product can be measured using the standard Euclidean dot product $\langle \cdot, \cdot \rangle$ (Fig.\,\ref{fig:pullback}). This value can be taken as the inner product between the original tangent vectors in $T_p\mathcal{M}$, in order to define the same intrinsic geometry on $\mathcal{M}$ as its image in the hidden layer state space.

In local coordinates, the pushforward at $p$ can be represented by the Jacobian $J(p)$ of $\varphi \circ \psi$. Then the tangent vectors $\mathbf{v}_i \in T_p\mathcal{M}$ are pushed forward to $J(p)\mathbf{v}_i$ in the hidden layer, where their inner product is simply $(J(p)\mathbf{v}_i) \cdot (J(p)\mathbf{v}_j)= \mathbf{v}_i^\top J(p)^\top J(p)\mathbf{v}_j$. In this way, the pullback metric can be represented as $G(p)=J(p)^\top J(p)\in\mathbb{R}^{m\times m}$ where $m=\text{dim}(\mathcal{M})$. In particular, if the pushforward is injective, meaning that $(\varphi \circ \psi)(\mathcal{M})$ is an \emph{immersion}, $G(p)$ is an SPD matrix. 

Several studies have shown that the geometry defined by $G(p)$ reflects the internal computations of deep networks \cite{hauser2017principles, tennenholtz2022uncertainty, brandon2025emergent}. We may therefore expect that, for two different networks $\varphi^1$ and $\varphi^2$ receiving inputs from the same input manifold, comparing their respective pullback metrics \rd{$G^{\varphi^1}(p)$ and $G^{\varphi^2}(p)$} should provide a means of quantifying the similarity of their computations. This requires making a choice of distance function\footnote{Distances are also called \textit{metrics}, a concept that is related to -- but not to be confused with -- the \textit{Riemannian metric}.} between SPD matrices.

\textbf{The spectral ratio of generalised eigenvalues.} Multiple distance functions on the space of SPD matrices have been proposed \cite{forstner2003metric, lin2019riemannian}. Such distance functions have for example been used to compare empirical covariance matrices, compute the KL divergence between Gaussian distributions, and define canonical correlations between data matrices. \rd{While a framework to compare Riemannian metrics could in principle be built upon these existing distances, they generally lack properties such as boundedness and scale invariance that are particularly desirable for \textit{similarity analysis}.} Here we \rd{instead} introduce a novel distance function on SPD matrix space \rd{that is} particularly suited for similarity analysis due to several properties. 
\begin{definition}[Spectral ratio]\label{def:sr}
    Consider two SPD matrices $G,G'\in\mathbb{R}^{m\times m}$, and the generalised eigenequation $G\mathbf{v}_i = \lambda_iG'\mathbf{v}_i$. Here $\lambda_i\in\mathbb{R}_+$ and $\mathbf{v}_i\in\mathbb{R}^m$ form the $i$th generalised eigenvalue-eigenvector pair of $(G, G')$ with $\lambda_{i}\leq\lambda_{i-1}$ for $i\in [m]$. We define the spectral ratio distance:
    \[
        d_\sr(G, G')=1-\textstyle\sqrt{\frac{\lambda_{m}}{\lambda_{1}}}\in[0, 1]
    \]
\end{definition}
\vspace{-5pt}
\rd{To get intuition as to the spectral ratio we can consider the case where $G'=I$. Then, the generalised eigenequation becomes the usual eigenequation for $G$, and the spectral ratio quantifies the direction which is least stretched by $G$ relative to that which is most stretched. Figure \ref{fig:sr} shows the behaviour of this similarity function by comparing pairs of $2 \times 2$ SPD matrices generated by continuously varying the angle or the relative magnitudes of the columns.

\begin{figure}[h]
\centering

\includegraphics[width=\linewidth]{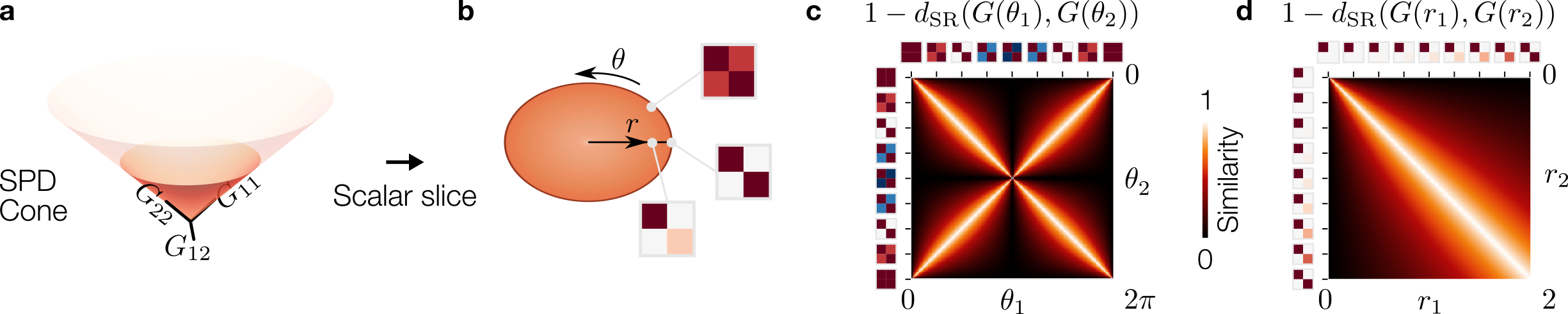}

\caption{\textbf{\textit{The spectral ratio is a \rd{pseudo}distance over SPD matrices.}} 
\textbf{a.} The set of SPD matrices (here $2\times 2$) is a solid cone in the Euclidean space of their entries. 
\textbf{b.} Slice of the SPD cone along matrices related by a scalar multiple. The slice is characterised by a variable $r$, defining the relative magnitude of the diagonal entries and $\theta$ defining the magnitude of the off-diagonal entries. The spectral ratio defines a similarity between pairs of SPD matrices, with different relative $\theta$ (\textbf{c}) or $r$ (\textbf{d}).}
\label{fig:sr}

\end{figure}

The spectral ratio is a bounded, monotonic reparametrisation of the Hilbert projective metric, from which several invariance properties follow. In appendix \ref{app:maths} we provide self-contained proofs of these invariances that do not rely on the partially ordered vector spaces of Hilbert metric theory\cite{bushell1973hilbert,birkhoff1957extensions,liverani1994generalization,mostajeran2024differential,nussbaum1988hilbert}.}

\begin{theorem}[The spectral ratio is a \rd{pseudo}distance]\label{thm:sr_distance}
    The SR is a \rd{pseudo}distance function on SPD matrices. This means that it satisfies i) separation: $d_\sr(G, G)=0$, ii) symmetry: $d_\sr(G, G')=d_\sr(G',G)$ and iii) the triangle inequality: $d_\sr(G, G'')\leq d_\sr(G, G')+d_\sr(G', G'')$.
    ~\hfill \hyperref[proof:sr_distance]\PROOF
\end{theorem}

Furthermore, since the spectral ratio is bounded, it can be naturally used to define a similarity:
\[
    1-d_\sr(G, G')\in [0, 1]
\]    
\rd{This similarity function between SPD matrices can then be used to define a similarity function between Riemannian metrics.}

\clearpage

\parsec{Metric Similarity Analysis (MSA).} We have argued for using the pullback metric to characterise how the intrinsic geometry around a particular point $p \in \mathcal{M}$ is transformed by a neural network. Recall that for a specific choice of coordinate system, the pullback metric can be represented as an SPD matrix. Thus, we can compare two networks globally by integrating their spectral ratio over $\mathcal{M}$:

\begin{definition}[Metric similarity analysis]
    Let  \rd{$G^{\varphi^1}(p)$ and $G^{\varphi^2}(p)$} be the local coordinate representations of the pullback metrics of the neural networks $\varphi^1$ and $\varphi^2$, at a point $p$ on an $m$-dimensional \rd{smooth} manifold $\mathcal{M}$. Then we define MSA as:
    \rd{\[
        d_\msa(\varphi^{1}, \varphi^{2}) = \frac{1}{\mathrm{Vol}_g(\mathcal{M})}\int_{\mathcal{M}} d_\sr(G^{\varphi^1}(p), G^{\varphi^2}(p)) \mathrm{dvol}_g(p)
    \]}
    where the normalising factor $\mathrm{Vol}_g(\mathcal{M})$ and integration over the volume form $\mathrm{dvol}_g$ are dependent on the metric $g$ on the input data manifold; we assume this matches the input data sampling density.
\end{definition}
In practice, this integral can be evaluated by standard numerical estimation methods based on cubature or Monte Carlo sampling, meaning that MSA can be applied even when the manifold is not explicitly given, as is the case for images (App.\,\ref{app:implementation}).
Furthermore, just as the SR is a \rd{pseudo}distance over SPD matrices, MSA defines a \rd{pseudo}distance over Riemannian metrics.
\begin{theorem}[MSA is a \rd{pseudo}distance between pullback Riemannian metrics]\label{thm:msa_distance} 
    MSA follows the separation, symmetry and triangle inequality identities.~\hfill \hyperref[proof:msa_distance]\PROOF
\end{theorem}
As for the SR, the MSA distance is bounded, and can be converted into a similarity function:
$1-d_\mathrm{MSA}(\varphi^1, \varphi^2) \in [0, 1]$. We use ``MSA'' to refer to either the distance or similarity function, depending on the context. \rd{In Appendix \ref{app:maths} we show that the spectral ratio and MSA are pseudodistances that become proper distances on an appropriate quotient space, and we henceforth simply refer to them as distance.}

\begin{figure}[h]
\centering

\begin{minipage}[c]{0.42\linewidth}\vspace{0pt}
  \setlength{\abovecaptionskip}{0pt}
  \setlength{\belowcaptionskip}{0pt}
  \caption{\textbf{\textit{Metric similarity analysis (MSA) enables comparison of the intrinsic geometry of neural representations.}} Two neural networks receiving inputs from the same manifold with different hidden-layer geometries, corresponding to distinct Riemannian metrics.}
  \label{fig:msa}
\end{minipage}
\hfill
\begin{minipage}[c]{0.55\linewidth}
  \centering
  \includegraphics[width=0.98\linewidth]{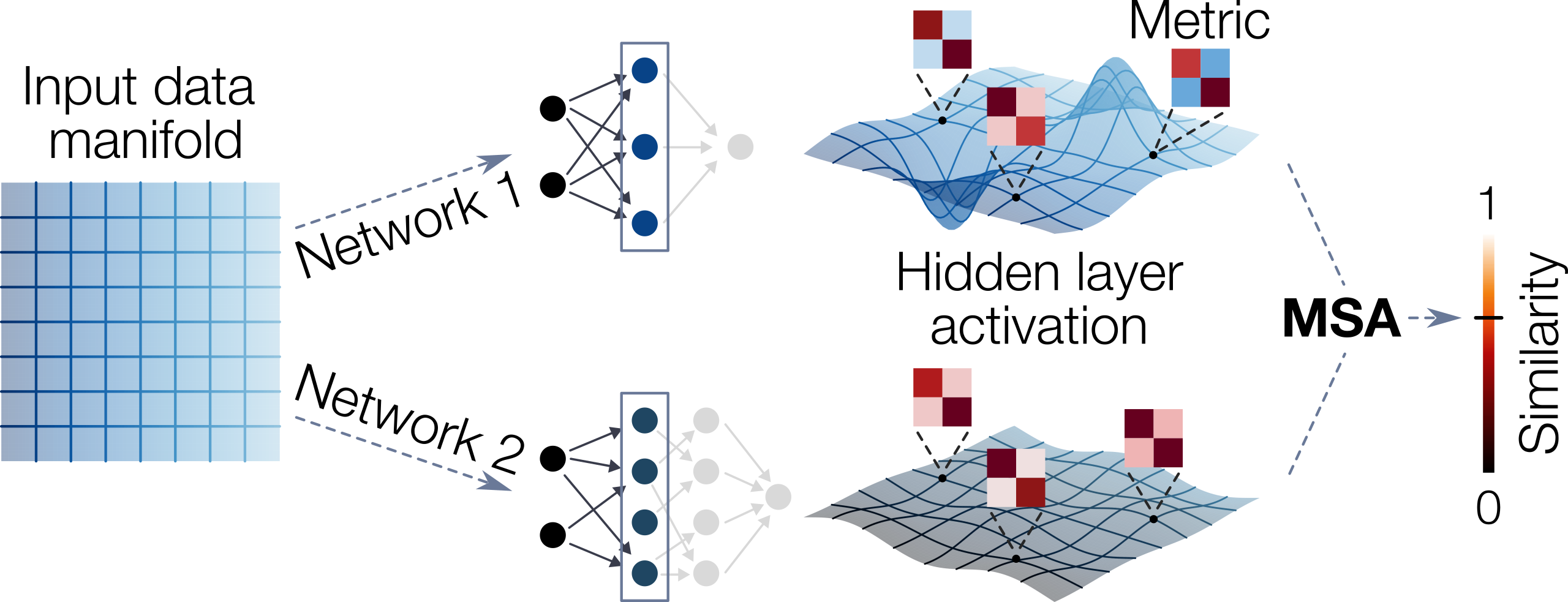}
\end{minipage}

\end{figure}

\vspace{5pt}

\textbf{Mathematical guarantees.} We briefly introduce two mathematical properties of MSA which are essential for it to be a well-defined similarity metric. Additional properties related to scale invariance and estimation error can be found in Appendix \ref{app:maths}. 

First, a fundamental desideratum of any similarity measure is invariance to rotations \rd{in state space}. This \rd{property is satisfied} for many methods based on extrinsic geometry, such as CCA and Procrustes. 

\begin{proposition}[MSA is invariant to state-space rotations]\label{thm:rotation_invariance}
    Consider two networks $\varphi^1:\mathbb{R}^{n_\mathrm{in}}\rightarrow\mathbb{R}^{n_1}$ and $\varphi^2:\mathbb{R}^{n_\mathrm{in}}\rightarrow\mathbb{R}^{n_2}$. Then:
    \[
        d_\msa(Q_1\circ\varphi^1, Q_2\circ\varphi^2)=d_\msa(\varphi^1, \varphi^2)
    \]
    where $Q_i \in O(n_i)$ is an orthogonal matrix.
    ~\hfill \hyperref[proof:rotation_invariance]\PROOF
\end{proposition}
This invariance to rotations can be understood intuitively: given two tangent vectors $\mathbf{v}_1, \mathbf{v}_2\in T_p\mathcal{M}$, composing the neural network functions with an orthogonal matrix $Q$ (or any Euclidean isometry), will rotate the images of each $\mathbf{v}_i$ under the pushforward, and thus will not change their dot product. Since the pullback metric is defined via this dot product, the intrinsic geometry remains the same. 

\rd{Second, a crucial property for Riemanian methods is coordinate invariance. A particular choice of coordinates for $\mathcal{M}$ defines a representation of the pullback metrics as matrices $G^{\varphi^{1}}(p)$, $G^{\varphi^{2}}(p)$ whose generalised eigenvalues can then be computed. However, changing local coordinates will change the matrix representations of these metrics, even while the underlying geometry remains the same. As a result, any similarity measure that is sensitive to the choice of coordinates is not geometrically meaningful, as it will depend on an arbitrary choice of representation. Importantly, while the matrix representations of  $G^{\varphi^{1}}(p)$, $G^{\varphi^2}(p)$ will depend on the specific choice of coordinates, their generalised eigenvalues do not. As a result, MSA is coordinate invariant.}

\clearpage

\begin{proposition}[MSA is invariant to \rd{transitions between} local coordinates]\label{thm:coordinate_invariance}
    Consider two networks $\varphi^1:\mathbb{R}^{n_\mathrm{in}}\rightarrow\mathbb{R}^{n_1}$, $\varphi^2:\mathbb{R}^{n_\mathrm{in}}\rightarrow\mathbb{R}^{n_2}$. Let $(U_j,\phi_j)$ and  $(U_k,\phi_k)$ be smooth charts on $\mathcal{M}$. Then:
    \[
        d_\msa(\varphi^1\circ \psi \circ \phi_j^{-1}, \varphi^2 \circ \psi \circ \phi_j^{-1}) = d_\msa(\varphi^1 \circ \psi \circ \phi_k^{-1} , \varphi^2 \circ \psi \circ \phi_k^{-1})
    \] 
    where $\psi:\mathcal{M}\rightarrow\mathbb{R}^{n_\mathrm{in}}$ is the input embedding.
    ~\hfill \hyperref[proof:coordinate_invariance]\PROOF
\end{proposition}


Together, these theoretical results demonstrate that MSA fulfils several desirable properties for similarity analysis: coordinate invariance, scale invariance, rotation invariance, and boundedness. We note that, \rd{depending on the application, only a subset of these properties may be required, in which case} MSA could in principle be used with pre-existing distances on the SPD cone, like the AIRM or Hilbert metric. Finally, existing methods based on extrinsic geometry generally fail to produce at least one of these principles, underscoring the fundamental contribution of MSA (Table \ref{tab:msa_comparison}). 

\begin{table}[h]
\centering
\begin{minipage}[t]{0.32\linewidth}\vspace{0pt}
\setlength{\abovecaptionskip}{0pt}
\setlength{\belowcaptionskip}{0pt}
\caption{
\textbf{Comparison of representation similarity measures.} \rd{Different measures have different desiderata.}
The left three columns use MSA with different SPD distances (App.\,\ref{app:maths}).
}
\label{tab:msa_comparison}
\end{minipage}
\hfill
\begin{minipage}[t]{0.65\linewidth}\vspace{0pt}
\centering
\small
\setlength{\tabcolsep}{2pt}
\renewcommand{\arraystretch}{1.0}

\begin{tabular}{@{}lcccccc@{}}
\textbf{Method} & \multicolumn{3}{c}{\textbf{{\color{darkBlue}MSA}}} & \textbf{Procrustes} & \textbf{CKA} & \textbf{RSA} \\
\cline{2-4}
\noalign{\vspace{2pt}}
Distance & {\color{darkBlue}SR} & AIRM & Hilbert & Frobenius & HSIC & Pearson \\
\textit{Rotation invariant}  & $\checkmark$ & $\checkmark$     & $\checkmark$    & $\checkmark$ & $\checkmark$ & $\checkmark$ \\
\textit{Scale invariant}  & $\checkmark$ & $\times$     & $\checkmark$     & $\times$ & $\checkmark$ & $\checkmark$ \\
\textit{Bounded}          & $\checkmark$ & $\times$     & $\times$ & $\times$ & $\checkmark$ & $\checkmark$ \\
\textit{Intrinsic geometry} & $\checkmark$ & $\checkmark$ & $\checkmark$ & $\times$ & $\times$ & $\times$ \\
\end{tabular}
\end{minipage}
\end{table}

\vspace{5pt}

\section{Numerical experiments}

\rd{Next}, we highlight three key applications of MSA: i) revealing structural differences in the intrinsic geometry of deep network representations, ii) disentangling the computational mechanisms of different state space models, and iii) comparing statistical manifolds of diffusion models.

\vspace{-5pt}

\subsection{Distinguishing rich and lazy representations in deep networks}\label{section:rich_lazy}

To highlight the importance of characterising the \textit{intrinsic} geometry of neural representations, we start by performing MSA on models known to produce different solutions to the same task: \emph{rich} and \emph{lazy} deep networks. \rd{Such networks arise from distinct \textit{feature learning} and \textit{kernel} regimes, as controlled by specific hyperparameters.} We trained a one-hidden-layer neural network to map inputs on a two-dimensional manifold to discrete classes. Specifically, we considered $\mathcal{M} = [-1,1]^2$, with:\vspace{-4pt}
    \[
        \mathbf{p}\in \mathcal{M}, \quad \mathbf{x}\overset{\scriptstyle \psi}{=}Q\mathbf{p}, \quad \mathbf{z} \overset{\scriptstyle \varphi}{=} \operatorname{tanh}(W_1 \mathbf{x}+\mathbf{b}), \quad \mathbf{y}\overset{\scriptstyle \zeta}{=} W_2\mathbf{z}.
    \]
The network was trained using a cross entropy loss to map $\mathbf{x}$ to a one-hot class vector (Fig.\,\ref{fig:rich_lazy}\textbf{a}). \rd{In appendix \ref{app:rich_lazy} we describe the training procedure and hyperparameters in detail.}

\begin{figure}[h]
    \centering
    \includegraphics[width=\linewidth]{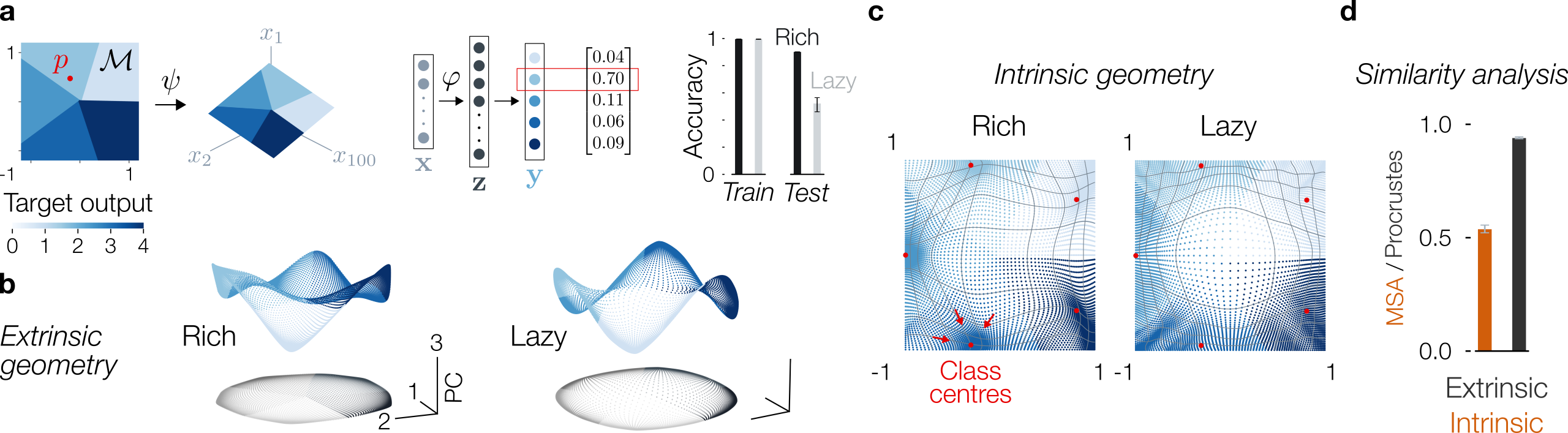}
    \vspace{-15pt}
    \caption{\textbf{\textit{Rich and lazy network representations have different intrinsic geometries.}} \textbf{a.} Network trained to map a 2D manifold to discrete classes. \textbf{b.} PCA applied to the activation of the rich and lazy network, coloured according to the target class. \textbf{c.} Mesh of one-dimensional geodesics illustrating the intrinsic geometry of the representations. \textbf{d.} Similarity between the rich and lazy representations.}
    \label{fig:rich_lazy}

\end{figure}

Considerable prior work has shown that neural networks can perform such tasks via two fundamentally different computational mechanisms.
In the rich regime, networks learn latent features of their training data, leading to structured representations with good generalisation performance, whereas networks in the lazy regime remain close to their random initialisation \cite{woodworth2020kernel}. \rd{Here we specifically consider} small-variance weight initialisations \rd{which} are known to produce rich learning, \rd{and} large-variance initialisations which lead to lazy learning
\cite{saxe2013exact}. 

What representational geometries support these different computations? Simply visualising the network activation in state space via principal component analysis (PCA) suggests a similar embedding geometry in each regime (Fig.\,\ref{fig:rich_lazy}\textbf{b}). In contrast, plotting geodesic grid lines under the pullback metric hints at distinct intrinsic geometries, where the rich network solves the task by warping the input manifold to concentrate points on class centres  (Fig.\,\ref{fig:rich_lazy}\textbf{c}). To quantify this, we performed similarity analyses between the rich and lazy networks comparing both their intrinsic (via MSA) and extrinsic (via Procrustes) geometries. Indeed, Procrustes reported near-perfect similarity between lazy and rich networks, whereas MSA reported low values compared to the seed-to-seed variability within a regime (Fig.\,\ref{fig:rich_lazy}\textbf{d}). \rd{Looking more specifically at the volume element and Dirichlet density of the two networks revealed that rich networks learned to warp their representations near class boundaries (Fig.\,\ref{sup_fig:rich_lazy_geometry}), consistent with learning geometric features that support generalisation. Instead, lazy networks remained close to their seed-dependent random initialisation (Fig.\,\ref{sup_fig:rich_lazy_training}).} Thus, characterising intrinsic representational geometry is necessary to disentangle fundamentally different network computations.

We next sought to exploit the fact that different hyperparameters generate distinct network solutions, due to different tasks (determined by the number of classes), learning regimes (by the variance of initial weights), or initialisations (by the pseudo-random seed).
To further test whether MSA can identify meaningful differences in network solutions, we varied these hyperparameters sequentially, forming a hierarchical structure (Fig.\,\ref{fig:rich_lazy2}\textbf{b}, \textit{top}). We then asked whether MSA identified this hierarchy. Indeed, MSA was able to capture this hierarchy, with similarities scaling evenly with the number of differing hyperparameters across models (Fig.\,\ref{fig:rich_lazy2}\textbf{a}). Instead, all other methods reported only  minimal dissimilarities between rich and lazy models, and some methods even reported high similarity between models with the least number of shared hyperparameters (Fig.\,\ref{fig:rich_lazy2}\textbf{b}, \textit{bottom}). 

\begin{figure}[h]

\vspace{5pt}

\centering

\begin{minipage}[t]{0.42\linewidth}\vspace{0pt}
  \setlength{\abovecaptionskip}{0pt}
  \setlength{\belowcaptionskip}{0pt}
  \caption{\textbf{\textit{MSA captures hierarchical structure across neural representations.}} Networks with different learning regimes, tasks, architectures, and seeds.  
  \textbf{a.} MSA, RSA and CCA applied to pairs of these networks. Note that RSA and CKA are equivalent when mean-centered \cite{williams2024equivalence}. MSA captures the hierarchical structure across models while CCA and RSA do not. 
  \textbf{b.} Instead of considering all combinations of such hyperparameters, we hierarchically vary them. Networks differ by learning regime (rich vs. lazy), task (5 vs. 2 classes), or seed (S1/S2, determining initial weights).}
  \label{fig:rich_lazy2}
\end{minipage}
\hfill
\begin{minipage}[t]{0.55\linewidth}\vspace{0pt}
  \centering
  \includegraphics[width=\linewidth]{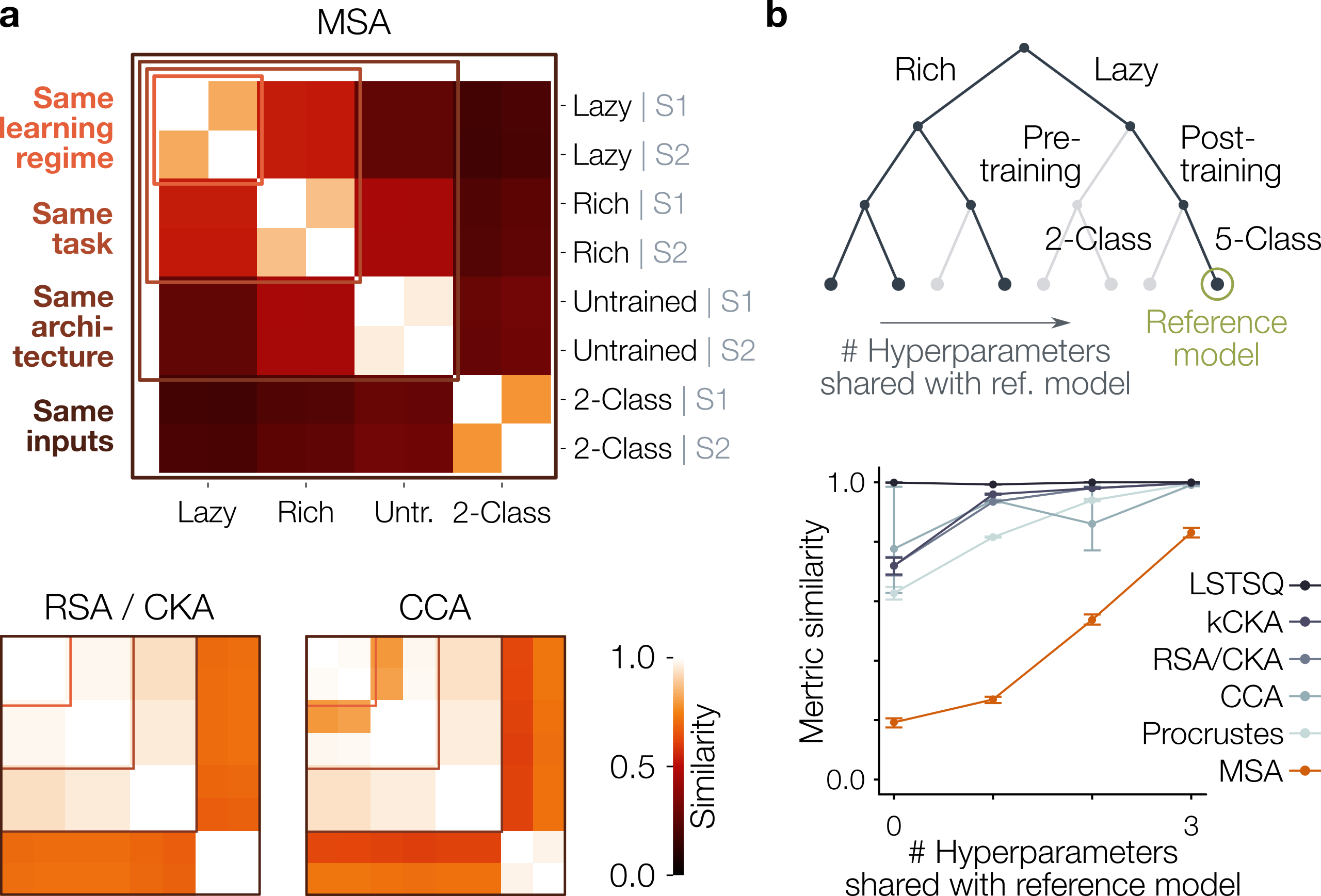}
\end{minipage}

\vspace{3pt}

\end{figure}

These results demonstrate that MSA can disentangle finer-grained task-relevant structure than methods based purely on extrinsic geometry, enabling more meaningful comparisons between networks. MSA can likewise be applied to larger-scale architectures, e.g., to quantify how the intrinsic geometry of ImageNet images are warped by successive layers of a ResNet (App.\,\ref{app:resnet}). In the next section, we will show that similar insights can also be gained from nonlinear dynamical systems. 

\vspace{-10pt}


\subsection{Comparing dynamical mechanisms of state space models}\label{section:dynamics}

\vspace{-5pt}

Here, we will show how MSA can be used to investigate dynamical systems models by examining how input manifolds are warped over time by nonlinear flows \cite{pellegrino2025rnns}. Specifically, we asked whether MSA could provide insight into the internal computations of dynamical system models performing a neuroscience task in which an internal memory of the input manifold must be maintained.\cite{xie2022geometry}

We trained models to reproduce an input sequence parametrised by two angles $\theta_1, \theta_2 \in \mathbb{S}^1$ following a variable delay (Fig.\,\ref{fig:rnnsmm}\textbf{a}). We considered two architectures used in neuroscience applications \cite{pellegrino2023low,barak2017recurrent}: vanilla RNNs and structured SSMs \cite{durstewitz2023reconstructing, ryoo2025generalizable}. \rd{The dynamics of each model was governed by:
\[
    \mathrm{RNN:} \quad \textstyle\frac{d}{dt}\mathbf{x}(t) = -\mathbf{x}(t) + A\tanh(\mathbf{x}(t)) + B\mathbf{u}(t), \quad
    \mathrm{SSM:} \quad \textstyle\frac{d}{dt}\mathbf{x}^{(i)}(t) = A^{(i)}\mathbf{x}^{(i)}(t) + B^{(i)}\mathbf{u}^{(i)}(t)
\]
where $\mathbf{u}(t) = \mathbf{u}^{(0)}(t) \in \mathbb{R}^3$ encodes the input angles and go cue for each model (Fig.\,\ref{fig:rnnsmm}\textbf{a}). The SSM consisted of two blocks of dynamics ($i\in\{0,1\}$) initialised with HiPPO connectivity and chained through a nonlinearity and residual connections: $\mathbf{u}^{(1)}=[\tanh(\mathbf{x}^{(0)}), \mathbf{u}(t)]$ \cite{gu2022train}. The outputs were linearly decoded from the hidden states in both the RNN and SSM.
}

\begin{figure}[h]
  \centering
  \includegraphics[width=\linewidth]{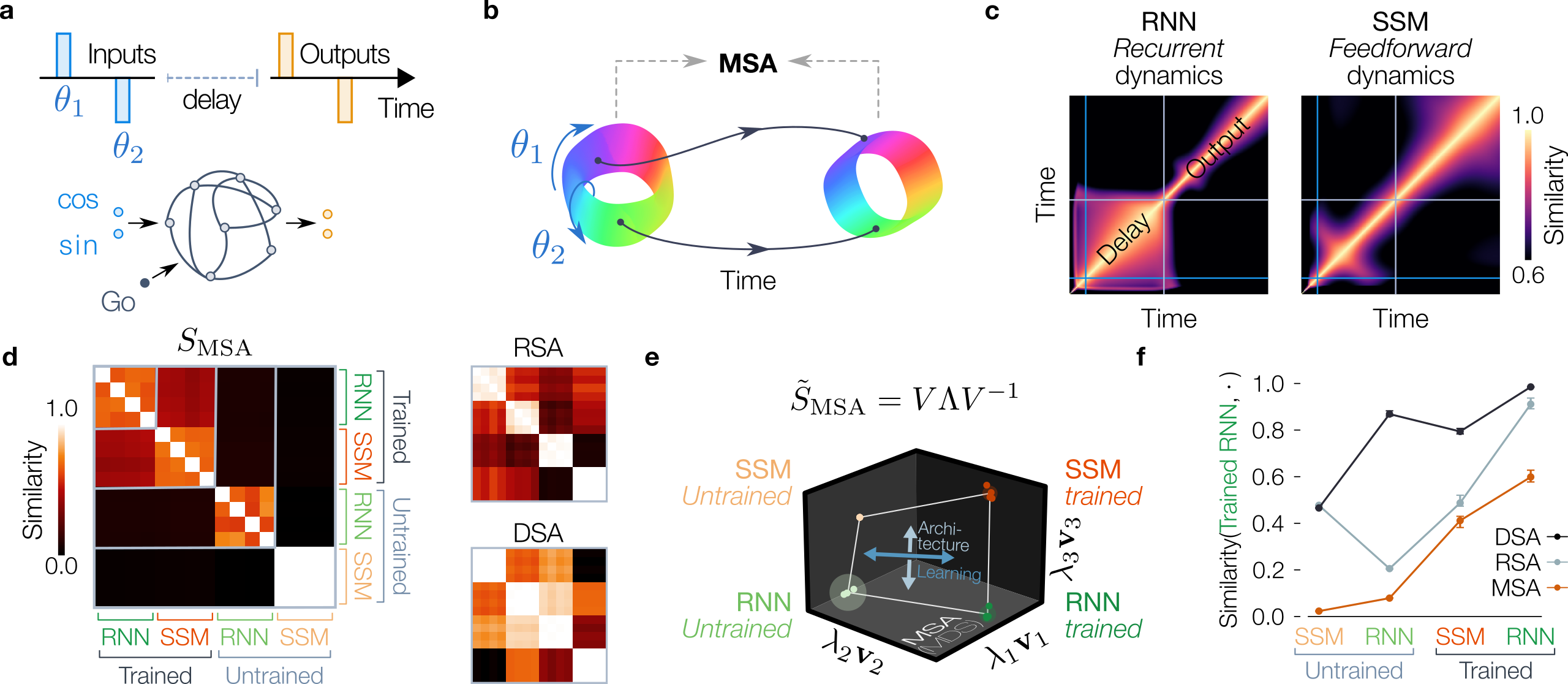}
  \vspace{-5pt}
  \caption{\textbf{\textit{MSA enables the comparison of state-space models and recurrent neural networks.}} \textbf{a.} The network receives a sequence of angular inputs: $u_i(t) = f_i(\theta_1)\mathds{1}_{t=t_1}+f_i(\theta_2)\mathds{1}_{t=t_2}$ for $i=1,2$, with $f_1(x) = \cos(x)$ and $f_2(x) = \sin(x)$, which it must keep in memory for a variable delay period. At the Go cue (signalled by $u_3(t) = \mathds{1}_{t=t_{\operatorname{go}}}$), it must output same sequence. \textbf{b.} PCA applied to the state of the RNN. After the second input, the state exactly lies on a torus representing the two angular variables, and whose geometry dynamically changes. \textbf{c.} MSA applied to the torus geometry at pairs of time points, within each model. \textbf{d.} Similarity matrix between RNN and SSM models before and after training. \textbf{e.} Eigenvectors of the doubly-centred distance matrix (i.e. MDS with MSA distance). \textbf{f.} Similarity analysis of the geometry (RSA), dynamics (DSA) and geometry+dynamics (MSA), between the trained RNN and all other models. 
  }
  \label{fig:rnnsmm}
  \vspace{-15pt}
\end{figure}

The input manifold for this task is a two-dimensional torus, $\mathcal{M}=\mathbb{T}^2=\mathbb{S}^1 \times \mathbb{S}^1$ parametrised by $\theta_1$ and $\theta_2$ (Fig.\,\ref{fig:rnnsmm}\textbf{b}). After receiving the inputs, the model must maintain a representation of both angles in its internal state during the delay period. Multiple computational mechanisms have been proposed for such working memory tasks, from continuous attractors of fixed points \cite{khona2022attractor} to dynamical memories supported by non-normal dynamics \cite{ganguli2008memory, goldman2009memory}. We thus asked whether the two models employ different computational strategies to perform the task.

To see signatures of dynamic geometry within each models, we first applied MSA to compare the representation within the \emph{same} model between timepoints $t$ and $t'$ (Fig.\,\ref{fig:rnnsmm}\textbf{c}). Interestingly, the two models showed different changes in their geometries over time, with only the RNN maintaining high similarity throughout the delay. These results \rd{could potentially be explained by the fact that linear SSM dynamics cannot create ring or toroidal attractors to store memories.} \rd{Instead}, recent work has proposed that working memory can be implemented by linear models via non-normal rotational solutions \cite{ritter2025efficient}, which can skew geometries due to shearing. 

We next applied MSA to the full input-time manifolds produced by the RNN and the SSM to compare \emph{across} model classes. Specifically, we consider time as a separate coordinate, forming a 3-dimensional manifold $\mathcal{M}'$ on which the $3 \times 3$ pullback metrics could be defined\rd{\cite{pellegrino2025rnns}. This input-time pullback metric thus enables the} geometry and dynamics of the models to be simultaneously compared (App.\,\ref{app:van_der_pol}). MSA was able to cluster the same architectures over seeds, while also identifying differences between trained and untrained models (Fig.\,\ref{fig:rnnsmm}\textbf{d}), suggesting potential applications in clustering dynamical solutions found by different models. To visualise these results we applied multidimensional scaling (MDS) on the MSA distance matrix, observing a low-dimensional representation that reflected the structure of the hyperparameters (architecture and training \rd{stage}, Fig.\,\ref{fig:rnnsmm}\textbf{e}).

We asked whether it was possible to observe these differences through comparisons of the geometry (via RSA) or dynamics (via DSA) separately (Fig.\,\ref{fig:rnnsmm}\textbf{d},\,\textit{right}). While RSA was able to separate between the two trained model architectures, it also reported the same similarity between a trained RNN vs. untrained SSM as between a trained RNN vs. trained SSM (Fig.\,\ref{fig:rnnsmm}\textbf{f}). DSA was \rd{more} sensitive to differences between \rd{architectures, but also reported that untrained RNNs were most similar to trained SSMs}. In contrast, MSA identified similarities that scaled both with architecture and training stage.

Together these results demonstrate how MSA reveals structural differences in dynamical representations of models while providing insight into their internal computations.

\vspace{-5pt}
 
\subsection{\rd{Comparing the information geometries of} statistical manifolds}\label{section:diffusion}
Thus far, we have focused on neural networks as deterministic mappings from an input manifold to hidden layer representations. Yet an advantage of MSA is that it can be naturally extended to statistical manifolds through its connection with information geometry. As we will demonstrate below, this enables the quantitative comparison of probabilistic models, such as diffusion models, which would be challenging with standard similarity measures based on extrinsic geometry.

\textbf{Information geometry.} A statistical manifold\cite{amari2016information} is a manifold $\mathcal{M}$ whose points $\boldsymbol{\pi}\in\mathcal{M}$ parametrise a family of probability densities $p_{\boldsymbol{\pi}}(\mathbf{x})$ over $\mathbf{x}\in\mathbb{R}^n$ (Fig.\,\ref{fig:statistical-manifolds}). This parametrisation endows $\mathcal{M}$ with a geometric structure defined by the Fisher-Rao metric. In brief, for a given sample $\mathbf{x}$, the log-likelihood function defines a map $\log p_{(\cdot)} (\mathbf{x}) : \mathcal{M} \to \mathbb{R}$. Pulling back the Euclidean metric through $\log p_{(\cdot)}(\mathbf{x})$ yields $J(\boldsymbol{\pi})^\top J(\boldsymbol{\pi}) \in \mathbb{R}^{m \times m}$, where $J(\boldsymbol{\pi})=(\nabla_{\boldsymbol{\pi}} \log p_{\boldsymbol{\pi}}(\mathbf{x}))^\top$. 

By taking an expectation over $\mathbf{x}$, we recover the Fisher information matrix:
\[
    \mathcal{M} \xlongrightarrow{\log p_{(\cdot)} (\mathbf{x})} \mathbb{R}, \qquad \mathcal{I}(\boldsymbol{\pi}) = \mathbb{E}_{\mathbf{x}\sim p_{\boldsymbol{\pi}}} [\nabla_{\boldsymbol{\pi}} \log p_{\boldsymbol{\pi}}(\mathbf{x})\otimes \nabla_{\boldsymbol{\pi}} \log p_{\boldsymbol{\pi}}(\mathbf{x})] \in \mathbb{R}^{m\times m}
\]
which defines the local coordinate expression of the Fisher-Rao metric, encoding changes in $p_{\boldsymbol{\pi}}(\mathbf{x})$ over $\mathcal{M}$. Under standard regularity and identifiability conditions\cite{amari2016information}, the Fisher-Rao metric is a Riemannian metric. This suggests a natural extension of MSA to statistical manifolds by computing the similarity between their Fisher-Rao metrics, thereby directly comparing their information geometries.

Consider two statistical models with a common set of parameters: $p_{\boldsymbol{\pi}}(\mathbf{x})$ and $q_{\boldsymbol{\pi}}(\mathbf{x})$ (Fig.\,\ref{fig:statistical-manifolds}). Each model corresponds to a distinct statistical manifold whose geometry is given by the Fisher information matrix in the local coordinates of the parameters. MSA quantifies the similarity of the information geometries defined by the two statistical models.

\begin{figure}[h]
\centering

\begin{minipage}{0.4\linewidth}\vspace{0pt}
  \caption{
  \textbf{\textit{Information geometry MSA.}}
  A manifold parametrising the mean and standard deviation of a Gaussian (top) and Laplace (bottom) distribution. Their Fisher-Rao metrics induce different geometries (shown as geodesic gridlines) that can be compared with MSA.
  }
  \label{fig:statistical-manifolds}
\end{minipage}
\hfill
\begin{minipage}{0.58\linewidth}\vspace{0pt}
  \centering
  \includegraphics[width=0.98\linewidth]{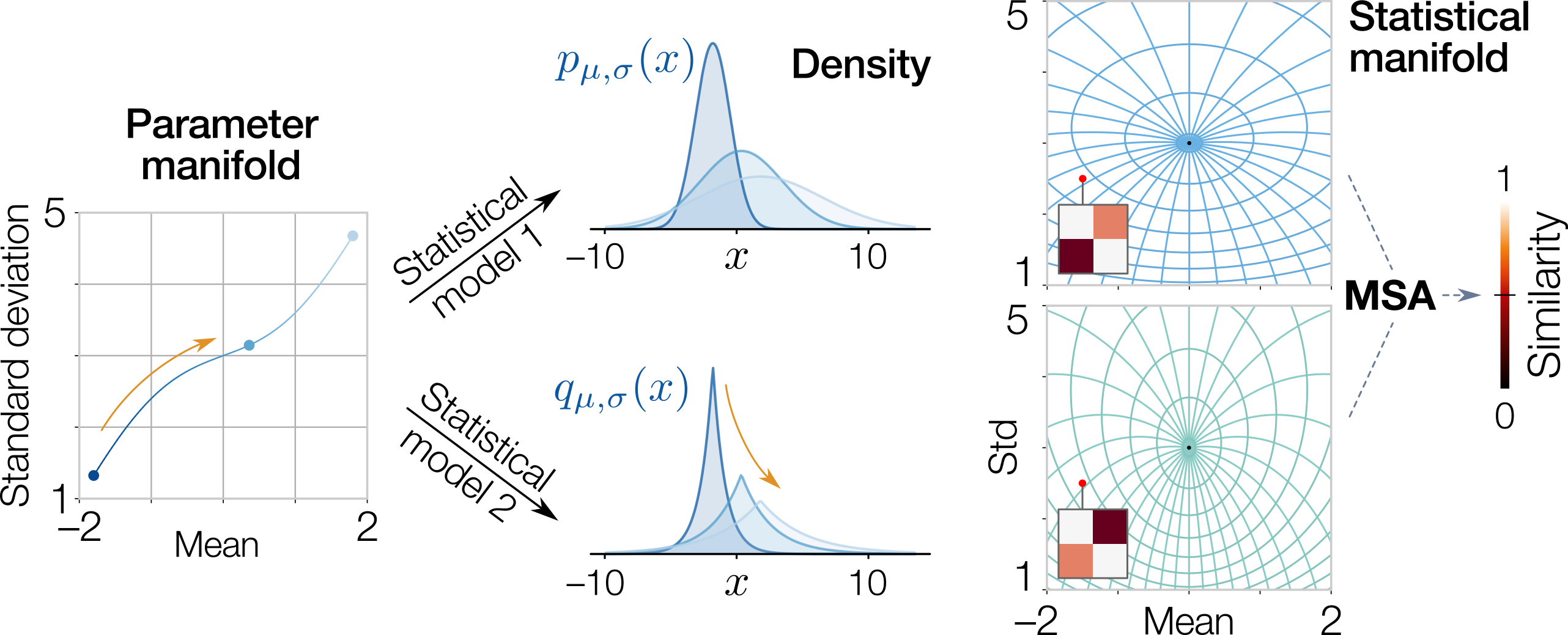}
\end{minipage}
\vspace{5pt}
\end{figure}

\rd{Thus, MSA provides a principled way to analyse statistical models by defining a similarity measure over information geometries. In App.\,\ref{app:cramer_rao} we also draw links between MSA, the Cramér-Rao bound and the asymptotic relative efficiency (ARE) of statistical estimators.}

\textbf{Application to large-scale diffusion models.} We next show how MSA can be applied to analyse the information geometry of large-scale diffusion models. Briefly, a diffusion model generates samples (e.g. images) via a backward-in-time stochastic differential equation (SDE). This process is controlled by a variable $\boldsymbol{\pi}\in\mathcal{M}$ (such as a text prompt) which guides the latent diffusion dynamics. Solving this system gives the marginal density of the learned generative distribution. Here, we study the information geometry of this distribution over the manifold $\mathcal{M}$ (App.\,\ref{app:diffusion}).

Thanks to the scalability of MSA, we can apply it to StableDiffusionXL, a large text-to-image diffusion model capable of generating high-quality images \cite{podell2023sdxl}. The conditional variable is the text embedding of the prompt such that the latent diffusion dynamics are given by:
\begin{equation*}
    \mathbf{dx} = \mathrm{UNet}(\mathbf{x}, \mathrm{TextEmb}(\text{\footnotesize``a leaf''}), t)\mathrm{dt} + \sigma(t)\mathrm{dB}
\end{equation*}
We defined a manifold $\mathcal{M}$ of text embeddings via bilinear interpolation between four basis embeddings (Fig.\,\ref{fig:diffusion}\textbf{b}). This produces a statistical manifold with each $\boldsymbol{\pi} \in \mathcal{M}$ corresponding to a different distribution of generated images (Fig.\,\ref{fig:diffusion}\textbf{c}; App.\,\ref{app:diffusion}).

\begin{figure}[h]

\vspace{10pt}

\centering
\includegraphics[width=\linewidth]{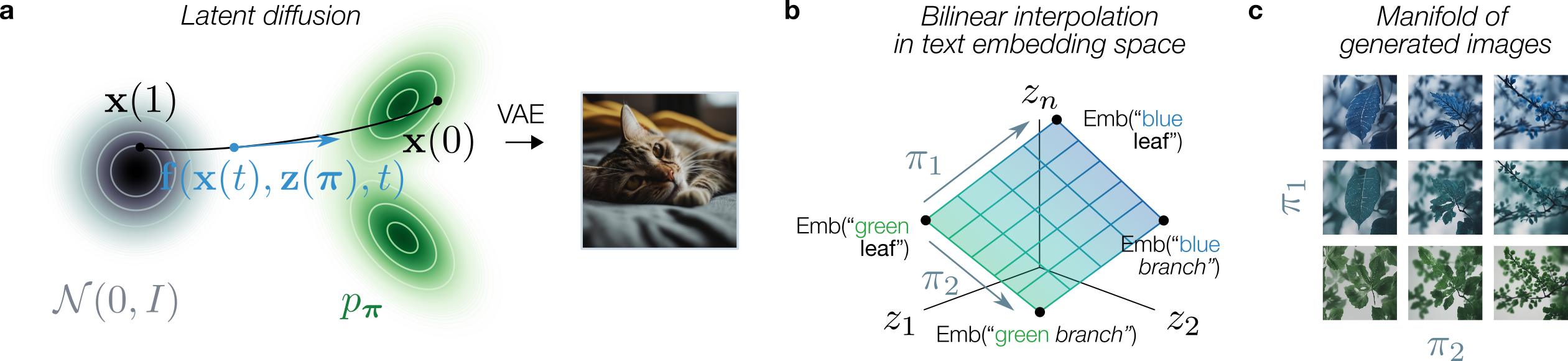}
\vspace{-20pt}
\caption{
\textbf{\textit{MSA can be applied to large-scale generative models.}}
\textbf{a.} Schematic of StableDiffusionXL: a backward diffusion process transports a Gaussian distribution to a latent generative distribution. Samples are then mapped to image space via a variational auto-encoder (VAE).
\textbf{b.} The diffusion process is conditioned on a text embedding. We generated a manifold of text embeddings via bilinear interpolation:
$
\mathbf{z}(\boldsymbol{\pi})=\pi_1\pi_2 \mathbf{v}_1+(1-\pi_1)\pi_2 \mathbf{v}_2+\pi_1(1-\pi_2) \mathbf{v}_3+(1-\pi_1)(1-\pi_2) \mathbf{v}_4.
$
This defines a statistical manifold of generative distributions.
\textbf{c.} Samples across the manifold for a fixed $\mathbf{x}(1)$.
}
\label{fig:diffusion}

\vspace{-15pt}
\end{figure}

Using MSA, we computed the similarity of the Fisher-Rao metric at pairs of time points during the diffusion process, revealing dynamic changes in the intrinsic geometry throughout the flow (Fig.\,\ref{fig:diffusion-results}\textbf{a}). Furthermore, applying MDS on the MSA distance matrix enabled the trajectory of the information geometry to be visualised in a low-dimensional embedding (Fig.\,\ref{fig:diffusion-results}\textbf{b}).

Finally, MSA can be used to compare different diffusion dynamics, e.g., while varying hyperparameters. Here, we chose to illustrate that by computing the similarity between different {guidance} levels. 
The guided diffusion dynamics are given by combining the conditional score with an estimate of the unconditional score, obtained using an empty-string embedding
(App.\,\ref{app:diffusion}). Previous work has found that the guidance level controls a trade-off between image diversity and alignment-to-prompt \cite{ho2022classifier} (Fig.\,\ref{sup_fig:diffusion_samples}). To study this phenomenon at the level of the latent variable $\mathbf{x}$ we computed the MSA similarity of models with different levels of guidance against the guidance-free model. We observed systematic changes in the information geometry of the statistical manifold over guidance levels (Fig.\,\ref{fig:diffusion-results}\textbf{c}), suggesting a potential use case in hyperparameter tuning. Overall, this highlights MSA as a scalable method to analyse the information geometry of large-scale statistical models.
\vspace{-5pt}

\begin{figure}[h]
\begin{minipage}{0.44\linewidth}
    \caption{
    \textbf{\textit{Diffusion information geometry.}}
    \textbf{a.} Similarity between the statistical manifold geometry at pairs of timepoints.
    \textbf{b.} Visualisation of the information geometry using MDS on the MSA distance.
    \textbf{c.} MSA similarity between different guidance levels vs. no guidance ($\gamma=1$).
    }
    \label{fig:diffusion-results}
\end{minipage}%
\hfill
\begin{minipage}{0.54\linewidth}
    \vspace{0pt}
    \includegraphics[width=\linewidth]{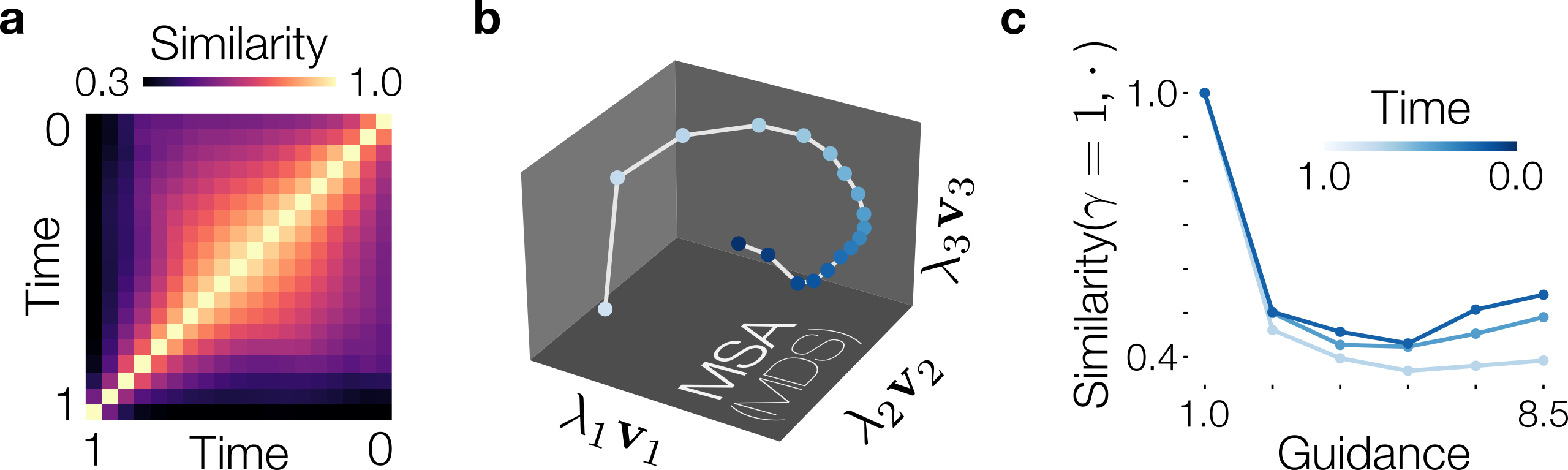}
\end{minipage}

\end{figure}

\vspace{-5pt}

\section{Discussion}

\vspace{-5pt}

\textbf{Summary.} \rd{We introduced MSA, a new similarity measure for comparing neural network representations through the Riemannian geometries induced by their pullback metrics. We proved several properties that make MSA a mathematically principled approach to comparing networks' internal computations, rather than their embedding geometries. We demonstrated in experiments the effectiveness of MSA applied to deep networks, dynamical systems, and statistical models.}

\textbf{Limitations.} \rd{MSA requires knowledge of the neural network function $\varphi$. In data applications (e.g., in neuroscience \cite{barbosa2025quantifying}), neural activations must be fitted with a parametric model, typically requiring dense sampling. On the other hand, methods like RSA and CKA work well even for sparsely sampled data. Second, MSA requires the input manifold to be well characterized and can be challenging in applications where the input manifold is high-dimensional and unknown (e.g., in vision models). In these cases, MSA can be applied to low-dimensional slices of $\mathcal{M}$, either through interpolation of exemplars (E.g. Fig.\,\ref{fig:diffusion} and \ref{sup_fig:resnet}) or by inferring an underlying coordinate system \cite{singer2012vector}. Such approaches are increasingly utilised for mechanistic interpretability \cite{wurgaft2026manifold, pellegrino2025rnns, choufeature, jacot2018neural}}.

\rd{\textbf{Future work.} Different similarity metrics entail different assumptions which may better or worse serve any given use case. MSA could be extended to other SPD matrix similarities than the SR, which relax or tighten different assumptions on the invariances. While we focused on trained rich-lazy models, future applications could probe how intrinsic geometry changes over learning, for example over silent alignment phenomena or saddle-to-saddle learning dynamics\cite{atanasov2021neural,jacot2021saddle}. Moreover, while our applications to diffusion models were primarily meant to highlight scalability, future work could use MSA to compare diffusion models with similar implicit biases\cite{brandon2026balance,karczewski2025spacetime}. Finally, the relationship between intrinsic geometry and downstream outputs is an open question, where MSA could be used to compare pullback metrics between output space and latent representations. \cite{wurgaft2026manifold}}

\clearpage 

\section*{Acknowledgments}

\textbf{Contributions and thanks.} We would like to thank David O'Neill the results he obtained during his Master's as UCL, which inspired the development of MSA. We would also like to thank Alexandre Vérine for discussion on the diffusion models section\rd{, and Ludovico Furlanetto for reviewing the mathematical derivations of the manuscript.} Finally, we would like to thank the Gatsby Unit for their early feedback on this work. 

\rd{\textbf{Funding statement.} This work was supported by the Agence Nationale de la Recherche (ANR-23-IACL-0008, ANR-26-CE23-2356-01, ANR-17-910 EURE-0017) and the European Research Council (Grant No. 101170541). The authors declare no competing interests.}

\AtNextBibliography{\small}
\printbibliography[
    segment=\therefsegment,
]

\end{refsegment}

\clearpage

\appendix

\renewcommand{\thesection}{A\arabic{section}}
\renewcommand{\thesubsection}{A\arabic{section}.\arabic{subsection}}

\renewcommand{\thefigure}{A\arabic{figure}}
\renewcommand{\thetable}{A\arabic{table}}

\renewcommand{\theequation}{A\arabic{equation}}

\renewcommand{\thetheorem}{A\arabic{theorem}}

\onecolumn


{\LARGE \bf Appendix}

\section{Metric similarity analysis} \label{app:maths}

\begin{refsegment}

\subsection{Pseudodistance function}

\begin{proof}[Proof of theorem \ref{thm:sr_distance}]\label{proof:sr_distance}
    Let $G, G', G''\in\mathbb{R}^{m\times m}$ be positive definite matrices. We will denote $\lambda_i^{(G,G')}$ to be the $i$th generalized eigenvalue of $G$ and $G'$ with corresponding eigenvector $\mathbf{v}_i^{(G,G')}$, satisfying: 
    \[
        G \mathbf{v}_i^{(G,G')} = \lambda_i^{(G,G')} G' \mathbf{v}_i^{(G,G')}.
    \]
    By convention we order the eigenvalues from largest to smallest: 
    \[
        \lambda_1^{(G,G')} \geq \dots \geq \lambda_m^{(G,G')}>0.
    \]
    
    We are now equipped to prove that the spectral ratio satisfies the three defining properties of a pseudodistance function.\footnote{\rd{Note that this, and the invariances below, follows from the fact that the spectral ratio is a bounded, monotonic reparameterization of the Hilbert metric: $d_\mathrm{SR}(G,G') =1 - e ^{-d_\mathrm{H}(G,G') /2}$ where $d_\mathrm{H}(G,G')=\log(\lambda_\mathrm{max}/\lambda_\mathrm{min})$. Here, we provide self-contained proofs that only require basic linear algebra.}}
    \begin{enumerate}[left=0pt,topsep=0pt,partopsep=0pt]
        \item \textit{Separation:} In this case, \rd{the generalized eigenvalue problem is $G\mathbf{v} = \lambda G\mathbf{v}$, such that all eigenvalues are $\lambda =1$}. 
        Therefore, $d_\sr(G, G) = 1 - \sqrt{\frac{1}{1}} = 0$.
        \item \textit{Symmetry:} Each generalized eigenvalue for $(G,G')$ corresponds to a generalized eigenvalue for $(G',G)$. Since the eigenvalues are nonzero:
        \[
        G \mathbf{v}_i^{(G,G')} = \lambda_i^{(G,G')} G' \mathbf{v}_i^{(G,G')} \quad \Rightarrow \quad
            \frac{1}{\lambda_i^{(G,G')}} G \mathbf{v}_i^{(G,G')} =  G'\mathbf{v}_i^{(G,G')}
        \]
        Since we define the generalized eigenvalues in decreasing order of magnitude, this means $\lambda_i^{(G',G)} = 1/(\lambda_{m-i+1}^{(G,G')})$. Thus,
        \[
            d_\sr(G',G) = 1- \sqrt{\frac{\lambda_m^{(G',G)}}{\lambda_1^{(G',G)}}} = 1- \sqrt{\frac{1/(\lambda_1^{(G,G')})}{1/(\lambda_m^{(G,G')})}} = 1- \sqrt{\frac{ \lambda_m^{(G,G')} }{ \lambda_1^{(G,G')} }} = d_\sr(G,G').
        \]
        \item \textit{Triangle inequality:} We start by defining $\Delta d_\sr =  d_\sr(G,G') + d_\sr(G',G'') - d_\sr(G,G'') $. It then suffices to prove that $\Delta d_\sr \geq 0$. \rd{
        To this end, a useful fact is that if $A$ and $B$ are SPD, then their generalized eigenvalues satisfy $\lambda_i^{(A,B)} = \lambda_i^{((B)^{-1/2}A(B)^{-1/2},I)}$. This follows from:
        \[
        A\mathbf{v} = \lambda B \mathbf{v} \quad\iff \quad (B)^{-1/2}A(B)^{-1/2}\mathbf{w} = \lambda\mathbf{w}, \quad \mathbf{w}=B^{1/2}\mathbf{v}
        \]
        Similarly, $\lambda_i^{(A,B)} = \lambda_i^{((C)^{-1/2}A(C)^{-1/2},(C)^{-1/2}B(C)^{-1/2})}$ for SPD $C$. Using these properties, as well as symmetry, we can now rewrite:
        \begin{align*}
            \Delta d_\sr &= d_\sr(G,G') + d_\sr(G'',G') - d_\sr(G,G'') \\
            &= d_\sr(\underbrace{(G')^{-1/2}G(G')^{-1/2}}_{:=\tilde G},I_m) + d_\sr(\underbrace{(G')^{-1/2}G''(G')^{-1/2}}_{:=\tilde G''},I_m)  \\ &  \qquad  \qquad \qquad - d_\sr(\underbrace{(G')^{-1/2}G(G')^{-1/2}}_{:=\tilde G},\underbrace{(G')^{-1/2}G''(G')^{-1/2}}_{:=\tilde G''})
        \end{align*}}
        To prove that $\Delta d_\sr \geq 0$ we will need to understand how the standard eigenvalues of $\tilde G$ and $\tilde G''$ relate to their generalized eigenvalues. For clarity we write $\mathbf{v}_i = \mathbf{v}_i^{(\tilde G,\tilde G'')}$ for the remainder of the proof. Then, we can write the generalized eigenvalues as:
        \begin{align*}
            \tilde G \mathbf{v}_i = \lambda_i^{(\tilde G, \tilde G'')} \tilde G'' \mathbf{v}_i \quad \Rightarrow \quad \lambda_i^{(\tilde G, \tilde G'')} = \frac{\mathbf{v}_i^\top \tilde G \mathbf{v}_i}{\mathbf{v}_i^\top \tilde G'' \mathbf{v}_i} = \frac{R(\tilde G,\mathbf{v}_i)}{R(\tilde G'',\mathbf{v}_i)}
        \end{align*}
        where $R(A,\mathbf{v})$ is the Rayleigh quotient.  Then, the spectral ratio distance can be written in terms of Rayleigh quotients:
        \begin{align*}
            d_\sr(\tilde G, \tilde G'') &= 1 - \sqrt{ \frac{R(\tilde G,\mathbf{v}_m)}{R(\tilde G'',\mathbf{v}_m)} \cdot \frac{R(\tilde G'',\mathbf{v}_1)}{R(\tilde G,\mathbf{v}_1)} } 
        \end{align*}
        Recall that the Rayleigh quotient is bounded above and below by the standard eigenvalues: $\lambda_m^{(A,I_m)} \leq R(A,\mathbf{v}) \leq \lambda_1^{(A,I_m)}$ for any $\mathbf{v}$. We can use this to bound the ratio of Rayleigh quotients:
        \rd{\begin{align*}
             \frac{R(\tilde G,\mathbf{v}_m)}{R(\tilde G'',\mathbf{v}_m)} \cdot \frac{R(\tilde G'',\mathbf{v}_1)}{R(\tilde G,\mathbf{v}_1)}  \geq  \frac{\lambda_m^{(\tilde G,I_m)}}{\lambda_1^{(\tilde G'',I_m)}} \cdot \frac{\lambda_m^{(\tilde G'',I_m)}}{\lambda_1^{(\tilde G,I_m)}}= \left( \frac{b}{a}\cdot\frac{b''}{a''} \right)^2
        \end{align*}}
        where we have defined the following quantities for clarity:
        \begin{align*}
            a&=\sqrt{\lambda_1^{(\tilde G,I_m)}},\phantom{,} \quad b=\sqrt{\lambda_m^{(\tilde G,I_m)}}\\
            a''&=\sqrt{\lambda_1^{(\tilde G'',I_m)}}, \quad b''=\sqrt{\lambda_m^{(\tilde G'',I_m)}}
        \end{align*}
        Putting it all together, we can write (or bound) all three terms in $\Delta d_\sr$ as: 
        \begin{align*}
            d_\sr(\tilde G,I_m)&=1-\frac{b}{a} \\
            d_\sr(\tilde G'',I_m)&=1-\frac{b''}{a''} \\
            d_\sr(\tilde G, \tilde G'') &\leq 1 -\frac{b}{a}\cdot\frac{b''}{a''}
        \end{align*}
        Finally, since $b/a \leq 1$ and $b''/a'' \leq 1$, 
        \begin{align*}
            0 & \leq \left( 1 - \frac{b}{a} \right)\left( 1 - \frac{b''}{a''} \right)\\
            & = 1 - \frac{b}{a} - \frac{b''}{a''} + \frac{b}{a}\cdot \frac{b''}{a''}\\
            & = \left( 1 - \frac{b}{a}\right) + \left (1 - \frac{b''}{a''}\right) - \left(1 - \frac{b}{a}\cdot \frac{b''}{a''}\right) \\
            &\leq d_\sr(\tilde G,I_m) + d_\sr(\tilde G'',I_m) - d_\sr(\tilde G,\tilde G'') = \Delta d_\sr,
        \end{align*}
        which concludes the proof.
        
        \end{enumerate}

\end{proof}

\begin{proof}[Proof of theorem \ref{thm:msa_distance}]\label{proof:msa_distance}
    The \textit{separation} and \textit{symmetry} immediately follow from Theorem \ref{thm:sr_distance} and the linearity of integration over a manifold. As for the \textit{triangle inequality}:
    \rd{\begin{align*}
        \Delta d_\msa &= d_\msa(\varphi^1, \varphi^2) + d_\msa(\varphi^2, \varphi^3) - d_\msa(\varphi^1, \varphi^3) \\
        &= \frac{1}{\mathrm{Vol}_g(\mathcal{M})}\int_\mathcal{M} d_\sr(G^{\varphi^1}(p), G^{\varphi^2}(p)) + d_\sr(G^{\varphi^2}(p), G^{\varphi^3}(p)) \\
        & \qquad \qquad \qquad \qquad \qquad \qquad\qquad \qquad \quad- d_\sr(G^{\varphi^1}(p), G^{\varphi^3}(p)) \mathrm{dvol}_g(p)\\
        &=\frac{1}{\mathrm{Vol}_g(\mathcal{M})}\int_\mathcal{M} \Delta d_\sr(p) \mathrm{dvol}_g(p).
    \end{align*}}
    In the proof of Theorem \ref{thm:sr_distance} we show that $\Delta d_\sr\geq 0$, and the integral of a non-negative number is itself non-negative. Hence \rd{the triangle inequality is satisfied}:
    \[
        0\leq \Delta d_\msa \quad \Longleftrightarrow \quad d_\msa(\varphi^1, \varphi^3)\leq d_\msa(\varphi^1, \varphi^2) + d_\msa(\varphi^2, \varphi^3).
    \]
\end{proof}

\begin{remark}
    Above, we have proven that the spectral ratio (SR) and metric similarity analysis (MSA) satisfy the separation property: $G=G'\Rightarrow d(G,G')=0$, which is a requirement for a pseudodistance. Technically, a proper distance also requires that $d(G,G')=0\Rightarrow G=G'$, which does not generally hold for MSA nor for the SR. In the main text, we drop the `pseudo' for readability, as is standard in the literature on similarity measures. Nevertheless, we note that a pseudodistance can always be converted into a true distance by considering the quotient space defined by the equivalence relation $G \sim G'\iff d(G,G')=0$. Below we discuss why the fact that MSA is a pseudodistance, rather than a distance, is a desirable property, as it enables MSA to be scale invariant.
\end{remark} 

\subsection{Scale and conformal invariance}
\begin{proposition}
    The spectral ratio is scale invariant: $d_\sr(sG,G') = d_\sr(G,G')$ for any $s>0$.
\end{proposition}
\begin{proof}
Let $\lambda_1,\dots,\lambda_m$ be the generalized eigenvalues of positive definite $G, G' \in \mathbb{R}^{m \times m}$, and let $s>0$ be a scalar. Then, $\lambda_i$ satisfy the following generalized eigenvalue equation:
\[
    G \mathbf{v}_i = \lambda_i G'\mathbf{v}_i \\
\]
where $\mathbf{v}_i\in \mathbb{R}^m$ are the eigenvectors. Multiplying both sides by $s$ yields:
\[    
    (s G) \mathbf{v}_i = (s \lambda_i)G'\mathbf{v}_i \\
\]
Thus, $\tilde \lambda_i = s\lambda_i$ are the generalized eigenvalues of $sG$ and $G'$. Then, the spectral ratio is:
\[
d_\sr(sG,G') = \sqrt{\frac{\tilde \lambda_m}{\tilde \lambda_1}} = \sqrt{\frac{\lambda_m}{\lambda_1}} = d_\sr(G,G').
\]
\end{proof}

A direct consequence of the scale invariance of the spectral ratio is that MSA is also scale invariant. In fact, since MSA is defined by integrating the spectral ratio, it is invariant to pointwise rescalings of the pullback metric at each point $p\in\mathcal{M}$, that is, of the form $G(p) \mapsto \Omega(p) G(p)$, where $\Omega: \mathcal{M} \to \mathbb{R}_{+}$. Thus, MSA is invariant under conformal transformations of the metric.

\begin{corollary}
    MSA is conformally invariant (hence, scale invariant).
\end{corollary}

\rd{Scale invariance is a natural requirement for similarity measures, as a global rescaling of the activations should not affect the similarity of two networks. Conformal invariance extends this to input-dependent rescalings: i.e., scenarios in which two networks' metrics differ by a position-dependent scalar, while preserving local angles. A consequence of this conformal invariance is that MSA is not insightful for one-dimensional manifolds, where no angular structure exists and there is only scale. }

\subsection{Invariance to state space rotations}

\begin{proof}[Proof of proposition \ref{thm:rotation_invariance}]\label{proof:rotation_invariance}
    It suffices to show that the metric itself is invariant to state-space rotation. We consider a network $\varphi^i:\mathbb{R}^{n_\text{in}}\rightarrow\mathbb{R}^{n_i}$ with $m$-dimensional input manifold $\mathcal{M}$ embedded through the mapping $\psi:\mathcal{M} \to \mathbb{R}^{n_\text{in}}$. Let $(U,\chart)$ be a smooth chart on $\mathcal{M}$ where $U \subset \mathcal{M}$ is an open set with $\chart: U \to \chart(U) \subseteq \mathbb{R}^m$ a diffeomorphism. The pullback metric at $p\in U$ can be expressed in local coordinates as the following product of Jacobians:
\[
    G^{\varphi^i}(p) = J_{\varphi^i\circ\psi\circ\chart^{-1} }^\top J_{\varphi^i\circ\psi \circ \chart^{-1}}
    = J_{\psi\circ \chart^{-1}}^\top J_{\varphi^i}^\top J_{\varphi^i}J_{\psi\circ \chart^{-1}}
\]  
    Let $\tilde G^{\varphi^i}(p)$ be the corresponding metric after rotating in state space by $Q_i\in O(n_i)$. Then:
    \begin{align*}
        \tilde G^{\varphi^i}(p) =&  J_{Q_i\circ\varphi^i\circ\psi \circ \chart^{-1} }^\top J_{Q_i\circ\varphi^i\circ\psi \circ \chart^{-1}}\\
        =&J_{\psi \circ \chart^{-1}}^\top J_{\varphi^i}^\top Q_i^\top  Q_iJ_{\varphi^i}J_{\psi\circ \chart^{-1}}\\
        =& J_{\psi \circ \chart^{-1}}^\top J_{\varphi^i}^\top J_{\varphi^i}J_{\psi\circ \chart^{-1}} = G^{\varphi^i}(p)
    \end{align*}
    Where $J_f(p)$ is the Jacobian of the function $f$ evaluated at $p$. Thus:
    \begin{align*}
        d_\msa(Q_1\circ \varphi^1, Q_2\circ \varphi^2) &= \frac{1}{v}\int_\mathcal{M} d_\sr(\tilde G^{\varphi^1}(p), \tilde G^{\varphi^2}(p))\mathrm{d}V(p) \\&= \frac{1}{v}\int_\mathcal{M} d_\sr(G^{\varphi^1}(p), G^{\varphi^2}(p))\mathrm{d}V(p) =  d_\msa(\varphi^1, \varphi^2)
    \end{align*}
    Hence MSA is invariant to neural network hidden-layer state-space rotations. 
\end{proof}
This result can be viewed as a special case of a more general proposition: MSA is invariant to state-space transformations which are point-wise isometries with respect to the Euclidean dot product. The proof follows mutatis mutandis, with $\tilde G^{\varphi^i}(p)$ the metric under the isometry.

We note that these arguments do not rely on having networks with the same hidden layer sizes, as they can be independently rotated. We can get a stronger invariance if we assume that both neural networks are simultaneously acted upon by the same transformation. So far we have considered invariances to transformations of the hidden-layer state-space, but MSA is also invariant to simultaneous transformations of the input manifold, and in particular changes of local coordinates. 

\subsection{Invariance to choice of coordinates}

 \begin{proof}[Proof of proposition \ref{thm:coordinate_invariance}]\label{proof:coordinate_invariance}
We consider two networks $\varphi^1:\mathbb{R}^{n_\text{in}}\rightarrow\mathbb{R}^{n_1}$, $\varphi^2:\mathbb{R}^{n_\text{in}}\rightarrow\mathbb{R}^{n_2}$, with $m$-dimensional input manifold $\mathcal{M}$ embedded through the mapping $\psi:\mathcal{M} \to \mathbb{R}^{n_\text{in}}$. Let $(U_j,\chart_j)$  and  $(U_k,\chart_k)$ be two smooth charts on $\mathcal{M}$ where $U_j,U_k \subset \mathcal{M}$ are open sets with $U_j \cap U_k \neq \emptyset $ and $\chart_j: U_j \to \chart_j(U_j) \subseteq \mathbb{R}^m$, $\chart_k: U_k \to \chart_k(U_k) \subseteq \mathbb{R}^m$ are diffeomorphisms onto their images. The pullback metric at $p\in U_j$ can be expressed in local coordinates as the following product of Jacobians:
\[
    G^{\varphi^i}(p) = J_{\varphi^i\circ\psi\circ\chart_j^{-1} }^\top J_{\varphi^i\circ\psi \circ \chart_j^{-1}}
    = J_{\psi\circ \chart_j^{-1}}^\top J_{\varphi^i}^\top J_{\varphi^i}J_{\psi\circ \chart_j^{-1}}
\]
Now suppose we wish to change the local coordinates to use chart $\chart_k$ at $p \in U_j \cap U_k$ instead of $\chart_j$. We define the following transition mapping:
\[
   \chart_k(U_j \cap U_k) \xrightarrow{\chart_k^{-1}} U_j \cap U_k \xrightarrow{\chart_j} \chart_j(U_j \cap U_k)
\]
The transition function is $\chart_j \circ \chart_k^{-1}$, which maps between open subsets of $\mathbb{R}^m$, and we denote its Jacobian by $J_{kj}(p)\in\mathbb{R}^{m\times m}$. Here, unlike in the proof of proposition~\ref{thm:rotation_invariance}, the metric will change under this transformation. Denoting $\tilde G^{\varphi^i}(p)$ the new metric in local coordinates (associated with $\chart_k$) at $p$:
\[
    \tilde G^{\varphi^i}(p) = J_{\varphi^i\circ\psi\circ\chart_j^{-1} \circ \chart_j \circ \chart_k^{-1} }^\top J_{\varphi^i\circ\psi\circ\chart_j^{-1} \circ \chart_j \circ \chart_k^{-1}}
    = J_{kj}^\top J_{\psi\circ \chart_j^{-1}}^\top J_{\varphi^i}^\top J_{\varphi^i}J_{\psi\circ \chart_j^{-1}}J_{kj}
\]
which is not equal to $G^{\varphi^i}(p)$ in general. For notational clarity we will write $\mathbf{J}^i=J_{\varphi^i}J_{\psi\circ \chart_j^{-1}}$, so that 
\[
    G^{\varphi^i}(p) = (\mathbf{J}^i)^\top \mathbf{J}^i, \qquad \tilde G^{\varphi^i}(p) = J_{kj}^\top(\mathbf{J}^i)^\top \mathbf{J}^iJ_{kj}
\]
Now, the generalised eigenvalue problem in the new local coordinates becomes:
\begin{align*}
    \tilde G^{\varphi^1}(p)\mathbf{v} = \lambda \tilde G^{\varphi^2}(p)\mathbf{v} \quad \Rightarrow \quad
    J_{kj}^\top(\mathbf{J}^1)^\top \mathbf{J}^1J_{kj}\mathbf{v} =\lambda J_{kj}^\top(\mathbf{J}^2)^\top \mathbf{J}^2J_{kj}\mathbf{v}.
\end{align*}
Since the transition function is invertible, we can left-multiply by $J_{kj}^{-\top}$ to get:
\[
    (\mathbf{J}^1)^\top \mathbf{J}^1\tilde{\mathbf{v}} =\lambda (\mathbf{J}^2)^\top \mathbf{J}^2\tilde{\mathbf{v}}
\]
where we let $\tilde{\mathbf{v}} = J_{kj}\mathbf{v}$. Thus, $\tilde{\mathbf{v}}/\|\tilde{\mathbf{v}}\|$ is a generalized eigenvector for the original local coordinates, with generalized eigenvalue $\lambda$. This tells us although the generalised eigenvectors may change under local coordinates, the generalised eigenvalues don't. Since the spectral ratio is only dependent on the generalised eigenvalues, MSA is independent of local coordinate changes.
 \end{proof}

\rd{\subsection{The difference between intrinsic vs. extrinsic similarity can be arbitrarily large}} \label{app:arbitrarilylarge}

\rd{\textbf{Example 1: Low extrinsic similarity, high intrinsic similarity}.

\emph{Intuition.} In Figure \ref{fig:intrinsic-extrinsic} (left), we saw an example where a Euclidean plane and a Swiss roll have low similarity under Procrustes, but high similarity under MSA. We can generalize this example to construct a sequence of maps such that $d_\mathrm{Proc}\to 1$ while $d_\mathrm{MSA}=0$. 

Let $\mathcal{M} = \left[-1,1\right ]^2$, and define the maps $\varphi_k^1,\varphi_k^2:\mathcal{M}\to\mathbb{R}^3$ as:
\begin{align*}
    &\varphi_k^1(\mathbf{x}) = (k\pi  x_1, x_2, 0) \\
    &\varphi_k^2(\mathbf{x}) = (\cos(k\pi x_1),\sin(k\pi x_1), x_2)
\end{align*}
where $k\in\mathbb{Z}_+$. As we increase $k$, this will thus correspond to stretching the plane defined by $\varphi^1_k$ in $x_1$, and rolling the cylinder defined by $\varphi_k^2$ tighter and tighter. Since rolling the plane into the cylinder is an isometric mapping, it does not affect the pullback metric. Thus, $G^{\varphi_k^1}=G^{\varphi_k^2}$, so $d_\mathrm{MSA}=0$. However, as $k\to \infty$, the points $\varphi^1_k(\mathbf{x})$ and $\varphi^2_k(\mathbf{x})$ will become increasing uncorrelated, so that $d_\mathrm{Proc}\to 1$. Thus the Procrustes similarity can be made arbitrarily close to zero, while the MSA similarity is exactly 1.

\emph{Note 1.} In this example we use $d_\mathrm{Proc}$ to refer to the scaled Procrustes distance: this symmetric variant is bounded by 1 and allows an optimal rescaling of the representation -- otherwise, it would be trivial to make the Procrustes distance diverge simply because of a distance in size. 

\emph{Note 2.} In this setup, $\varphi_k^2$ wraps around itself exactly $k$ times, and is self-intersecting. Thus $\varphi^2_k$ is not  embedding; nevertheless, it is an immersion, so the pullback is well defined. In principle we could reproduce the same example with a tightly rolled Swiss roll to avoid self-intersections, but the algebra would be more involved.

\emph{Proof}. We can now straightforwardly calculate the Jacobian and pullback metric of $\varphi_k^1$ as:
\[
J_{\varphi^1_k} = \begin{bmatrix}
k\pi & 0 \\
0 & 1 \\
0 & 0
\end{bmatrix} \quad \implies \quad  G^{\varphi_k^1} = \begin{bmatrix}
(k\pi)^2 & 0 \\
0 & 1 
\end{bmatrix} 
\]
and for $\varphi_k^2$:
\[
J_{\varphi^2_k} = \begin{bmatrix}
-\lambda \sin(k\pi x_1) & 0 \\
\lambda \cos(k\pi x_1) & 0 \\
0 & 1
\end{bmatrix} \quad \implies \quad  G^{\varphi^2_k} =  \begin{bmatrix}
(k\pi)^2 & 0 \\
0 & 1 
\end{bmatrix} 
\]
Therefore we see that the two maps have the same pullback metrics, and thus $d_\mathrm{MSA}(\varphi_k^1,\varphi_k^2) = 0$.

We next consider the scaled Procrustes distance:
\[
d_\mathrm{Proc}(Y_1,Y_2) = \min_{c \in \mathbb{R},Q\in O(n)} \frac{\| cQY_1-Y_2 \|_F}{\|Y_2 \|_F}
\]
where $Y_\ell\in\mathbb{R}^{3\times q}$ is the matrix of $q$ samples, whose $i$th column is $\varphi_k^\ell(\mathbf{x}^{(i)})$ for sample $\mathbf{x}^{(i)}\in \mathcal{M}$. To obtain a closed form solution, note that we can rewrite the scaled Procrustes distance as:
\[
\|Y_2\|_F^{-2} \min_{c,Q} f(c,Q), \qquad f(c,Q) =  \| cQY_1 - Y_2\|_F^2
\]
We reduce the optimization to first solving for optimal $c^*$ as a function of $Q$, enabling us to turn the dual optimization $\min_{c,Q} f(c,Q)$ into a single optimization over $O(n)$: $\min_Q f(c^*(Q),Q)$. To do this we first expand $f$ as a quadratic in $c$:
\[
f(c,Q) = c^2 \|Y_1\|_F^2 - 2c\langle QY_1,Y_2\rangle +\| Y_2\|_F^2
\]
Solving $\frac{\partial f}{\partial c}=0$, we obtain the following optimal $c$ as a function of $Q$:
\[
c^*(Q) = \frac{\langle QY_1,Y_2\rangle}{\|Y_1\|_F^2} \quad \implies \quad f(c^*(Q),Q) = \| Y_2\|_F^2-\frac{\langle QY_1,Y_2\rangle^2}{\|Y_1\|_F^2}
\]
Therefore minimizing $f$ entails maximizing $|\langle Q Y_1,Y_1\rangle|$ .  Using the cyclic property of the trace:
\[
|\langle Q Y_1,Y_2 \rangle |= |\mathrm{tr}( Q^\top Y_2Y_1^\top)| \leq \sum_{i=1}^n \sigma_i(Q) \sigma_i(Y_2Y_1^\top)
\]
where the second step follows from Von Neumann's trace inequality, with $\sigma_i(A)$ the $i$th singular value of $A$. In particular, equality holds iff $Q$ and $Y_2Y_1^\top$ share the same left and right singular vectors. Since we wish to maximize this quantity we can in particular take $Q=UV^\top$, where $U$, $V$ are the left and right singular vectors of $Y_2Y_1^\top$. Then the bound is attained, so that:
\[
\max_{Q\in O(n)} |\langle QY_1,Y_2 \rangle | = \| Y_2 Y_1^\top \|_*
\]
where $\| \cdot \|_*$ denotes the nuclear norm. Therefore, the scaled Procrustes distance has the following closed-form solution:
\[
(d_\mathrm{Proc}(Y_1,Y_2))^2 = \frac{1}{\|Y_2\|_F^2}\left( \min_{Q \in O(n)} f(c^*(Q),Q)\right) = 1 - \frac{\| Y_2Y_1^\top \|^2_*}{\|Y_1\|_F^2 \| Y_2 \|_F^2}
\]
As we take the number of samples $q \to \infty$ uniformly the , the scaled Frobenius norms in the denominator converge to integral:
\[
\frac{1}q\|Y_\ell \|_F^2 = \frac{1}{q} \sum_{i=1}^q\|(Y_\ell)_{\cdot,j}\|_2^2 \to \iint_{[-1,1]^2} \|\varphi_k^\ell(\mathbf{x})\|_2^2 \, dx_1 dx_2
\]
Note there is a factor of $\mathrm{Vol_g(\mathcal{M}) =4}$ missing but these will cancel in the final ratio. Now we can straightforwardly evaluate these integrals for $\varphi_k^1$:
\[
\frac{1}q\|Y_1 \|_F^2 \,\longrightarrow   \,  \iint_{[-1,1]^2} \|\varphi^1_k(\mathbf{x})\|_2^2 \, dx_1 dx_2 = \int_{-1}^1\int_{-1}^1 \left( k^2 \pi^2x_1^2 + x_2^2 \right) \, dx_1 dx_2 = \frac{4}{3}(k^2\pi^2+1)
\]
And for $\varphi_k^2$:
\[
\frac{1}q\|Y_2 \|_F^2 \,\longrightarrow   \,     \iint_{[-1,1]^2} \|\varphi_k^2(\mathbf{x})\|_2^2 \, dx_1 dx_2 = \int_{-1}^1\int_{-1}^1 \left( 1 + x_2^2 \right) \, dx_1 dx_2 = \frac{16}{3}
\]
Finally, the term in the numerator requires the singular values of the matrix $Y_2Y_1^\top \in \mathbb{R}^{3 \times 3}$. To evaluate this, note that entry in the matrix converges to an integral as well:
\[
\frac{1}q(Y_2 Y_1^\top)_{ij}  = \frac{1}{q}\sum_{i'=1}^q (Y_2)_{ii'}(Y_1)_{ji'} \longrightarrow \iint_{[-1,1]^2} (\varphi_k^1(\mathbf{x}))_j\,( \varphi_k^2(\mathbf{x}))_i \, dx_1 dx_2 
\]
We can evaluate these integrals one-by-one. Most evaluate to zero, except for the following two entries:
\[
\frac{1}q(Y_2 Y_1^\top)_{2,1} \,\longrightarrow   \, \iint_{[-1,1]^2}(k\pi)\, x_1\sin (\pi x_1 dx_1 dx_2=\frac{4}{k\pi} \sin(k\pi) - 4 \cos(k \pi) = (-1)^{k+1} 4
\]
\[
\frac{1}q(Y_2 Y_1^\top)_{3,2} \,\longrightarrow   \, \iint_{[-1,1]^2}(k\pi)\, x_2^2\,dx_1 dx_2=\frac{4}{3} 
\]
Thus the numerator converges to:
\[
\frac{1}{q} \| Y_2 Y_1^\top\|_*\to  \left \| \begin{pmatrix}
    0 & 0 & 0\\ (-1)^{k+1}4 & 0 & 0 \\ 0 & \frac{4}{3} & 0
\end{pmatrix} \right \|_* = \frac{16}{3}
\]
Putting this all together, we see that as $k\to\infty$ the Procrustes distance converges to 1:
\[
d_\mathrm{Proc}(\varphi_k^1,\varphi_k^2) = \sqrt{1-\frac{4}{k^2\pi^2+1}} \longrightarrow1
\]
\qed

\rd{\textbf{Example 2: High extrinsic similarity, low intrinsic similarity}}

\emph{Intuition.} In Figure \ref{fig:intrinsic-extrinsic} (right), we showed an example where a Euclidean and non-Euclidean plane have high extrinsic similarity under Procrustes but low intrinsic similarity under MSA. We can generalize this example by constructing a map that warps the input manifold with small-amplitude, high-frequency deviations from Euclidean coordinates. Then as the amplitude decreases we will show that $d_\mathrm{Proc}\to 0$ while $d_\mathrm{MSA}\to1$.

Let $\mathcal{M}=[0,1]^2$, and consider the maps ${\varphi}^1,\varphi^2_\epsilon:\mathcal{M}\to\mathbb{R}^3$ defined as:
\begin{align*}
&\varphi^1(\mathbf{x}) = (x_1,x_2,0) \\
&\varphi_\epsilon^2(\mathbf{x}) = \left(x_1 + \epsilon\cos^2\left(\frac{x_1}{\epsilon}\right), x_2 + \epsilon\sin^2\left(\frac{x_2}{ \epsilon^2}\right), 0 \right)
\end{align*}
As $\epsilon\to 0$, the amplitude of the sinusoidal components will decrease, so the points $\varphi_\epsilon^2{(\mathbf{x}})$ will converge to $\varphi^1{(\mathbf{x}})$, so that the Procrustes distance $d_\mathrm{Proc}\to 0$. However, the high-frequency warping leads to a highly anisotropic metric dependent on the position.  Thus $d_\mathrm{MSA} \to 1$, so that the MSA similarity can be made arbitrarily low while the Procrustes similarity can be arbitrarily high. 

\emph{Proof}. First note that $\|\varphi^{1}(\mathbf{x})-\varphi_\epsilon^2(\mathbf{x})\| <\sqrt{2}\epsilon$. Thus, ${\varphi^2_\epsilon}$ converges uniformly to $\varphi^1$ as $\epsilon \to 0$, so that the Procrustes dissimilarity will tend to $d_\mathrm{Proc}({\varphi^1},\varphi_\epsilon^2)\to0$. .

Next, the Jacobian and pullback metric can be derived as:
\[
J_{\varphi^1} = \begin{bmatrix}
1 & 0 \\
0 & 1 \\
0 & 0 \end{bmatrix}  
 \qquad \implies \qquad G^{\varphi^1} = \begin{bmatrix}
1 & 0 \\
0 & 1 \end{bmatrix}  
\]
and for $\varphi^2$:
 \begin{align*}
&J_{\varphi_\epsilon^2} = \begin{bmatrix}
1-2\sin\left(\frac{x_1}{\epsilon}\right)\cos\left(\frac{x_1}{\epsilon}\right) & 0 \\
0 & 1+\frac{2}{\epsilon}\sin\left(\frac{x_2}{\epsilon^2}\right)\cos\left(\frac{x_2}{\epsilon^2}\right) \\
0 & 0 \end{bmatrix}  \\
&\qquad \qquad \qquad\implies \quad G^{\varphi^2_\epsilon} = \begin{bmatrix}
\left(1-\sin\left(\frac{2x_1}{\epsilon}\right)\right)^2 & 0\\
0 & \left(1+\frac{1}{\epsilon}\sin\left(\frac{2x_2}{ \epsilon^2}\right)\right)^2
\end{bmatrix}
\end{align*}
The MSA distance between can be expressed as:
\[
d_\mathrm{MSA}(\varphi^1,\varphi_\epsilon^2) = 1 - I(\epsilon), \qquad I(\epsilon) = \int_0^1\int_0^1 \min\left(\frac{a(x_1)}{b(x_2)}, \frac{b(x_2)}{a(x_1)}\right) dx_1dx_2
\]
where $a(x_1)=|1-\sin(2x_1\epsilon^{-1})|$, $b(x_2)=|1+\epsilon^{-1}\sin(2x_2\epsilon^{-2})|$. The integrand is undefined when $a(x_1)=b(x_2)=0$, but this only occurs in a set of measure zero. To evaluate the integral, define:
\[ 
S = \{x_2 \in [0,1]:|\sin(2x_2/\epsilon^2)|<\sqrt{\epsilon} \}
\]
We will then partition $\mathcal{M}$ into subsets $[0,1]\times S$ and $[0,1]\times ([0,1]\setminus S)$ in order to evaluate their integrals separately:
\[
I(\epsilon) = I_{[0,1]\times S}(\epsilon) + I_{[0,1]\times ([0,1]\setminus S)}(\epsilon)
\]
First, note that within $[0,1]\times ([0,1]\setminus S)$, $b(x_2)$ can be lower bounded as:
\[
b(x_2) = \frac{|\epsilon+\sin(2x_2/\epsilon^2)|}{\epsilon} \geq \frac{\sqrt{\epsilon}-\epsilon}{\epsilon} = \frac{1-\sqrt{\epsilon}}{\sqrt{\epsilon}}
\]
where the first step uses the reverse triangle inequality. In particular, if $0<\epsilon<1$, then $\sqrt{\epsilon}>\epsilon$, and $b(x_2) > 0$. With $b(x_2)$ nonzero, we can obtain the following upper bound on the integrand:
\[
    \min\left(\frac{a(x_1)}{b(x_2)}, \frac{b(x_2)}{a(x_1)}\right) \leq \frac{a(x_1)}{b(x_2)} \leq \frac{2}{b(x_2)} \leq \frac{2\sqrt{\epsilon}}{1-\sqrt{\epsilon}}
\]
Then, the integral in this region is upper bounded by:
\[
I_{[0,1]\times ([0,1]\setminus S)}(\epsilon) \leq \frac{2\sqrt{\epsilon}}{1-\sqrt{\epsilon}}
\]
demonstrating that $I_{[0,1]\times ([0,1]\setminus S)}(\epsilon)\to 0$ as $\epsilon \to 0$.

Within ${[0,1]\times S}$, we can additionally show that the integral is small. Making the substitution $u=2x_2/\epsilon^2 $, the Lebesgue measure of $S$ is:
\[
|S| = \frac{\epsilon^2}{2}|\{ u \in [0,2\epsilon^{-2}]: |\sin u|<\sqrt{\epsilon}\} |
\]
The function $|\sin(u)|$ has period $\pi$. Thus within each period, $|\sin(u)|<\sqrt{\epsilon}$ on a set of measure $2\arcsin(\sqrt{\epsilon})$. Note that for $0\leq t<1$, $\arcsin(t)\leq \frac{\pi}{2} t$, so that $2\arcsin(\sqrt{\epsilon})\leq \pi \sqrt{\epsilon}$. The interval $[0,2\epsilon^{-2}]$ is covered by at most $2\epsilon^{-2}\pi^{-1} +1$ periods. Putting this together:
\[
|S|\leq \frac{\epsilon^2}{2} \cdot \left( \frac{2}{\pi\epsilon^2} +1 \right)  \cdot \pi \sqrt{\epsilon} = \left( 1 +\frac{\pi\epsilon^2}{2}\right) \sqrt{\epsilon}
\]
Thus the integral is bounded as:
\[
I_{{[0,1]\times S}} = \int_{{[0,1]\times S}} \min\left(\frac{a(x_1)}{b(x_2)}, \frac{b(x_2)}{a(x_1)}\right) dx_1 dx_2 \leq 1 \cdot |S| \leq \left( 1 +\frac{\pi\epsilon^2}{2}\right) \sqrt{\epsilon}
\]
Thus as $\epsilon \to 0$, $I_{[0,1]\times S}(\epsilon)\to 0$ as well. Together this shows that $d_\mathrm{MSA}(\varphi^1,\varphi_\epsilon^2) \to 1$, i.e., MSA similarity $\to 0$.
}
\qed

\rd{\subsection{Stability of MSA}}

\rd{\begin{proposition}
Let $\mathcal{M}$ be an $m$-dimensional manifold, and $\varphi^1, \varphi^2:\mathcal{M}\to\mathbb{R}^n$ be $C^1$ functions. Consider $C^1$ approximation functions $\hat{\varphi}^1, \hat{\varphi}^2:\mathcal{M}\to\mathbb{R}^n$ satisfying:
\[
\|\hat{\varphi}^i - \varphi^i\|_{\mathrm{Lip}} \leq \epsilon
\]
Assume there is some $c>0$ such that the maximum singular value of the Jacobian $J_{\varphi^i}(p)$ satisfies:
\[
\sigma_\mathrm{max}(J_{\varphi^i}(p)) \geq c   \qquad \mathrm{a.e.}
\]
Furthermore, assume that the average inverse spectral ratio similarity between the pullback metric $G^{\varphi^i}=J_{\varphi^i}^\top J_{\varphi^i}$ and $I_m$ is bounded over $\mathcal{M}$:
\[
s^{\varphi^i} := \frac{1}{\mathrm{Vol}_g(\mathcal{M})}\int_\mathcal{M} \frac{1}{1-d_\mathrm{SR}(G^{\varphi^i}(p),I_m)} \mathrm{dvol}_g(p) < \infty
\]
Then, the MSA distance is bounded as:
\[
\left| \mathrm{MSA}(\varphi^1, \varphi^2) - \mathrm{MSA}(\hat{\varphi}^1, \hat{\varphi}^2) \right| \leq \frac{2\epsilon}{c}  \left( s^{\varphi^1}+s^{\varphi^2} \right)
\]
In particular, this means the MSA estimation error is $\mathcal{O}(\epsilon)$
\end{proposition}
\begin{proof} 
Let $J_{\varphi^i},J_{\hat\varphi^i}$ denote the Jacobians of $\varphi_i,\hat \varphi_i$, respectively, and define $\bar J_{i} =  J_{\hat\varphi^i} - J_{\varphi^i}$. By the Lipschitz assumption, $\|\bar J_{i} \|_2 \leq \epsilon$.  By the triangle inequality, for any $\mathbf{v}\in \mathbb{R}^m$,
\[
\|  J_{\hat \varphi^i} \mathbf{v} \|_2 = \| ( J_{\varphi^i} + \bar J_{i}) \mathbf{v} \|_2  \leq  \|  J_{\varphi^i} \mathbf{v} \|_2 + \| \bar J_{i} \mathbf{v} \|_2  \leq \|  J_{\varphi^i} \mathbf{v} \|_2 + \epsilon \| \mathbf{v} \|_2  
\]
Similarly, by the reverse triangle inequality:
\[
\|  J_{\hat \varphi^i} \mathbf{v} \|_2 \geq \| J_{\varphi^i} \mathbf{v}\|_2 - \|\bar J_{i} \mathbf{v} \|_2 \geq  \|  J_{\varphi^i} \mathbf{v}\|_2 - \epsilon \|\mathbf{v} \|_2
\]
Define $\delta_i = \frac{\epsilon}{\sigma_\mathrm{min}(J_{\varphi^i})}$ (noting that $\sigma_\mathrm{min}(J_{\varphi^i})>0$ almost everywhere since  $s^{\varphi^i}<\infty$). Using the fact that $\|A \mathbf{v}\|_2 \geq \sigma_\mathrm{min}(A) \,  \| \mathbf{v}\|_2$, the bounds above can be combined to yield:
\[
(1-\delta_i) \, \| J_{\varphi^i}\mathbf{v}\| \leq \| J_{\hat \varphi^i} \mathbf{v} \|_2 \leq (1+\delta_i) \, \| J_{\varphi^i }\mathbf{v}\|_2
\]
Since the generalized eigenvalues are given by $\lambda_\mathrm{max} = \sup_{\mathbf{v}\neq 0} {\frac{\mathbf{v}^\top G^{\varphi^1}\mathbf{v}}{\mathbf{v}^\top G^{\varphi^2}\mathbf{v}}}$, and $\lambda_\mathrm{min} = \inf_{\mathbf{v}\neq 0} {\frac{\mathbf{v}^\top G^{\varphi^1}\mathbf{v}}{\mathbf{v}^\top G^{\varphi^2}\mathbf{v}}}$, the spectral ratio similarity can be expressed as:
\[
1-d_\mathrm{SR}(G^{\varphi^1},G^{\varphi^2}) = \frac{\inf_{\mathbf{v} \neq 0} \frac{\| J_{\varphi^1}\mathbf{v}\|_2}{\| J_{\varphi^2}\mathbf{v}\|_2} }{\sup_{\mathbf{v} \neq 0} \frac{\| J_{\varphi^1}\mathbf{v}\|_2}{\| J_{\varphi^2}\mathbf{v}\|_2} }
\]
In order to bound this quantity, we will first consider $d_\mathrm{SR}$ when perturbing  $G^{\varphi^1}$ to $G^{\hat\varphi^1}$ to get bound for $|d_\mathrm{SR}( G^{\varphi^1}, G^{\varphi^2}) - d_\mathrm{SR}( G^{\hat \varphi^1}, G^{\varphi^2})|$, then we will do the same to get  a bound for  $| d_\mathrm{SR}( G^{\hat \varphi^1}, G^{ \varphi^2})- d_\mathrm{SR}( G^{\hat \varphi^1}, G^{\hat \varphi^2})|$, and we will then apply the triangle inequality. 

First, consider the case $\delta_1<1$. Divide the bound above by $\|J_{\varphi^2} \mathbf{v} \|_2$ to get:
\[
(1-\delta_1) \,\frac{\|J_{\varphi^1} \mathbf{v} \|_2}{\|J_{\varphi^2} \mathbf{v} \|_2} \leq \frac{\|J_{\hat \varphi^1} \mathbf{v} \|_2}{\|J_{\varphi^2} \mathbf{v} \|_2} \leq (1+\delta_1) \frac{\|J_{\varphi^1} \mathbf{v} \|_2}{\|J_{\varphi^2} \mathbf{v} \|_2}
\]
In particular, we pass either the infimum or the supremum through the bound above. Combining this with the spectral ratio definition above, we may obtain:
\[
 \frac{ (1-\delta_1)\inf\frac{\|J_{\varphi^1} \mathbf{v} \|_2}{\|J_{\varphi^2} \mathbf{v} \|_2} }{(1+\delta_1)\sup \frac{\|J_{\varphi^1} \mathbf{v} \|_2}{\|J_{\varphi^2} \mathbf{v} \|_2}} \leq 1-d_\mathrm{SR}(G^{\hat\varphi^1},G^{\varphi^2}) \leq \frac{ (1+\delta_1)\inf \frac{\|J_{\varphi^1} \mathbf{v} \|_2}{\|J_{\varphi^2} \mathbf{v} \|_2} }{(1-\delta_1)\sup \frac{\|J_{\varphi^1} \mathbf{v} \|_2}{\|J_{\varphi^2} \mathbf{v} \|_2}} 
\]
which can also be written as:
\[
\frac{1-\delta_1}{1+\delta_1}(1-d_\mathrm{SR}(G^{\varphi^1},G^{\varphi^2})) \leq 1-d_\mathrm{SR}(G^{\hat\varphi^1},G^{\varphi^2}) \leq \frac{1+\delta_1}{1-\delta_1}(1-d_\mathrm{SR}(G^{\varphi^1},G^{\varphi^2}))
\]
Then using these inequalities, and considering separately the case where $d_\mathrm{SR}(G^{\varphi^1},G^{\varphi^2}) > d_\mathrm{SR}(G^{\hat\varphi^1},G^{\varphi^2})$ or $d_\mathrm{SR}(G^{\varphi^1},G^{\varphi^2}) <d_\mathrm{SR}(G^{\hat\varphi^1},G^{\varphi^2})$, we arrive at:
\[
|d_\mathrm{SR}(G^{\varphi^1},G^{\varphi^2}) -d_\mathrm{SR}(G^{\hat\varphi^1},G^{\varphi^2}) | \leq \frac{2\delta_1}{1+\delta_1} \leq \frac{2\epsilon}{\sigma_\mathrm{min}(J_{\varphi^1})}
\]
Note that the bound holds trivially if $\delta_1>1$ since $d_\mathrm{SR}\in [0,1]$. We can apply the same logic to then perturbing $J_{\varphi^2}$ to get:
\[
|d_\mathrm{SR}(G^{\hat\varphi^1},G^{\hat\varphi^2}) -d_\mathrm{SR}(G^{\hat\varphi^1},G^{\varphi^2}) | \leq \frac{2\epsilon}{\sigma_\mathrm{min}(J_{\varphi^2})}
\]
Combining these and using the triangle inequality finally gives us our desired bound:
\[
|d_\mathrm{SR}(G^{\hat\varphi^1},G^{\hat\varphi^2}) -d_\mathrm{SR}(G^{\varphi^1},G^{\varphi^2}) | \leq 2\epsilon \left ( \frac{1}{\sigma_\mathrm{min}(J_{\varphi^1})} + \frac{1}{\sigma_\mathrm{min}(J_{\varphi^2})} \right)
\]
We can rewrite this each factor in terms of $d_\mathrm{SR}(G^{\varphi^i},I_m)$ with a little algebra:
\[
\frac{1}{\sigma_\mathrm{min}(J_{\varphi^i})} = \frac{1}{\sigma_\mathrm{max}(J_{\varphi^i})(1-d_\mathrm{SR}(G^{\varphi^i},I_m))}  \leq \frac{1}{c(1-d_\mathrm{SR}(G^{\varphi^i},I_m))} 
\]
Note that this assumes the input metric is in Euclidean or represented in orthonormal coordinates, otherwise we may replace $I_m$ with $g$. Finally, integrating over $\mathcal{M}$, we obtain our final bound:
\begin{align*}    
|\mathrm{MSA}(\varphi^1,\varphi^2)-\mathrm{MSA}(\hat\varphi^1,\hat\varphi^2)| &\leq \frac{1}{\mathrm{Vol}_g(\mathcal{M})}\int_\mathcal{M} \left| d_\mathrm{SR}(G^{\hat\varphi^1},G^{\hat\varphi^2}) -d_\mathrm{SR}(G^{\varphi^1},G^{\varphi^2}) \right|\mathrm{dvol}_g(p)\\
& \leq \frac{2\epsilon}{c}\left( s^{\varphi^1}+s^{\varphi^2} \right)
\end{align*}
Thus the MSA estimation error is $\mathcal{O}(\epsilon)$.
\end{proof}
}

\subsection{\rd{Relationship to nonlinear CKA}} \label{app:kcka}

\rd{Topological RSA\cite{lin2024topology} proposes transforming the elements of representational dissimilarity matrices with a monotonic, saturating function, with the aim of defining a similarity between the topological structure of neural representations independent of their embedding geometries. Williams, 2024\cite{williams2024equivalence} showed that this approach is equivalent to nonlinear CKA with a standard radial basis function (RBF) kernel. Below, we show that in the limit of dense sampling and small kernel width, RBF-CKA is equivalent to integrating a function of the pullback metrics over $\mathcal{M}$, establishing a new mathematical connection between nonlinear CKA and MSA.}

\rd{\begin{proposition}[CKA formulated via the pullback metric]\label{thm:kcka}
    Let $\mathcal{M}$ be a compact $m$-dimensional manifold, and let $\varphi^1,\varphi^2:\mathcal{M}\rightarrow\mathbb{R}^{n}$ be two neural network functions with corresponding pullback metrics $G^{\varphi^1}(p),G^{\varphi^2}(p)\in\mathbb{R}^{m \times m}$. Consider nonlinear CKA with a radial basis function (RBF) kernel:
    \[
    k(\mathbf{x},\mathbf{x}')=\exp\left(-\frac{\left\| \mathbf{x}-\mathbf{x}' \right\|_2^2}{2\alpha^2} \right), \quad \mathbf{x},\mathbf{x}'\in\mathbb{R}^n
    \]
    In the limit of small kernel width ($\alpha \to 0$) and dense sampling ($q\to\infty$, assuming a probability measure fully supported on $\mathcal{M}$), the CKA similarity is given by:
    \[
    \mathrm{CKA}(\varphi^1,\varphi^2) \propto \int_\mathcal{M} \frac{1}{\sqrt{\det \left(\bar G(p)\right)}}\, \mathrm{dvol}_g(p)
    \]
    where $\bar G(p) = \frac{1}{2}(G^{\varphi^1}(p)+G^{\varphi^2}(p))$.
    ~\hfill \hyperref[proof:kcka]\PROOF
\end{proposition}}

\rd{\textbf{CKA does not respect the intrinsic geometry of Riemannian metrics}. Proposition \ref{thm:kcka} establishes that in the limit of dense sampling and small kernel width, the CKA similarity can be formulated as an integral over $\mathcal{M}$ of a function $f:\mathbf{S}_m^{++}\times\mathbf{S}_m^{++}\to \mathbb{R}$ (where $\mathbf{S}_m^{++}$ denotes the SPD cone). Conceptually, the function $f(G,G')= 1/\sqrt{\det(\bar G)}$ quantifies the inverse volume scaling of the arithmetic mean $\bar G = (G+G')/2$ on a reference volume element. This establishes a conceptual relationship between the CKA similarity and MSA distance, which instead integrates the function $d_\mathrm{SR}(G,G')$ over $\mathcal{M}$.  However, in contrast to $d_\mathrm{SR}$, $f$ is not intrinsically geometric, which results in two disadvantages for similarity analysis.}

\rd{First,  a key difference between CKA and MSA is that $f$ is not pointwise coordinate invariant. Under a change of local coordinates with Jacobian $J\in\mathbb{R}^{m\times m}$, the pullback metrics in the new coordinates  become:
\[
    \tilde G = J^\top G J, \quad \tilde G' = J^\top G' J
\]
The CKA integrand becomes:
\[
f(\tilde G,\tilde G') = \frac{1}{\sqrt{\det\left(J^\top \bar G J\right)}} = \frac{1}{\det(J)}f(G,G')
\]
When integrated over $\mathcal{M}$, the factor $\det(J)^{-1}$ above is cancelled out by the factor $\det(J)$ from the change in coordinates for $\mathrm{dvol}_g$, so that the integral defining CKA is coordinate invariant. However, the integrand $f$ itself is not pointwise coordinate invariant and cannot be meaningfully interpreted in the same way as $d_\mathrm{SR}$.
}

\rd{More crucially, by averaging $G$ and $G'$, $f$ loses information regarding their geometric structure. which can result in paradoxical results. Specifically, the arithmetic mean does not distinguish between metrics that expand the same vs. different directions. As an example, consider jointly diagonalizable metrics with opposing expansion/compression eigenvectors: 
\begin{align*}
    G &= V \Lambda V^\top, \hspace{1.5mm} \qquad  \lambda_k = \frac{2k}{m+1}, \\
    G'&=V\Lambda'V^\top, \qquad \lambda'_k = \frac{2(m+1-k)}{m+1}, 
\end{align*}
for $k=1,\dots,m$. Then the CKA integrand becomes:
\[
f(G,G') = \prod_k \left( \frac{\lambda_k + \lambda'_k }{2}\right)^{-1/2} = 1
\]
This results in exactly the same value as that between the Euclidean metric and itself: $f(I_m,I_m) = 1$. In contrast, by using the generalized eigenvalues of $G$ and $G'$, $d_\mathrm{SR}$ is sensitive to such anisotropies. }

\rd{In the above example, $f(G,G')=1$; however, $\mathrm{CKA}(G,G')\neq 1$ because of the normalizing factor, which is based on products of the eigenvalues: e.g., $\det(G)^{-1/2}\det(G')^{-1/2} = \prod_k (\lambda_k\lambda'_k)^{-1/2}\neq 1$. Because of this normalization, CKA does not ultimately return the same similarity as for the Euclidean metric to itself. Nevertheless, it is possible to construct paradoxical examples for CKA even when the normalizing factor is taken into account.}

\rd{To this end, let $G=\mathrm{diag}(c,c^{-1})$, corresponding to expanding one coordinate by $c>0$ and compressing the other by $c^{-1}$. We assume that $G$ is uniform on $\mathcal{M}$. Then, the CKA similarity between $I_2$ and $G$ is:
\begin{align*}
    \mathrm{CKA} (I_2,G) = \frac{\mathrm{HSIC}(I_2,G)}{\sqrt{\mathrm{HSIC}(I_2,I_2)\mathrm{HSIC}(G,G)}} = \frac{\det\left(\frac{1}{2}(I_2+G)\right)^{-1/2}}{1 \cdot 1} 
    =\frac{2\sqrt{c}}{1+c}
\end{align*}
Surprisingly, this yields the same similarity as that between $I_2$ and the scaled Euclidean metric $cI_2$:
\begin{align*}
    \mathrm{CKA} (I_2,cI_2)   = \frac{\det\left(\frac{1}{2}(I_2+cI_2)\right)^{-1/2}}{1 \cdot \sqrt{\det(cI_2)^{-1/2}}} 
    =\frac{2\sqrt{c}}{1+c}
\end{align*}
further demonstrating that the insensitivity to anisotropy leads to paradoxical results for the similarity.
}

\rd{Fundamentally, these issues arise because $f$ does not respect the geometry of the SPD cone. The arithmetic mean is only a midpoint between $G$ and $G'$ in the Euclidean geometry of $\mathbb{R}^{m \times m}$; a more natural choice would be to use a geometric mean determined as the point equidistant between $G$ and $G'$ along a geodesic defined by the AIRM or other choice of Riemannian metric on $\mathbf{S}_m^{++}$.  Instead, by defining $d_\mathrm{SR}$ as a distance respecting the intrinsic geometry of the SPD cone, MSA is a meaningful similarity measure between Riemannian metrics.  }

\phantomsection
\label{proof:kcka}
\begin{proof}[Proof of Proposition \ref{thm:kcka}]
\rd{Consider an $m$-dimensional manifold $\mathcal{M}$ and two immersions $\varphi^1,\varphi^2:\mathcal{M}\to\mathbb{R}^n$ with uniform samples $\{\varphi^\ell(p^{(i)})\}_{i\in[q]}$  for $\ell =1,2$, where $p^{(i)}\in\mathcal{M}$ are uniformly sampled according to the metric $g$. Let $k:\mathbb{R}^n \times \mathbb{R}^n \to \mathbb{R}$ be a positive definite kernel. The CKA similarity\cite{kornblith2019similarity} is then defined as the normalized empirical Hilbert-Schmidt inner product between kernel matrices $K_1$ and $K_2$:
\[
\mathrm{CKA}(K_1,K_2)=\frac{\mathrm{HSIC}(K_1,K_2)}{\sqrt{\mathrm{HSIC}(K_1,K_1)\mathrm{HSIC}(K_2,K_2)}}
\]
where the empirical kernel matrices are given by:
\[
(K_\ell)_{ij} = k(\varphi^\ell(p^{(i)}),\varphi^\ell(p^{(j)})), \quad \ell=1,2.
\]
We can pull back the kernel to the manifold separately through each function to define two new kernels $k_1,k_2:\mathcal{M}\times\mathcal{M} \to \mathbb{R}$, where:
\[
k_\ell(p^{(i)},p^{(j)})=k(\varphi^\ell(p^{(i)}),\varphi^\ell(p^{(j)})), \quad \ell=1,2
\]
In the limit $q\to\infty$, the empirical HSIC converges to the population HSIC:\cite{gretton2005measuring}
\[
\mathrm{HSIC}(\varphi^1,\varphi^2)=\frac{1}{\mathrm{Vol_g(\mathcal{M})^2}}\int_{\mathcal{M}} \int_{\mathcal{M}} k_1(p,p') k_2(p,p') \mathrm{dvol}_g(p)\mathrm{dvol}_g(p') + \mu_1 + \mu_2
\]
where $\mu_1$ and $\mu_2 $ contains terms arising from centering the empirical kernels:
\begin{align*}
&\mu_1 =     \frac{1}{\mathrm{Vol_g(\mathcal{M})^4}}\left(\int_{\mathcal{M}}\int_{\mathcal{M}} k_1(p,p') \mathrm{dvol}_g(p) \mathrm{dvol}_g(p')\right)\left(\int_{\mathcal{M}}\int_{\mathcal{M}} k_2(p,p') \mathrm{dvol}_g(p)\mathrm{dvol}_g(p')\right)\\
&\mu_2 =    -\frac{2}{\mathrm{Vol_g(\mathcal{M})^3}} \int_{\mathcal{M}} \left( 
\int_{\mathcal{M}} k_1(p,p')  \mathrm{dvol}_g(p)\right) \left( 
\int_{\mathcal{M}} k_2(p,p')  \mathrm{dvol}_g(p)\right)\mathrm{dvol}_g(p')
\end{align*}
In particular, for the RBF kernel:
\[
k_\ell (p,p')=\exp\left(-\frac{\|\varphi^\ell(p)-\varphi^\ell(p')\|_2^2}{2\alpha^2}\right).
\]
} 

\rd{Note that we can expand the squared Euclidean norm around arbitrary point $p$ in terms of the pullback metric:
\[
\|\varphi^\ell(p)-\varphi^\ell(p')\|_2^2 = \delta p^\top \underbrace{J_{\varphi^\ell}(p)^\top J_{\varphi^\ell}(p)}_{G^{\varphi^\ell}(p)} \delta p +\mathcal{O}(\delta p^3)
\]
where $\delta p\in T_p\mathcal{M}$ is the tangent vector such that $p'=\exp_p(\delta p)$.  The pulled back kernel becomes:
\[
k_\ell(p,p') = \exp\left( -\frac{1}{2\alpha^2}\delta p^\top G^{\varphi^\ell}(p)\delta p \right)\exp\left( -\frac{1}{2\alpha^2}\mathcal{O}(\delta p ^3) \right)
\]
Using this expansion, the integrand in the first term of the HSIC becomes:
\[
k_1(p,p')k_2(p,p')=\exp\left( -\frac{1}{2\alpha^2} \delta p^\top (G^{\varphi^1}(p)+G^{\varphi^2}(p)) \delta p \right) \exp \left( -\frac{1}{\alpha^2} \mathcal{O}(\delta p^3) \right)
\]
We can now evaluate the HSIC by directly changing the inner integration variables from points $p'\in\mathcal{M}$ to scaled tangent vectors $\mathbf{u}\in\mathbb{R}^m$ via the scaled exponential map $p'=\exp_p(\alpha \mathbf{u})$. The volume measure transforms as $\mathrm{dvol}_g(p')=\alpha^m v(p,\alpha \mathbf{u})d\mathbf{u}$ where $v$ is the volume density of the exponential map. We then obtain:
\[
\mathrm{HSIC}(\varphi^1,\varphi^2) = \frac{\alpha^m}{\mathrm{Vol}_g(\mathcal{M})^2} \int_\mathcal{M}\int_{\mathbb{R}^m} \exp \left( -\mathbf{u}^\top \bar G(p)\mathbf{u}\right) \exp \left(-\mathcal{O}(\alpha\|\mathbf{u}\|^3) \right) v(p,\alpha \mathbf{u}) \mathrm{d\mathbf{u}} \mathrm{dvol}_g(p)
\]
where we have used the fact that the Gaussian decaying integrand makes the contribution outside the injectivity radius negligible as $\alpha \to 0$.  Note that we have dropped the centring terms: by the same exponential map substitution above, each $\int_\mathcal{M}k_\ell(p,p'
)\mathrm{dvol}_g(p)=\mathcal{O}(\alpha^m)$ , and the outer integral over $\mathrm{dvol}_g(p')$ simply adds a scaling factor of $\mathrm{Vol}_g(\mathcal{M})$. Then because they contain a product of two such terms, $\mu_1,\mu_2 = \mathcal{O}(\alpha^{2m})$. The first term is $\mathcal{O}(\alpha^m)$ and dominates over $\mu_1,\mu_2$ as $\alpha \to 0$.

Next, applying the same derivation to the terms in the denominator of the HSIC, the $\alpha^m$ will cancel out. Moreover, as $\alpha \to 0$, the higher-order error term in the exponential vanishes, and the volume density $v(p,\alpha\mathbf{u})\to v(p,\mathbf{0})=1$, yielding:
\[
\mathrm{CKA}(\varphi^1,\varphi^2) = c' \int_\mathcal{M}\int_{\mathbb{R}^m} \exp\left( -\mathbf{u}^\top \bar G(p)\mathbf{u}\right)\mathrm{d\mathbf{u}}\mathrm{dvol}_g(p)
\]
where $c'$ is asymptotically constant as $\alpha \to 0$, $q\to\infty$. Finally, we can use the fact that the inner integral is the integral of a Gaussian over $\mathbb{R}^m$ with precision matrix $\bar G$, which evaluates to:
\[
\int_{\mathbb{R}^m} \exp (-\mathbf{u}^\top \bar G(p) \mathbf{u})\mathrm{d\mathbf{u}} = \frac{\pi^{m/2}}{\sqrt{\det\left(\bar G(p)\right)}}
\]
so that we finally obtain
\[
\mathrm{CKA}(\varphi^1,\varphi^2)=c'' \int_\mathcal{M} \frac{1}{\sqrt{\det\left(\bar G(p)\right)}} \mathrm{dvol}_g(p)
\]
for some normalizing factor $c''$.
}
\end{proof}

\subsection{\rd{MSA for dynamical systems}} \label{app:van_der_pol}

\rd{MSA can be applied to dynamical systems, where the a manifold may parametrise the initial state of, or the time-varying input to, the system. We refer the reader to Pellegrino \& Chadwick 2025\cite{pellegrino2025rnns} for a thorough discussion of pullback metrics through dynamical systems.}

\rd{\textbf{General formulation.} Consider the generic dynamical system:
\begin{align*}
    &\dot{\mathbf{x}}={\mathbf f}({\mathbf{x}},\mathbf{u}(t),t)  \\
    &{\mathbf x}(0) = \mathbf{x}_0
\end{align*}
where $\mathbf{x}(t)\in\mathbb{R}^n$,  $\mathbf{u}(t)\in\mathbb{R}^{n_\mathrm{in}}$, and $\mathbf{f}:\mathbb{R}^n\times\mathbb{R}^{n_\mathrm{in}}\times \mathbb{R}_{\geq 0} \to \mathbb{R}^n$.
We consider the flow map $\Phi_t^\mathbf{u}(\mathbf{x}_0) = \mathbf{x}(t)$, where $\mathbf{x}(t)$ is the solution of the system above for a prescribed control signal $\mathbf{u}(t)$ (i.e. time-varying input). To calculate the pullback metric, we parametrise both the initial condition $\mathbf{x}(0)$ and the time-varying control $\mathbf{u}(t)$ through a product manifold $\mathcal{M}=\mathcal{M}_0\times\mathcal{M}_u$. Specifically, we define the following two smooth maps:
\[
\psi_0:\mathcal{M}_0 \to \mathbb{R}^n, \quad \psi_u:\mathcal{M}_u \times \mathbb{R}_{\geq 0}\to \mathbb{R}^{n_\mathrm{in}}
\]
where $\mathcal{M}_0$, $\mathcal{M}_u$ are manifolds of dimensionality $m_0$ and $m_u$, respectively, with $m=m_0+m_u$. The initial condition and control signal are thus given by:
\[
\mathbf{x}_0=\psi_0(p_0), \quad \mathbf{u}_{p_u}(t) = \psi_u(p_u,t),
\]
where $p_0\in\mathcal{M}_0$, $p_u\in\mathcal{M}_u$. We define the map $F:\mathcal{M}\times\mathbb{R}_{\geq 0} \to \mathbb{R}^n$ by:
\[
    F(p,t) = \Phi_t^{\mathbf{u}_{p_u}}(\psi_0(p_0)), \quad p = (p_0,p_u)
\]
The Jacobian of $F$ defines how the state depends on changes in the initial state, control signal, and time: $J(p,t)\in \mathbb{R}^{n\times(m+1)}$. In particular, if $\mathbf{q}=[q_1,\dots,q_m]\in \mathbb{R}^{m}$ are local coordinates around $p\in \mathcal{M}$ and writing $F$ in these local coordinates, the Jacobian can be written as:
\[
J(p,t) = \left[ \frac{\partial F}{\partial q_1} \quad \dots \quad \frac{\partial F}{\partial q_m} \quad  \frac{\partial F}{\partial t} \right] \in \mathbb{R}^{n\times(m+1)}
\]
Then, the input-time pullback metric can be written as:
\[
G(p,t) = J(p,t)^\top J(p,t) \in \mathbb{R}^{(m+1)\times(m+1)}
\]
and MSA can be applied by integrating over both $\mathcal{M}$ and time (as in Fig. \ref{fig:rnnsmm}d). If instead one wishes to compare how inputs warp geometries at a specific time of interest $t^\star$ (as in Fig. \ref{fig:rnnsmm}c), one may instead calculate the Jacobian and pullback metric only over changes in $\mathcal{M}$:
\[
\tilde G(p,t^\star) = \tilde J(p,t^\star)^\top \tilde J(p,t^\star) \in \mathbb{R}^{m\times m}, \qquad 
\tilde J(p,t^\star) = \frac{\partial F}{\partial \mathbf{q}} \in \mathbb{R}^{n\times m}
\]
and integrate the MSA distance only over  $\mathcal{M}$.} 

\rd{\textbf{Relationship to other dynamical similarity methods.} Since the pullback metric is computed from the Jacobian of the flow map, it captures dynamics through the local velocity field.
However, in its focus on intrinsic geometry, MSA contrasts from existing dynamical similarity measures such as DSA \cite{ostrow2023beyond,huang2025inputdsa}, DFORM \cite{chen2024dform}, and DAA \cite{sagodi2025dynamical}. Such methods are primarily concerned with the topology of the dynamics, and consider two dynamical systems equivalent if they are topologically conjugate. In particular, nonlinearly transforming the state space of a dynamical system with a homeomorphism:
\[
    \mathbf{h}(\mathbf{x}) = \mathbf{y}\in \mathbb{R}^{n'}, \qquad \mathbf{\dot y} = \frac{\mathbf{dh}}{ \mathbf{dx}}\mathbf{f}(\mathbf{h}^{-1}(\mathbf{y}), \mathbf{u}(t), t)
\] 
will, by definition, result in a dynamical system that is topologically equivalent to the first, but will generally not be geometrically equivalent. Instead, the pullback metric transforms as:
\[
J^\top J \to J^\top \left(\frac{\mathbf{dh}}{\mathbf{dx}} \right)^\top \left(\frac{\mathbf{dh}}{\mathbf{dx}} \right) J
\]
Hence a general nonlinear transformation will modify the metric since: 
\[
\left(\frac{\mathbf{dh}}{\mathbf{dx}} \right)^\top \left(\frac{\mathbf{dh}}{\mathbf{dx}} \right) \neq I. 
\]
unless it is isometric.

Therefore, topological conjugacy and intrinsic geometric similarity are fundamentally different properties. Since topologically conjugate dynamical systems share the same invariant manifolds of their flows, conjugacy-based similarity measures can be particularly insightful in cases where such manifolds are thought underlie computational mechanisms, as in the case of attractor manifolds. Yet, for task and model classes in which transient dynamics are fundamental computational mechanisms, topological conjugacy alone is insufficient to measure similarity. In such cases, geometric structure can provide additional information, as illustrated in Fig.\,\ref{fig:rnnsmm}.
}

\rd{\textbf{Example of topologically conjugate but not isometric systems.} Consider the Van der Pol oscillator:
\begin{align*}
\begin{cases}
    \dot{x}_1 &= x_2 \\
    \dot{x}_2 &= \mu (1-x_1^2)x_2-x_1
\end{cases}
\end{align*}
Following Gosztolai et al., 2025\cite{gosztolai2025marble}, we embed the oscillator onto a paraboloid $\mathcal{P}\subset\mathbb{R}^3$ via the homeomorphism $\mathbf{h}:\mathbb{R}^2 \to \mathcal{P}$, defined by $\mathbf{y}=\mathbf{h}(\mathbf{x})$:
\begin{align*}
\begin{cases}
    y_1 &= x_1 \\
    y_2 &= x_2 \\
    y_3 &= -(\alpha x_1)^2 - (\alpha x_2)^2,
\end{cases}
\end{align*}
where $\alpha >0$ is a parameter that determines the extrinsic curvature of the paraboloid embedding. By construction, the dynamical systems for $\mathbf{x}$ and $\mathbf{y}$ are topologically conjugate. However, they do not have the same intrinsic geometry. To see this, note that:
\[
\frac{\mathbf{dh}}{ \mathbf{dx}} = \begin{pmatrix}
1 & 0 \\
0 & 1 \\
-2\alpha^2 x_1 & -2\alpha^2 x_2
\end{pmatrix}
\]
The pullback metric is then transformed as:
\[
J^\top J \to J^\top \left( \frac{\mathbf{dh}}{\mathbf{dx}} \right)^\top \left(\frac{ \mathbf{dh}}{\mathbf{dx}}\right) J= J^\top \begin{pmatrix}
1+ 4\alpha^4 x_1^2 & 4\alpha^4 x_1x_2 \\
4\alpha^4 x_1x_2 & 1+ 4\alpha^4 x_2^2
\end{pmatrix} J
\]
Therefore $\mathbf{h}$ is isometric only in the trivial case where $\alpha =0$, which corresponds to the paraboloid reverting to a plane. This shows that topological conjugacy does not generally preserve intrinsic geometry. As a result, MSA will generally be sensitive to nonlinear deformations that preserve topological conjugacy but distort the pullback metric. 
}

\rd{\subsection{Relationship to statistical estimators and the Cramér-Rao bound}}\label{app:cramer_rao}

Although we have focused on MSA as a way to compare Riemannian manifolds, in the case of statistical manifolds there is an interesting interpretation of the spectral ratio in terms of the relative information of different model families, connecting with classic work in mathematical statistics.

For $\mathcal{M}$ be an $m$-dimensional manifold, consider two families of distributions $p(\mathbf{x};{\theta})$ and $q(\mathbf{x};{\theta})$, both parametrized by ${\theta}\in\mathcal{M}$, with Fisher information matrices given by $\mathcal{I}_p({\theta})$ and $\mathcal{I}_q({\theta})$.  Suppose we are interested in estimating a scalar function of the parameters: $f:\mathcal{M}\to\mathbb{R}$. For any unbiased estimator $\hat{ f}_{p}$ from $p(\cdot\,;{\theta})$, the Cramér-Rao bound states that $\mathrm{Var}(\hat f_p)\geq g_\theta(\mathrm{grad(f),\mathrm{grad}(f)})$; or in local coordinates:
\[
\mathrm{Var}(\hat f_p) \geq \nabla f(\theta)^\top \mathcal{I}_p({\theta})^{-1} \nabla f(\theta)
\]
Thus, the best possible precision of $\hat{f}_p$ is fundamentally limited by the information geometry defined by $\mathcal{I}_p({\theta})$. Taking the ratio of the Cramér-Rao lower bound under the the information geometries defined by $p(\cdot\,;\theta)$ and $q(\cdot\,;\theta)$ yields a generalized Rayleigh quotient:
\[
    \frac{  \nabla f(\theta)^\top\mathcal{I}_p({\theta})^{-1} \nabla f(\theta)}{  \nabla f(\theta)^\top\mathcal{I}_q({\theta})^{-1} \nabla f(\theta)} 
\]
As a Rayleigh quotient, this ratio can be bounded by the generalized eigenvalues, and thus by the SR similarity between the two information matrices. 

\begin{lemma}\label{lemma:rayleighbound}
    Let $G,G'\in\mathbb{R}^{m \times m}$ be SPD matrices, and define their generalized Rayleigh quotient:
    $
    R(\mathbf{v}) := \frac{ \mathbf{v}^\top G\mathbf{v}}{ \mathbf{v}^\top G'\mathbf{v}}
    $ where $\mathbf{v}\in\mathbb{R}^m$.
    Then $R(\mathbf{v})$ is bounded as:
    \[
        \bar \lambda(\underbrace{1-d_\mathrm{SR}}_{\textrm{similarity}}) \leq R(\mathbf{v}) \leq \frac{\bar \lambda }{1-d_\mathrm{SR}}
    \]
    where $d_\mathrm{SR} = 1-\sqrt{\frac{\lambda_\textrm{min}}{\lambda_\textrm{max}}}$ is the spectral ratio distance, $\lambda_\textrm{max}$ and $\lambda_\textrm{min}$ are the maximum and minimum generalized eigenvalues between $G,G'$, and $\bar{\lambda}=\sqrt{\lambda_\mathrm{min}\lambda_\mathrm{max}}$ is their geometric mean.
\end{lemma}
\begin{proof}
Since $R(\mathbf{v})$ is a Rayleigh quotient, it obeys  $\lambda_\mathrm{min}\leq R(\mathbf{v})\leq \lambda_\mathrm{max}$. Then:
\[ R(\mathbf{v}) \geq \lambda_\mathrm{min} = \sqrt\frac{\lambda_\mathrm{min}}{\lambda_\mathrm{max}}\cdot\sqrt{\lambda_\mathrm{min}\lambda_\mathrm{max}} = (1-d_\mathrm{SR}) \cdot\bar \lambda\]
The upper bound follows mutatis mutandis.
\end{proof}

In particular, applying Lemma~\ref{lemma:rayleighbound} gives a bound on the relative performance of two efficient unbiased estimators based on the similarity of the two information geometries:
\[
\bar \lambda(1-d_\mathrm{SR}) \leq \frac{\mathrm{Var}(\hat f_p)}{\mathrm{Var}(\hat f_q)} \leq \frac{\bar \lambda}{1-d_\mathrm{SR}}
\]
While the existence of an efficient unbiased estimator for all $f$ is a strong assumption, the same bound is satisfied asymptotically for the the maximum likelihood estimator (MLE). Let $\hat\theta_{p,n}$ be the MLE based on $n$ i.i.d. samples from $p(\cdot\,;\theta)$. Under standard regularity distributions, $\hat\theta_{p,n}$  is asymptotically normal
\[
\sqrt{n}(\hat\theta_{p,n}-\theta) \to_d \mathcal{N}(0,\mathcal{I}_p(\theta)^{-1})
\]
By the delta method this translates to: 
\[\sqrt{n}\left(f(\hat\theta_{p,n})-f(\theta)\right)
\to_d \mathcal{N}\!\left(0,\nabla f(\theta)^\top\mathcal{I}_p(\theta)^{-1}\nabla f(\theta)\right)\]

The asymptotic relative efficiency (ARE) is defined as the ratio of the asymptotic variances of two MLEs for $f(\theta)$:
\[
\mathrm{ARE}_f = \frac{\lim_{n\to\infty}n\mathrm{Var}(f(\hat{ \theta}_{p,n}))}{\lim_{n\to\infty}n\mathrm{Var}(f(\hat{ \theta}_{q,n}))}
\]
so that $\mathrm{ARE}_f>1$ indicates that $f({\theta})$ is estimated more precisely under $q(\cdot\,;{\theta})$ than $p(\cdot\,;{\theta})$ as $n \to \infty$. We thus obtain a similar bound on the ARE:
\[
\bar\lambda(1-d_\mathrm{SR}) \leq \mathrm{ARE}_f \leq \frac{\bar\lambda}{1-d_\mathrm{SR}}
\]

Thus far, we have simply recast a standard bound on a Rayleigh quotient using the SR distance instead of the generalized eigenvalues. A more interesting result is its connection with MSA. Note that $\bar \lambda$ is the geometric midpoint of $[\lambda_\mathrm{min},\lambda_\mathrm{max}]$ (the range of the ARE) and quantifies the overall scale of the difference in information geometries. Integrating the lower bound above over the manifold then produces the following result:
\[
\frac{1}{\mathrm{Vol}_g(\mathcal{M})}\int_\mathcal{M} \frac{\mathrm{ARE}_f(\theta)}{\bar\lambda(\theta)}\mathrm{dvol}_g(\theta) \geq 1-d_\mathrm{MSA}(\mathcal{I}_p,\mathcal{I}_q)
\]
In other words, MSA bounds the average normalized performance of the two estimators over the manifold.

\section{Implementations}\label{app:implementation}

\subsection{MSA Implementation}

\textbf{Implementation.} Like most existing similarity measures, MSA relies on estimating an integral over data.\footnote{For example, methods such as Procrustes, CCA, or CKA all rely on empirical estimates of covariance matrices.} This can be done either via an explicit characterization of the manifold (Algorithm \ref{alg:msa_cubature}) or via Monte Carlo sampling (Algorithm \ref{alg:msa_mc}).

In many domains, access to the underlying manifold is not a significant limitation. In neuroscience models, input manifolds are typically low-dimensional and parametrised by a small number of controlled task variables\cite{mante2013context,valente2022extracting,pellegrino2025rnns}. In machine learning, mechanistic interpretability often focuses on intrinsically low-dimensional manifolds induced by controlled transformations.\cite{tennenholtz2022uncertainty,zavatone2023does,karczewski2025spacetime,gurnee2026models} In these cases, standard numerical integration methods such as cubature can be used.

\begin{algorithm}[h]
\caption{MSA Estimation via Numerical Cubature}\label{alg:msa_cubature}
\begin{algorithmic}[1]
\Require Networks $\varphi^1, \varphi^2$, embedding $\psi:\mathcal{M}\rightarrow\mathbb{R}^{n_\text{in}}$, metric $g$.
\Require Cubature rule with $N$ nodes $\{p_i\}_{i=1}^N \subset \mathcal{M}$ and weights $\{w_i\}_{i=1}^N$.
\Ensure Estimated MSA distance $\hat{d}_\text{MSA}(\varphi^1, \varphi^2)$
\State Initialise distance $D \gets 0$, Initialise volume $V \gets 0$
\For{$i = 1$ \textbf{to} $N$} in parallel:
    \State Compute Jacobians $J^1(p_i)$ and $J^2(p_i)$ of $\varphi^1 \circ \psi$ and $\varphi^2 \circ \psi$
    \State Compute pullback metric $G^1(p_i) \gets J^1(p_i)^\top J^1(p_i)$
    \State Compute pullback metric $G^2(p_i) \gets J^2(p_i)^\top J^2(p_i)$
    \State Solve $G^1(p_i)\mathbf{v} = \lambda G^2(p_i)\mathbf{v}$ for eigenvalues $\lambda_1 \ge \dots \ge \lambda_m > 0$
    \State Accumulate distance $D \gets D + w_i \left(1 - \sqrt{\lambda_m / \lambda_1}\right)$
    \State Accumulate volume $V \gets V + w_i$
\EndFor
\State \Return $\hat{d}_\text{MSA} \gets D / V$
\end{algorithmic}
\end{algorithm}

When the manifold is not known a priori, MSA can still be estimated via sampling, assuming access to samples drawn from a distribution that is uniform with respect to the manifold's volume form.\\

\begin{algorithm}[h]
\caption{MSA Estimation via Monte Carlo Sampling}\label{alg:msa_mc}
\begin{algorithmic}[1]
\Require Networks $\varphi^1, \varphi^2$, embedding $\psi:\mathcal{M}\rightarrow\mathbb{R}^{n_\text{in}}$, charts $\{U_k\}_{k=1}^K$, sample size $N$.
\Ensure Estimated MSA distance $\hat{d}_\text{MSA}(\varphi^1, \varphi^2)$
\State $D \gets 0$, $V \gets 0$
\For{$i = 1$ \textbf{to} $N$} in parallel:
    \State Sample chart index $k \sim \mathcal{U}(\{1, \dots, K\})$ and point $p_i \sim \mathcal{U}(U_k)$
    \State Compute Jacobians $J^1(p_i)$ and $J^2(p_i)$ of $\varphi^1 \circ \psi$ and $\varphi^2 \circ \psi$
    \State $G^1(p_i) \gets J^1(p_i)^\top J^1(p_i)$ \textbf{and} $G^2(p_i) \gets J^2(p_i)^\top J^2(p_i)$
    \State Solve $G^1(p_i)\mathbf{v} = \lambda G^2(p_i)\mathbf{v}$ for eigenvalues $\lambda_1 \ge \dots \ge \lambda_m > 0$
    \State $D \gets D + \left(1 - \sqrt{\lambda_m / \lambda_1}\right)$
\EndFor
\State \Return $\hat{d}_\text{MSA} \gets D / N$
\end{algorithmic}
\end{algorithm}

As discussed below, these algorithms closely mirror those used by other similarity metrics assuming data on the manifold, besides the computation of the pullback metric via the Jacobian.

\textbf{Package.} A broadly applicable package taking Pytorch modules or Jax functions as arguments to automatically compute their distances and similarity on a given manifold of inputs will be released upon publication.

\textbf{Compute resources.} The experiments in the rich-lazy models (Section\,\ref{section:rich_lazy}) and RNN-SSM models (Section\,\ref{section:dynamics}) ran in under a couple of hours with a RTX 4070 GPU. The experiments in the large-scale diffusion (Section\,\ref{section:diffusion}) and ResNet (Section\,\ref{app:resnet}) models ran on a A100 GPU under 24h. The compute time of MSA was similar to that of other methods.

\subsection{Other similarity measures}

To evaluate the alignment between representations, classic methods compare two centred data matrices $X \in \mathbb{R}^{k \times n_x}$ and $Y \in \mathbb{R}^{k \times n_y}$ of neural activations, where $k$ is the number of samples and $n_x, n_y$ are the number of dimensions (features). Certain methods require equal dimensionality (e.g. Procrustes), in which case $n_x = n_y = n$.

Several similarity measures rely on the feature-wise covariance structure of the data. Over the $k$ samples, we define the empirical auto-covariance and cross-covariance matrices as:
\[ \Sigma_{XX} = \frac{1}{k-1} X^\top X, \quad \Sigma_{YY} = \frac{1}{k-1} Y^\top Y , \quad \Sigma_{XY} = \frac{1}{k-1} X^\top Y \]
For data on representational manifolds, as in the setting described in section \ref{section:framework}, assuming uniform density relative to the volume form and centred network activations, the underlying population covariances can be written as:
\[
    \Sigma_{XX}=\frac{1}{\mathrm{Vol}(\mathcal{M})} \int_\mathcal{M} \varphi_1(p)\otimes \varphi_1(p)\, \mathrm{dvol}(p), \quad \Sigma_{XY}=\frac{1}{\mathrm{Vol}(\mathcal{M})} \int_\mathcal{M} \varphi_1(p)\otimes \varphi_2(p)\, \mathrm{dvol}(p)
\]
This closely mirrors the quantities involved in the definition of MSA.

\textbf{Canonical Correlation Analysis (CCA).} CCA finds linear projections that maximize the cross-correlation between the two sets of variables. The first canonical variable is defined through the following optimization:
\[ \rho_1 = \max_{a_1, b_1} \mathrm{corr}(X a_1, Y b_1) \]
Subsequent canonical variables are defined similarly while being constrained to be uncorrelated with all previous canonical variables. This reduces to finding the singular values of the whitened cross-covariance matrix:
\[ K = \Sigma_{XX}^{-1/2} \Sigma_{XY} \Sigma_{YY}^{-1/2} \]
The canonical correlations $\{ \rho_i \}$ are the singular values of $K$. The final score is the mean of $\rho_i$.

\textbf{Orthogonal Procrustes.} Given equal dimensions $n$, Procrustes solves for the optimal orthogonal transformation $R \in O(n)$ that minimizes $\| X R - Y \|_F^2$. Using the SVD of the cross-covariance matrix $X^\top Y = U \Sigma V^\top$, the optimal rotation is:
\[ R = U V^\top \]
For scaled Procrustes, we additionally calculate an isotropic scaling factor $\alpha = \frac{\mathrm{Tr}(\Sigma)}{\|X\|_F^2}$. The alignment is reported via the $R^2$ coefficient:
\[ R^2 = 1 - \frac{\|X_{\text{aligned}} - Y\|_F^2}{\|Y\|_F^2} \]
The Procrustes $R^2$ is not bounded below. However $R^2=0$ is meaningful, as it corresponds to the alignment of the zero data matrix $X=0$ to an arbitrary $Y$ (given the centring step).

\textbf{Representational Similarity Analysis (RSA).} RSA compares the $k \times k$ sample-wise similarity matrices. Let $S_X$ and $S_Y$ be the centred sample-by-sample covariance matrices computed over the $n$ features:
\[ S_X = \frac{1}{n-1} X X^\top, \quad S_Y = \frac{1}{n-1} Y Y^\top \]
These matrices are flattened into vectors $s_X = \mathrm{vec}(S_X)$ and $s_Y = \mathrm{vec}(S_Y)$. The similarity is then computed using the Pearson correlation between these flattened vectors.

\textbf{Centered Kernel Alignment (CKA).} For linear CKA, we compute Gram matrices $K = X X^\top$ and $L = Y Y^\top$. These are double-centred using $H = I - \frac{1}{k}\mathbf{1}\mathbf{1}^\top$ to obtain $K_c = HKH$ and $L_c = HLH$. The metric is defined as:
\[ \mathrm{CKA}(X,Y) = \frac{\langle K_c, L_c \rangle_F}{\|K_c\|_F \|L_c\|_F} \]
In Kernel CKA (kCKA), $K$ and $L$ are replaced by non-linear kernels, such as the RBF kernel $k(x_i, x_j) = \exp(-\frac{\|x_i - x_j\|_2^2}{2\sigma^2})$.

\textbf{Symmetrised Linear Regression (LSTSQ).} We fit ordinary least squares maps (including a bias term) in both directions. For $\tilde{Y} = [Y, \mathbf{1}]$:
\[ W_{X \leftarrow Y} = (\tilde{Y}^\top \tilde{Y})^{-1} \tilde{Y}^\top X, \quad R^2_{X \leftarrow Y} = 1 - \frac{\|X - \tilde{Y} W_{X \leftarrow Y}\|_F^2}{\|X - \bar{X}\|_F^2} \]
The final metric is the bidirectional average: $R^2 = \frac{1}{2}(R^2_{X \leftarrow Y} + R^2_{Y \leftarrow X})$.

\textbf{Dynamical Similarity Analysis (DSA).} DSA\cite{ostrow2023beyond}  first applies Dynamic Mode Decomposition (DMD) to delay-embedded state observations, fitting a transition matrix $A_X$ where $X(t+\Delta t) = A_X X(t)$. To compare $A_X$ and $A_Y$, it applies a modified Procrustes Analysis to align these operators  by solving the optimization problem:
\[
\mathrm{DSA}(A_X,A_Y) = \min_{C \in O(n)} \| A_X - C A_Y C^{-1} \|_F
\]
The final score is then computed either from the angle between the aligned operators (reported as a cosine similarity) or from the Wasserstein distance between their eigenvalue spectra.

\section{Numerical experiments details}

\subsection{\rd{Plane vs. Non-Euclidean plane vs. Swiss roll}}\label{app:swiss_roll}

\rd{In Fig.~\ref{fig:intrinsic-extrinsic}, we compare three two-dimensional manifolds embedded in $\mathbb{R}^3$: a Euclidean plane, a non-Euclidean plane and a Swiss roll. Each manifold is parameterised by the same coordinates $\boldsymbol{\theta}=(\theta_1,\theta_2)\in [-1, 1]\times [-a, a]$ through an embedding function $\varphi:\mathbb{R}^2\rightarrow\mathbb{R}^3$.

\textbf{Euclidean plane.} We simply define the plane as:
\[
    \varphi_{\mathrm{plane}}(\boldsymbol{\theta}) =
    \begin{bmatrix}
    \theta_1\\
    \theta_2\\
    0
    \end{bmatrix}
\]

\textbf{Non-Euclidean plane.} To define the plane with non-Euclidean geometry, we reparametrise the plane with a non-isometric function. We combine two warping effects: an expansion away from the centre of the plane and a rotation. In this reparametrisation, we define the non-Euclidean plane as: 
\[
    \varphi_{\mathrm{spiral}}(\boldsymbol{\theta}) =
    \begin{bmatrix}
    \gamma_1\\
    \gamma_2\\
    0
    \end{bmatrix}\quad \text{where:} \quad \begin{bmatrix}
        \gamma_1\\
        \gamma_2
    \end{bmatrix} = 
    e(\tilde \theta)\,R(\tilde \theta)
    \begin{bmatrix}
        \theta_1 \\
        \theta_2
    \end{bmatrix}
 \]
The expansion and rotations are respectively defined as:
\[
    e(\tilde\theta) = (1+b \tilde \theta), \qquad R(\tilde \theta)=\begin{bmatrix}
        \cos \tilde\theta & -\sin\tilde \theta\\
        \sin \tilde\theta & \cos \tilde\theta
    \end{bmatrix}
\]
with $b$ constant hyperparameters defining the amount of expansion. The rotation/expansion variable is defined as:
\[
    \tilde \theta = \left((1-\theta_1^2)(1-\theta_2^2/a^2)\right)^2 
\]
which ensures the boundary point are the same as that of the Euclidean plane, meaning that $\varphi_{\mathrm{spiral}}(\boldsymbol{\theta})=\varphi_{\mathrm{plane}}(\boldsymbol{\theta})$ for $\boldsymbol{\theta}$ in the boundary of $[-1,1]\times [-a, a]$.

\textbf{Swiss roll.} Finally, we define the Swiss roll as:
\[
    \varphi_{\mathrm{Swiss}}(\boldsymbol{\theta})
    =
    \begin{bmatrix}
        c\,\gamma(\theta_1)\cos\gamma(\theta_1)\\
        \theta_2\\
        c\,\gamma(\theta_1)\sin\gamma(\theta_1)
    \end{bmatrix},
\]
which is the Cartesian product of a spiral and a line, and $\gamma(\theta_1)$ parametrises the spiral via its arc length. For the standard spiral $[\gamma\cos\gamma, \gamma\sin\gamma]$ the pullback metric is $1+\gamma^2$ so space stretches with the radius (angle). To counteract this stretching, we define:
\[
    A(\gamma)
    =
    \frac{1}{2}
    \left(
        \gamma\sqrt{1+\gamma^2} + \operatorname{arsinh}\gamma
    \right),
\]
such that $A'(\gamma)=\sqrt{1+\gamma^2}$. Now $\gamma(\theta_1)$ is defined such that:
\[
    A\left(\gamma(\theta_1)\right)
    =
    A(\gamma_{\min}) +
    \left(\frac{\theta_1}{2}+\frac{1}{3}\right)
    \left[
        A(\gamma_{\max})-A(\gamma_{\min})
    \right],
\]
where $\gamma_{\min}$ and $\gamma_{\max}$ are constant hyperparameters defining how many periods (in radians) the Swiss roll should be wrapped around. Finally, setting $c=0.5/(A(\gamma_{\max})-A(\gamma_{\min}))$ ensures a constant metric w.r.t. $\theta_1$. The resulting pullback metric is therefore:
\[
    G_{\mathrm{Swiss}}
    =
    \begin{bmatrix}
        1 & 0\\
        0 & 1
    \end{bmatrix}.
\]
Thus, the Swiss roll has the standard Euclidean metric.

We summarise the hyperparameters of all three manifolds below:
\begin{table}[H]
\centering
\begin{tabular}{c c c c}
\hline
$a$ & $b$ & $\gamma_{\min}$ & $\gamma_{\max}$ \\
\hline
\noalign{\vskip 3pt}
$\frac{19}{15}$ & $0.6$ & $\frac{8\pi}{5}$ & $\frac{18\pi}{5}$ \\
\noalign{\vskip 3pt}
\hline
\end{tabular}
\end{table}

\textbf{Similarity analysis.} For the numerical simulations we use a uniform $q\times q=101\times 101$ grid of $\theta_1, \theta_2$. The ``intrinsic'' and ``extrinsic'' bars are respectively the MSA similarity and the Procrustes $R^2$:
\[
    1-d_{\mathrm{MSA}} =\frac{1}{q^2}\sum_{\boldsymbol{\theta}} \frac{\lambda_{\min}(\boldsymbol{\theta})}{\lambda_{\max}(\boldsymbol{\theta})}, \quad R^2 = 1 - \frac{1}{q^2}\min_{Q\in O(n)}\sum_{\boldsymbol{\theta}} \lVert \bar\varphi_{\mathrm{Swiss}}(\boldsymbol{\theta}) - Q\bar\varphi_{\mathrm{plane}}(\boldsymbol{\theta}) \rVert^2
\]
where $\bar\varphi(\boldsymbol{\theta}) = \varphi(\boldsymbol{\theta}) - \frac{1}{q^2}\sum_{\boldsymbol{\theta}'}\varphi(\boldsymbol{\theta}')$ are the centred embeddings.
}

\subsection{Rich-lazy networks}\label{app:rich_lazy}

In Section \ref{section:rich_lazy} we consider a one-layer deep nonlinear network trained in rich and lazy learning regimes on a circular classification task. The network is given by:
\rd{\[
    \mathbf{p}\in [-1, 1]^2, \quad  \mathbf{x} = Q\mathbf{p}\in\mathbb{R}^{n_{\mathrm{in}}} \quad \quad \mathbf{z} = \tanh(W_1\mathbf{x}+\mathbf{b})\in\mathbb{R}^n, \quad \mathbf{y} = W_2 \mathbf{z}\in \mathbb{R}^k
\]
where is a semi-orthogonal matrix so that $Q^\top Q = I_{2}$. To preserve radial symmetry in the loss and gradient we train the network on randomly sampled $(r^{(i)}, \theta^{(i)})$ where $\sqrt{r}\sim \mathcal{U}(0,1)$ and $\theta\sim \mathcal{U}(0, 2\pi)$, such that $\mathbf{p}^{(i)}=\left(r^{(i)}\cos\theta^{(i)}, r^{(i)}\sin\theta^{(i)}\right)$, which corresponds to a uniform density in the disk with unit radius. All subsequent analyses use the whole $[-1,1]^2$ manifold.} The target was a one-hot encoded class vector:
\[
    \mathbf{\bar y}(\theta) = \text{onehot}(\theta, k)=[\mathds{1}_{\theta\in[0,2\pi/k)},...,\mathds{1}_{\theta\in[2\pi-2\pi/k,2\pi)}]
\]
where $k$ is the number of classes (Fig.\,\ref{fig:rich_lazy}\textbf{a}). We used the cross-entropy loss on the softmax of the output:
\[
\mathcal{L}(W_1, W_2) = - \sum_{k=1}^{K} \bar y_k(\theta)\,\log\left(\mathrm{softmax}(\mathbf{y})_k\right)
\]
We summarise the hyperparameters below:
\begin{table}[h]
\centering
\begin{tabular}{c c c c c c c c}
\hline
$n_{\mathrm{in}}$ & $n$ & $\sigma_{\text{init}}$ (Rich) & $\sigma_{\text{init}}$ (Lazy) & lr & Iterations & Batch size & Optimiser \\
\hline
$100$ & $500$ & $0.1$ & $10$ & $1e-3$ & $300$ & $128^2$ & Adam\cite{kingma2014adam} \\
\hline
\end{tabular}
\end{table}

\rd{\textbf{Supplementary analyses.} In main text we found through MSA that the intrinsic geometry of rich and lazy networks trained on this task differed. We therefore performed additional analyses to better understand \textit{how} they differed. We started by computing two classic quantities describing the intrinsic geometry of a manifold: the volume density and Dirichlet density. First, for a map $\varphi:\mathcal{M}\rightarrow\mathbb{R}^n$, the volume form is denoted by $\mathrm{dvol}_{g^{\varphi}}$, where $g^\varphi$ is the pullback metric. In the coordinates of the input manifold it is quantified by the volume density $\sqrt{\mathrm{det}(G^{\varphi})}$. 

Second, the total Dirichlet energy of $\varphi$ is defined as $\frac{1}{2}\int_{\mathcal{M}}\lVert \mathrm{d}\varphi \rVert_{g}^2\mathrm{dvol}_g(p)$. The term in the integral can be rewritten as:
\[
\|\mathrm{d}\varphi\|_g^2 = \sum_{k=1}^n \|\mathrm{d}\varphi_k\|_g^2
= \sum_{k=1}^n g^{-1}(\mathrm{d}\varphi_k,\mathrm{d}\varphi_k)
\]
In local coordinates this is given by:
\[
      \|\mathrm{d}\varphi\|_g^2 =  \sum_{k=1}^n\sum_{i,j=1}^{m,m}G^{-1}_{ij}
    \partial_i\varphi_k\,\partial_j\varphi_k = \sum_{i,j=1}^{m,m} G^{-1}_{ij}G^\varphi_{ij}\\
    = \mathrm{tr}(G^{-1}G^\varphi)
\]
where $g$ is the original metric on $\mathcal{M}$, and $G$ its matrix representation in local coordinates. We call this quantity the Dirichlet density, up to the $1/2$ factor in front of the integral. Note that the Dirichlet density scale linearly with the hidden-layer dimension $n$, and the volume density by $n^{m/2}$; hence when plotting, we rescale these quantities by $n$ (since here $m=2$). 

\begin{figure}[h]
    \centering
    \includegraphics[width=0.9\linewidth]{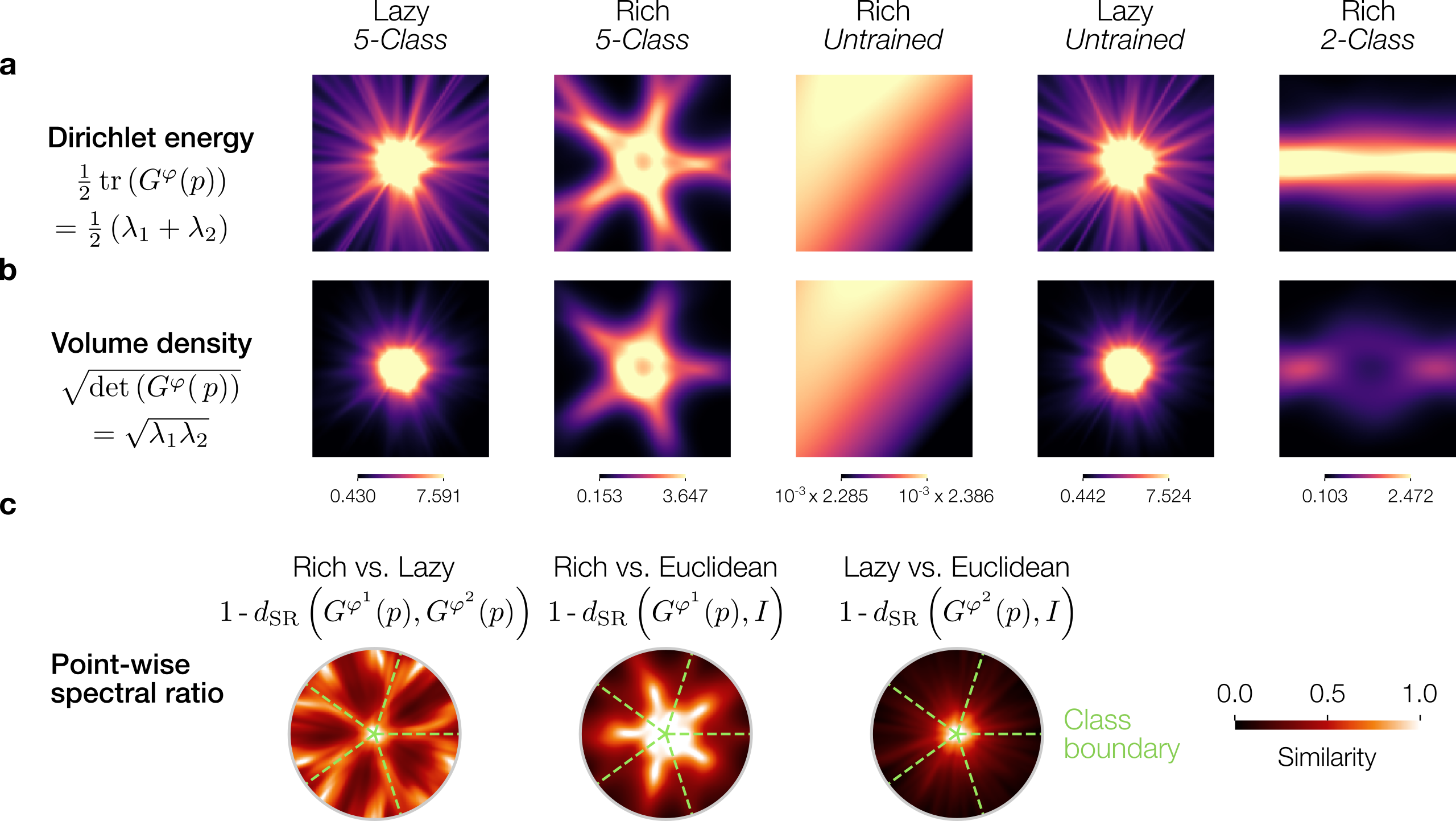}
    \caption{\rd{\textbf{Rich and lazy networks have different intrinsic geometries.} \textbf{a.} Point-wise Dirichlet energy. \textbf{b.} Point-wise volume element of the pullback metric. Both are normalised by the hidden layer dimension. \textbf{c.} Point-wise spectral ratio between different pairs of models, on the submanifold of inputs that is a disk to highlight symmetries.}}
    \label{sup_fig:rich_lazy_geometry}
\end{figure}

We computed these quantities for rich and lazy networks, both before and after training (Fig.\ref{sup_fig:rich_lazy_geometry}\textbf{a-b}). We found that rich networks learned highly structured intrinsic geometric representations where the Dirichlet energy (and to a lesser extent the volume density) was higher along class boundaries. In comparison, the lazy model retained its pre-training representation which was mostly random. This is consistent with extrinsic accounts of lazy networks operating in the random feature kernel regime\cite{gerace2020generalisation}.  To better understand these differences, we computed the spectral ratio point-wise over the manifold, between different pairs of models (Fig.\ref{sup_fig:rich_lazy_geometry}\textbf{c}). We found that the low MSA similarity between rich and lazy reported in the main text was mostly due to large spectral ratio distances near class boundaries, while the manifold was closer to Euclidean geometry near class centres. Importantly, MSA was able to reveal these subtle differences in intrinsic geometry between rich and lazy networks. Where extrinsic geometric methods had the same rankings as MSA (MSA vs. Procrustes, RSA, kCKA all had Spearman rank correlation of 1, Fig. \ref{fig:rich_lazy2}b; cf. rank correlations of 0.74 for CCA and 0.8 for symmetrized least squares), these methods reported near-perfect similarity between rich and lazy (Fig. \ref{fig:rich_lazy}d).}

\rd{To further investigate the effect of training on the intrinsic geometry of representations, we computed the full similarity matrix between pre- and post-training models and found a hierarchical structure (Fig.\,\ref{sup_fig:rich_lazy_training}\textbf{a}). We used multi-dimensional scaling (MDS) to visualize the matrix of mean similarities by computing eigenvectors satisfying $\tilde S_{\mathrm{MSA}}\mathbf{v}_i=\lambda_i\mathbf{v}_i$ where $\tilde S_{\mathrm{MSA}}=HS_{\mathrm{MSA}} H$ is the doubly-centred similarity matrix with centring matrix $H = I - \frac{1}{n}\mathds{1}\mathds{1}^\top$. The $i$th MDS coordinate of model $j$ is then $x_i^{(j)}=\sqrt{\lambda}_iv_{i,j}$. From this, linear discriminant analysis could distinguish between rich vs. lazy with perfect classification accuracy.

\begin{figure}[h]
    \centering
    \includegraphics[width=0.9\linewidth]{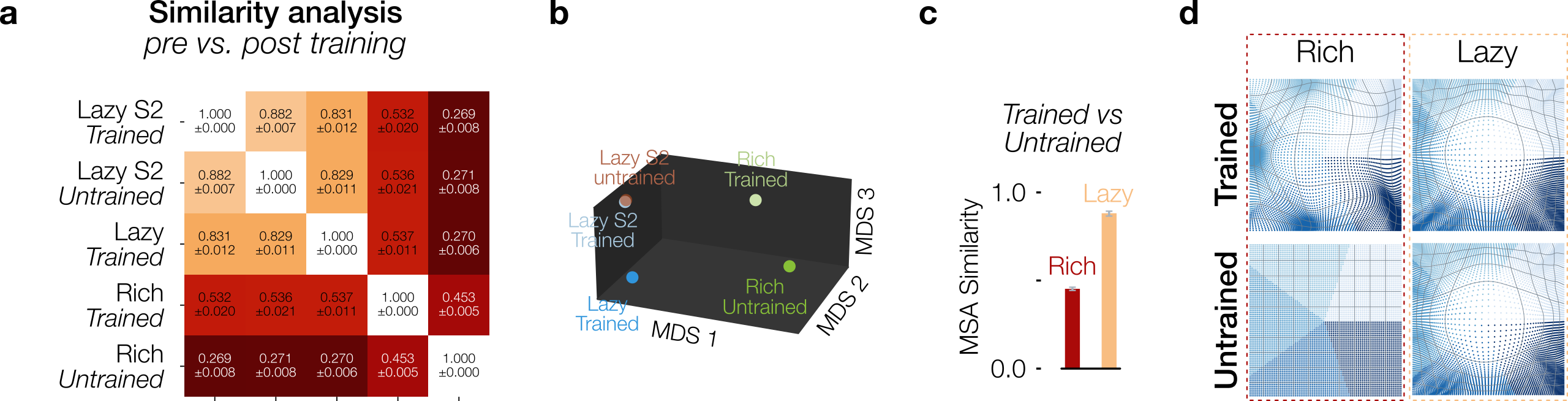}
    \caption{\rd{\textbf{Lazy network geometry remains close to initialisation.} \textbf{a.} MSA similarity between pairs of trained and untrained models highlighting a hierarchical structure of model similarities. \textbf{b.} MDS embedding of the MSA similarity matrix. Trained and untrained lazy models from the same seed (denoted by ``S2'' for Seed 2) are close in the MDS space. \textbf{c.} Similarity between trained and untrained rich models and trained vs. untrained lazy models. Lazy models remain close to initialisation. \textbf{d.} Mesh grids of the intrinsic geometries of the rpepresentations for rich and lazy models, before and after training.}}
    \label{sup_fig:rich_lazy_training}

\end{figure}

While rich models were well separated pre- and post-training, the lazy model differences primarily came from seed-to-seed variability (Fig.\,\ref{sup_fig:rich_lazy_training}\textbf{b-c}). This was also evident in the one-dimensional geodesics grid before and after training (Fig.\,\ref{sup_fig:rich_lazy_training}\textbf{d}), motivating the need for precise, quantitative measures to compare intrinsic geometries.  Rich models went from a completely Euclidean geometry pre-initialisation to a highly structure geometry pushing the points near class centres. The Euclidean pre-training geometry of rich  models can be explained by the fact that $\mathrm{tanh}$ is close to the identity for small $\sigma$, so that $J^\top J\approx W^\top W \approx c I_2$ for random isotropic Gaussian entries for $W$. Instead, lazy models preserved their random pre-training geometry, which is explained by the random saturation of $\mathrm{tanh}$ to $\pm 1$ for large $\sigma_\text{init}$. 
}

\subsection{Nonlinear dynamics of SSMs and RNNs}

\rd{In Section \ref{section:dynamics} we consider RNN and SSM models of a neuroscience working memory task.\cite{pellegrino2025rnns} The models sequentially received two inputs drawn from a circular manifold, and were trained to output the same sequence after a variable delay. The dynamics were given by:
\[
    \mathbf{\dot x} = f(\mathbf{x}) + B\mathbf{u}(t) + \text{go}_h(t)\mathbf{c}\in\mathbb{R}^n
\]
where $f(\mathbf{x})$ was either the RNN or SSM dynamics as given in the main text. Input parameters were $\mathbf{u}(t)=[\cos\theta(t), \sin\theta(t)]\mathds{1}_{t\in [0, 0.5]\cup [1.0,1.5]}$ where $\theta(t)$ was piece-wise constant, and $h\in [0.5, 3.5]$ was the delay period, ended by a pulse $\text{go}_h(t)=\mathds{1}_{t\in [1.5+h,2.0+h]}$. The state was initialised at $\mathbf{x}(0) = 0$. }

The architectural parameters are summarised below:

\begin{table}[h]
\centering
\begin{tabular}{c c c c}
\hline
$n$ & $\sigma_{\text{init}}$ (RNN) & HiPPO (SSM) & Blocks (SSM) \\
\hline
$100$ & $0.5$ & Normal LegS & $2$\\
\hline
\end{tabular}
\end{table}

The training parameters are:

\begin{table}[h!]
\centering
\begin{tabular}{c c c c c}
\hline
lr & Iterations & Batch size & Optimiser & ODE Solver \\
\hline
$5e-3$ & $3000$ & $7^2$ & Adam\cite{kingma2014adam} & Heun\cite{kidger2022neural} \\
\hline
\end{tabular}
\end{table}

The loss was:
\[
    \mathcal{L}(A,B, D) = \textstyle\int_0^{T-h}||D\phi(\mathbf{x}(t+h+1.0)-\mathbf{u}(t)||^2dt
\]

\rd{For DSA, we used the original package\cite{ostrow2023beyond} with the following hyperparameters: \texttt{O(n) group}, \texttt{n\_delays = 10}, \texttt{delay\_interval = 5} for data with 200 time samples. To compute the similarity we either used the \texttt{Wasserstein} distance on the eigenvalues, which we transformed into a similarity via $1-\frac{D}{D_{max}}$ (Figure \ref{fig:rnnsmm}d) or \texttt{angular}, which we transformed into a cosine similarity by taking the cosine (Figure \ref{fig:rnnsmm}f). Results were qualitatively similar between the two choices (Spearman rank correlation of 1).

Importantly, Spearman rank correlations between MSA and other methods were below 1, demonstrating that MSA is not simply a monotonic transformation of these methods.

\begin{table}[h]
\centering\small
\begin{tabular}{lcccc}
 & \textbf{RSA} & \textbf{DSA (Cos)} & \textbf{DSA (Wass)} & \textbf{MSA} \\
\hline
\textbf{RSA}      & 1.0 & 0.4 & 0.4 & 0.8 \\
\textbf{DSA (Cos)} & 0.4 & 1.0 & 1.0 & 0.8 \\
\textbf{DSA (Wass)} & 0.4 & 1.0 & 1.0 & 0.8 \\
\textbf{MSA} & 0.8 & 0.8 & 0.8 & 1.0 \\
\end{tabular}
\end{table}
}

\subsection{Diffusion Model SDE and Probability Flow}
\label{app:diffusion}

In Section \ref{section:diffusion} we analysed a pre-trained diffusion model ``StableDiffusionXL''\cite{podell2023sdxl}. In short, a diffusion model is defined by a backward-in-time stochastic differential equation (SDE):
\begin{equation}
    \mathbf{dx} = \mathbf{f}(\mathbf{x}, {\boldsymbol\pi}, t)\mathrm{dt} + \sigma(t)\mathrm{dB}, \quad \mathbf{x}(1)\sim \mathcal{N}(\boldsymbol{0}, I)
\end{equation}
where $B$ is a Brownian motion and $\boldsymbol{\pi}\in\mathcal{M}$ is a variable that deterministically controls the diffusion process (e.g., a text prompt for conditional image generation). The model can be summarised as:
\begin{equation}
    \mathbb{R}^n  \xlongrightarrow{\varphi_{\pi}} \mathbb{R}^n, \qquad \mathbf{x}(1) \rightarrow \mathbf{x}(0)
\end{equation}
where $\varphi_\pi$ is the flow of the system, and $\mathbf{x}(0)$ the generated sample (e.g., an image). We used a manifold of text embeddings parameterised via bilinear interpolation of text embeddings:
\[
\mathbf{z}(\boldsymbol{\pi})=\pi_1\pi_2 \mathbf{v}_1+(1-\pi_1)\pi_2 \mathbf{v}_2+\pi_1(1-\pi_2) \mathbf{v}_3+(1-\pi_1)(1-\pi_2) \mathbf{v}_4.
\]
where $\mathbf{v}_i$ is a particular text embedding (Fig.\,\ref{fig:diffusion}). This SDE is associated with a corresponding probability flow:
\begin{equation}
    \partial_t p_{\boldsymbol{\pi}}(\mathbf{x}, t) =  - \nabla_{\mathbf{x}} \cdot \left(\mathbf{f}(\mathbf{x}, \boldsymbol{\pi}, t)p_{\boldsymbol{\pi}}(\mathbf{x}, t)\right) + \frac{1}{2}\sigma^2(t) \Delta_{\mathbf{x}}p_{\boldsymbol{\pi}}(\mathbf{x}, t)
\end{equation}
where $\nabla \cdot$ is the divergence, $\Delta$ the Laplacian, and $p_{\boldsymbol{\pi}}(\mathbf{x}(1), 1) = \mathcal{N}(\boldsymbol{0}, I)$. Solving this flow at $p_{\boldsymbol{\pi}}(\mathbf{x}(0), 0)$ gives the marginal density of the learned generative distribution of the SDE at $\mathbf{x}(0)$. In our guidance experiment we consider an estimate of the marginal field given by an empty string embedding:
\[
    \mathbf{dx} = \bigl( \gamma\mathrm{UNet}(\mathbf{x}, \mathrm{TextEmb}(\text{\footnotesize``a leaf''}), t) +(1- \gamma) \mathrm{UNet}(\mathbf{x}, \mathrm{TextEmb}(\text{\footnotesize`` ''}), t)\bigr) \mathrm{dt}+\sigma(t) \mathrm{dB}
\]

\begin{figure}[h]
    \centering
    \includegraphics[width=0.6\linewidth]{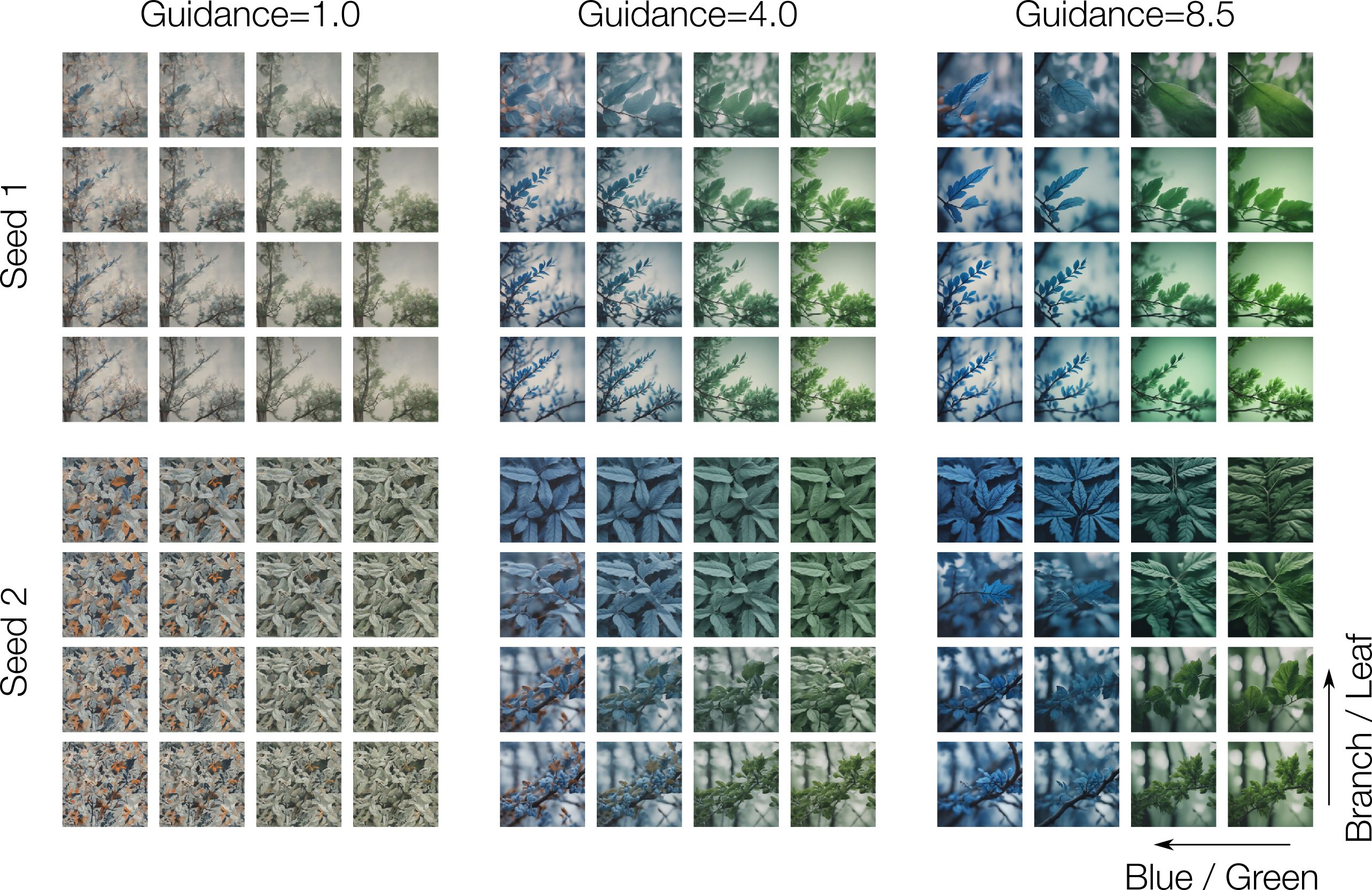}
    \caption{\textbf{Guidance improves prompt alignment.} Generated samples over the manifold of prompt embeddings (inner square), and across seed and guidance levels.}
    \label{sup_fig:diffusion_samples}
\end{figure}

The hyperparameters were given by:
\begin{table}[h]
\centering
\begin{tabular}{ll}
\hline
ODE Solver & Euler \\
Noise samples & 15 \\
Time points & 20 \\
Guidance levels $\gamma$ & $\{1, 2.5, 4.0, 5.5, 7.0, 8.5\}$ \\
$\pi_i$ values & $\{0, 0.25, 0.5, 0.75, 1.0\}$ \\
Total data points & 28{,}800 \\
\hline
\end{tabular}
\end{table}

\subsection{Similarity between layers of ResNet}\label{app:resnet}

To highlight the applicability of MSA in a broader range of context and to large-scale networks and image datasets, we also applied it to ResNet18\cite{he2015deep} on ImageNet. Previous work has suggested that the intrinsic geometry of ResNet progressively changes across layers\cite{kaul2019riemannian} and more generally prior theoretical work has suggested that deep neural networks progressively warp the inputs across layer during classification tasks\cite{hauser2017principles}. 

The input manifold was generated with a variational autoencoder (VAE; \texttt{sd-vae-ft-mse} from Stable Diffusion\cite{rombach2022high}). We aimed to interpolate a two-dimensional manifold defined by the simplex spanned by the image of three target images $\mathbf{x}_i\in\mathbb{R}^{3\cdot n_\mathrm{pix}}$ in the latent space. Specifically, we passed each target image through the VAE encoder, to get its image in the latent space:
\[
\mathbf{z}_i = \mathrm{Encoder}(\mathbf{x}_i) \in \mathbb{R}^{n_\mathrm{latent}}
\]
Then, we interpolated between the three $\mathbf{z}_i$ in the latent space to define a simplex. After adding noise, this resulting in the following samples:
\[
\mathbf{\bar z}(\mathbf{a}) = a_1 \mathbf{z}_1 +  a_2 \mathbf{z}_2  +  a_3 \mathbf{z}_3 + \mathcal{N}(0,\sigma^2I), \qquad a_1 +a_2+a_3 \leq 1
\]
We then passed these through the VAE decoder to generate samples from the curved manifold $\mathcal{M}$:
\rd{\[
\mathbf{y}(\mathbf{a}) = \mathrm{Decoder}(\mathbf{\bar z}(\mathbf{a})) 
\]}
These samples were then used to calculate the similarity between layers $\mathrm{MSA}(\varphi_i,\varphi_j)$.

\begin{figure}[h]
    \centering
    \includegraphics[width=\linewidth]{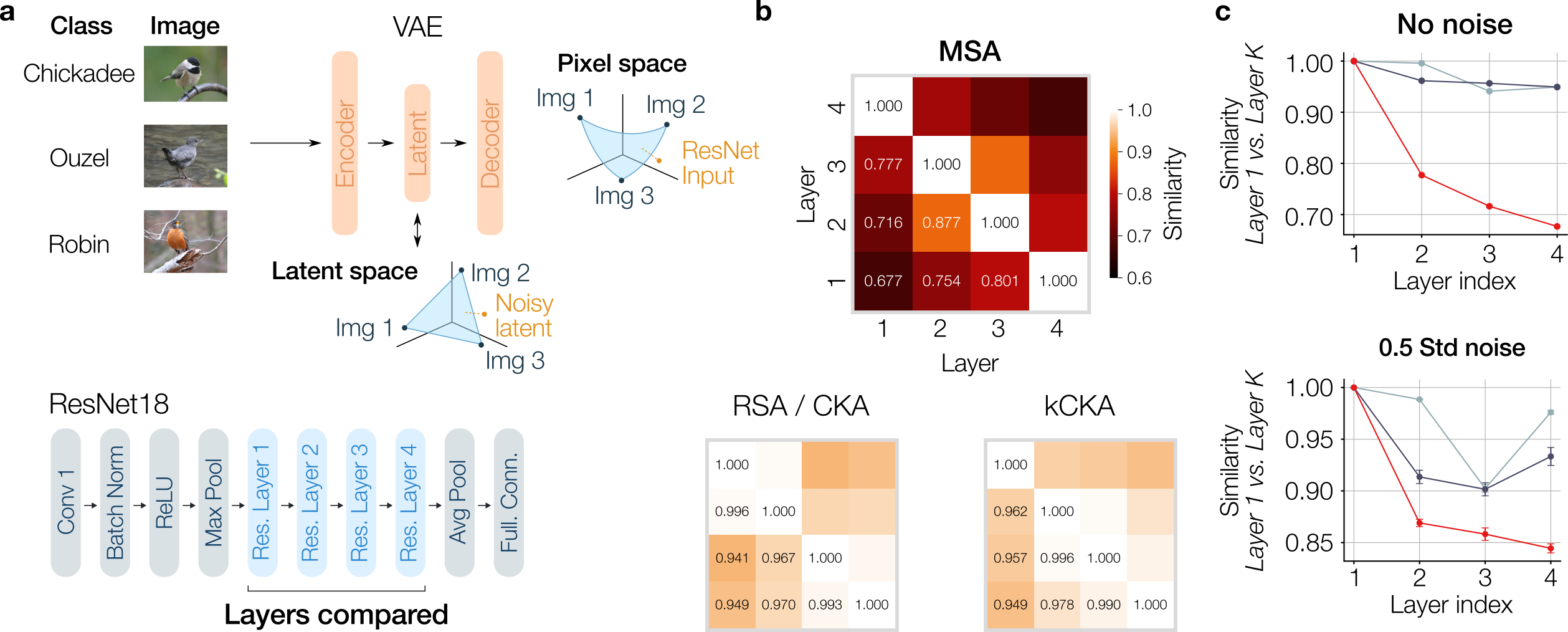}
    \caption{\textbf{The Riemannian geometry of ResNet18 changes from layer to layer.} \textbf{a.} Architecture and dataset. Following previous work\cite{kaul2019riemannian}, we constructed a manifold of ImageNet images across different classes using a simplex. In contrast to previous work, we constructed a curved manifold by interpolating a flat simplex in the latent layer of a VAE, such that the decoded simplex is curved. We passed this curved manifold of images as inputs to ResNet18, and compared the residual layer's geometry. \textbf{b.} Similarity matrix across layers. MSA suggests a progressive change in intrinsic geometry across layers. \textbf{c.} To test the sensitivity of different methods to noise, we added noise to the latent manifold before decoding it.}
    \label{sup_fig:resnet}
\end{figure}

\clearpage

\AtNextBibliography{\small}
\printbibliography[
    segment=\therefsegment,
    heading=subbibliography,
    title={\large Supplementary References},
    filter=supprefs    
]

@article{amari1998natural,
  title={Natural gradient works efficiently in learning},
  author={Amari, Shun-Ichi},
  journal={Neural computation},
  volume={10},
  number={2},
  pages={251--276},
  year={1998},
  publisher={MIT Press}
}

@article{arvanitidis2017latent,
  title={Latent space oddity: on the curvature of deep generative models. arXiv},
  author={Arvanitidis, Georgios and Hansen, Lars Kai and Hauberg, S{\o}ren},
  journal={Preprint posted online October},
  volume={31},
  year={2017}
}

@article{azeglio2026multi,
  title={A multi-scale information geometry reveals the structure of mutual information in neural populations},
  author={Azeglio, Simone and Laquitaine, Steeve and Ferrari, Ulisse and Chalk, Matthew},
  journal={arXiv preprint arXiv:2605.06304},
  year={2026}
}

@article{barak2017recurrent,
  title={Recurrent neural networks as versatile tools of neuroscience research},
  author={Barak, Omri},
  journal={Current opinion in neurobiology},
  volume={46},
  pages={1--6},
  year={2017},
  publisher={Elsevier}
}

@article{barbosa2025quantifying,
  title={Quantifying Differences in Neural Population Activity With Shape Metrics},
  author={Barbosa, Joao and Nejatbakhsh, Amin and Duong, Lyndon and Harvey, Sarah E and Brincat, Scott L and Siegel, Markus and Miller, Earl K and Williams, Alex H},
  journal={bioRxiv},
  pages={2025--01},
  year={2025},
  publisher={Cold Spring Harbor Laboratory}
}

@article{birkhoff1957extensions,
  title={Extensions of Jentzsch’s theorem},
  author={Birkhoff, Garrett},
  journal={Transactions of the American Mathematical Society},
  volume={85},
  number={1},
  pages={219--227},
  year={1957}
}

@article{boix2022gulp,
  title={GULP: a prediction-based metric between representations},
  author={Boix-Adsera, Enric and Lawrence, Hannah and Stepaniants, George and Rigollet, Philippe},
  journal={Advances in Neural Information Processing Systems},
  volume={35},
  pages={7115--7127},
  year={2022}
}

@InProceedings{brandon2025emergent,
  title = 	 {Emergent Riemannian geometry over learning discrete computations on continuous manifolds},
  author =       {Brandon, Julian and Chadwick, Angus and Pellegrino, Arthur},
  booktitle = 	 {Proceedings of the 4th (2025) and 3rd (2024) NeurIPS Workshops on Symmetry and Geometry in Neural Representations},
  pages = 	 {49--72},
  year = 	 {2026},
  volume = 	 {282},
  series = 	 {Proceedings of Machine Learning Research},
  month = 	 {14 Dec 2024--07 Dec 2025},
  publisher =    {PMLR},
  url = 	 {https://proceedings.mlr.press/v282/brandon26a.html},
}

@article{brandon2026balance,
  title={The balance between feature learning and collapse in generative dynamical systems models},
  author={Brandon, Julian and Loureiro, Bruno and Cayco-Gajic, N Alex and Pellegrino, Arthur},
  journal={Advances in Neural Information Processing Systems},
  year={2026}
}

@article{bushell1973hilbert,
  title={Hilbert's metric and positive contraction mappings in a Banach space},
  author={Bushell, Peter J},
  journal={Archive for Rational Mechanics and Analysis},
  volume={52},
  number={4},
  pages={330--338},
  year={1973},
  publisher={Springer}
}

@article{chen2024dform,
  title={Dform: Diffeomorphic vector field alignment for assessing dynamics across learned models},
  author={Chen, Ruiqi and Vedovati, Giacomo and Braver, Todd and Ching, ShiNung},
  journal={arXiv preprint arXiv:2402.09735},
  year={2024}
}

@inproceedings{choufeature,
  title={Feature Learning beyond the Lazy-Rich Dichotomy: Insights from Representational Geometry},
  author={Chou, Chi-Ning and Le, Hang and Wang, Yichen and Chung, SueYeon},
  booktitle={Forty-second International Conference on Machine Learning},
year={2025}
}

@article{chizat2019lazy,
  title={On lazy training in differentiable programming},
  author={Chizat, Lenaic and Oyallon, Edouard and Bach, Francis},
  journal={Advances in neural information processing systems},
  volume={32},
  year={2019}
}

@inproceedings{gretton2005measuring,
  title={Measuring statistical dependence with Hilbert-Schmidt norms},
  author={Gretton, Arthur and Bousquet, Olivier and Smola, Alex and Sch{\"o}lkopf, Bernhard},
  booktitle={International conference on algorithmic learning theory},
  pages={63--77},
  year={2005},
  organization={Springer}
}

@inproceedings{rombach2022high,
  title={High-resolution image synthesis with latent diffusion models},
  author={Rombach, Robin and Blattmann, Andreas and Lorenz, Dominik and Esser, Patrick and Ommer, Bj{\"o}rn},
  booktitle={Proceedings of the IEEE/CVF conference on computer vision and pattern recognition},
  pages={10684--10695},
  year={2022}
}

@article{karczewski2025spacetime,
  title={The Spacetime of Diffusion Models: An Information Geometry Perspective},
  author={Karczewski, Rafa{\l} and Heinonen, Markus and Pouplin, Alison and Hauberg, S{\o}ren and Garg, Vikas},
  journal={arXiv preprint arXiv:2505.17517},
  year={2025}
}

@article{zavatone2023does,
  title={How does training shape the Riemannian geometry of neural network representations?},
  author={Zavatone-Veth, Jacob A and Yang, Sheng and Rubinfien, Julian A and Pehlevan, Cengiz},
  journal={arXiv preprint arXiv:2301.11375},
  year={2023}
}

@article{tennenholtz2022uncertainty,
  title={Uncertainty estimation using riemannian model dynamics for offline reinforcement learning},
  author={Tennenholtz, Guy and Mannor, Shie},
  journal={Advances in Neural Information Processing Systems},
  volume={35},
  pages={19008--19021},
  year={2022}
}

@article{gurnee2026models,
  title={When models manipulate manifolds: The geometry of a counting task},
  author={Gurnee, Wes and Ameisen, Emmanuel and Kauvar, Isaac and Tarng, Julius and Pearce, Adam and Olah, Chris and Batson, Joshua},
  journal={arXiv preprint arXiv:2601.04480},
  year={2026}
}

@article{valente2022extracting,
  title={Extracting computational mechanisms from neural data using low-rank rnns},
  author={Valente, Adrian and Pillow, Jonathan W and Ostojic, Srdjan},
  journal={Advances in Neural Information Processing Systems},
  volume={35},
  pages={24072--24086},
  year={2022}
}

@article{mante2013context,
  title={Context-dependent computation by recurrent dynamics in prefrontal cortex},
  author={Mante, Valerio and Sussillo, David and Shenoy, Krishna V and Newsome, William T},
  journal={nature},
  volume={503},
  number={7474},
  pages={78--84},
  year={2013},
  publisher={Nature Publishing Group UK London}
}

@inproceedings{gerace2020generalisation,
  title={Generalisation error in learning with random features and the hidden manifold model},
  author={Gerace, Federica and Loureiro, Bruno and Krzakala, Florent and M{\'e}zard, Marc and Zdeborov{\'a}, Lenka},
  booktitle={International Conference on Machine Learning},
  pages={3452--3462},
  year={2020},
  organization={PMLR}
}

@article{jacot2021saddle,
  title={Saddle-to-saddle dynamics in deep linear networks: Small initialization training, symmetry, and sparsity},
  author={Jacot, Arthur and Ged, Fran{\c{c}}ois and {\c{S}}im{\c{s}}ek, Berfin and Hongler, Cl{\'e}ment and Gabriel, Franck},
  journal={arXiv preprint arXiv:2106.15933},
  year={2021}
}

@article{atanasov2021neural,
  title={Neural networks as kernel learners: The silent alignment effect},
  author={Atanasov, Alexander and Bordelon, Blake and Pehlevan, Cengiz},
  journal={arXiv preprint arXiv:2111.00034},
  year={2021}
}

@book{amari2016information,
  title={Information geometry and its applications},
  author={Amari, Shun-ichi},
  year={2016},
  publisher={Springer}
}

@article{durstewitz2023reconstructing,
  title={Reconstructing computational system dynamics from neural data with recurrent neural networks},
  author={Durstewitz, Daniel and Koppe, Georgia and Thurm, Max Ingo},
  journal={Nature Reviews Neuroscience},
  volume={24},
  number={11},
  pages={693--710},
  year={2023},
  publisher={Nature Publishing Group UK London}
}

@inproceedings{feather2025discriminating,
  title={Discriminating image representations with principal distortions},
  author={Feather, Jenelle and Lipshutz, David and Harvey, Sarah and Williams, Alex and Simoncelli, Eero},
  booktitle={International Conference on Learning Representations},
  volume={2025},
  pages={55022--55048},
  year={2025}
}

@incollection{forstner2003metric,
  title={A metric for covariance matrices},
  author={F{\"o}rstner, Wolfgang and Moonen, Boudewijn},
  booktitle={Geodesy-the Challenge of the 3rd Millennium},
  pages={299--309},
  year={2003},
  publisher={Springer}
}

@article{ganguli2008memory,
  title={Memory traces in dynamical systems},
  author={Ganguli, Surya and Huh, Dongsung and Sompolinsky, Haim},
  journal={Proceedings of the national academy of sciences},
  volume={105},
  number={48},
  pages={18970--18975},
  year={2008},
  publisher={National Academy of Sciences}
}

@article{goldman2009memory,
  title={Memory without feedback in a neural network},
  author={Goldman, Mark S},
  journal={Neuron},
  volume={61},
  number={4},
  pages={621--634},
  year={2009},
  publisher={Elsevier}
}

@article{gosztolai2025marble,
  title={MARBLE: interpretable representations of neural population dynamics using geometric deep learning},
  author={Gosztolai, Adam and Peach, Robert L and Arnaudon, Alexis and Barahona, Mauricio and Vandergheynst, Pierre},
  journal={Nature Methods},
  pages={1--9},
  year={2025},
  publisher={Nature Publishing Group US New York}
}

@inproceedings{grave2019unsupervised,
  title={Unsupervised alignment of embeddings with wasserstein procrustes},
  author={Grave, Edouard and Joulin, Armand and Berthet, Quentin},
  booktitle={The 22nd International Conference on Artificial Intelligence and Statistics},
  pages={1880--1890},
  year={2019},
  organization={PMLR}
}

@article{gruffaz2025riemannian,
  title={Riemannian metric learning: Closer to you than you imagine},
  author={Gruffaz, Samuel and Sassen, Josua},
  journal={arXiv preprint arXiv:2503.05321},
  year={2025}
}

@article{harvey2024representational,
  title={What representational similarity measures imply about decodable information},
  author={Harvey, Sarah E and Lipshutz, David and Williams, Alex H},
  journal={arXiv preprint arXiv:2411.08197},
  year={2024}
}

@article{hauser2017principles,
  title={Principles of Riemannian geometry in neural networks},
  author={Hauser, Michael and Ray, Asok},
  journal={Advances in neural information processing systems},
  volume={30},
  year={2017}
}

@article{gu2022train,
  title={How to train your hippo: State space models with generalized orthogonal basis projections},
  author={Gu, Albert and Johnson, Isys and Timalsina, Aman and Rudra, Atri and R{\'e}, Christopher},
  journal={arXiv preprint arXiv:2206.12037},
  year={2022}
}

@article{ho2022classifier,
  title={Classifier-free diffusion guidance},
  author={Ho, Jonathan and Salimans, Tim},
  journal={arXiv preprint arXiv:2207.12598},
  year={2022}
}

@article{huang2025inputdsa,
  title={Inputdsa: Demixing then comparing recurrent and externally driven dynamics},
  author={Huang, Ann and Ostrow, Mitchell and Singh, Satpreet H and Kozachkov, Leo and Fiete, Ila and Rajan, Kanaka},
  journal={arXiv preprint arXiv:2510.25943},
  year={2025}
}

@article{jacot2018neural,
  title={Neural tangent kernel: Convergence and generalization in neural networks},
  author={Jacot, Arthur and Gabriel, Franck and Hongler, Cl{\'e}ment},
  journal={Advances in neural information processing systems},
  volume={31},
  year={2018}
}

@article{johnston2023abstract,
  title={Abstract representations emerge naturally in neural networks trained to perform multiple tasks},
  author={Johnston, W Jeffrey and Fusi, Stefano},
  journal={Nature Communications},
  volume={14},
  number={1},
  pages={1040},
  year={2023},
  publisher={Nature Publishing Group UK London}
}

@article{kaul2019riemannian,
  title={Riemannian curvature of deep neural networks},
  author={Kaul, Piyush and Lall, Brejesh},
  journal={IEEE transactions on neural networks and learning systems},
  volume={31},
  number={4},
  pages={1410--1416},
  year={2019},
  publisher={IEEE}
}

@article{khona2022attractor,
  title={Attractor and integrator networks in the brain},
  author={Khona, Mikail and Fiete, Ila R},
  journal={Nature Reviews Neuroscience},
  volume={23},
  number={12},
  pages={744--766},
  year={2022},
  publisher={Nature Publishing Group UK London}
}

@article{klabunde2025similarity,
  title={Similarity of neural network models: A survey of functional and representational measures},
  author={Klabunde, Max and Schumacher, Tobias and Strohmaier, Markus and Lemmerich, Florian},
  journal={ACM Computing Surveys},
  volume={57},
  number={9},
  pages={1--52},
  year={2025},
  publisher={ACM New York, NY}
}

@inproceedings{kornblith2019similarity,
  title={Similarity of neural network representations revisited},
  author={Kornblith, Simon and Norouzi, Mohammad and Lee, Honglak and Hinton, Geoffrey},
  booktitle={International conference on machine learning},
  pages={3519--3529},
  year={2019},
  organization={PMlR}
}

@article{kriegeskorte2008representational,
  title={Representational similarity analysis-connecting the branches of systems neuroscience},
  author={Kriegeskorte, Nikolaus and Mur, Marieke and Bandettini, Peter A},
  journal={Frontiers in systems neuroscience},
  volume={2},
  pages={249},
  year={2008},
  publisher={Frontiers}
}

@incollection{lee2003smooth,
  title={Smooth manifolds},
  author={Lee, John M},
  booktitle={Introduction to smooth manifolds},
  pages={1--29},
  year={2003},
  publisher={Springer}
}

@article{li2015convergent,
  title={Convergent learning: Do different neural networks learn the same representations?},
  author={Li, Yixuan and Yosinski, Jason and Clune, Jeff and Lipson, Hod and Hopcroft, John},
  journal={arXiv preprint arXiv:1511.07543},
  year={2015}
}

@article{li2024representations,
  title={Representations and generalization in artificial and brain neural networks},
  author={Li, Qianyi and Sorscher, Ben and Sompolinsky, Haim},
  journal={Proceedings of the National Academy of Sciences},
  volume={121},
  number={27},
  pages={e2311805121},
  year={2024},
  publisher={National Academy of Sciences}
}

@article{lin2024topology,
  title={The topology and geometry of neural representations},
  author={Lin, Baihan and Kriegeskorte, Nikolaus},
  journal={Proceedings of the National Academy of Sciences},
  volume={121},
  number={42},
  pages={e2317881121},
  year={2024},
  publisher={National Academy of Sciences}
}

@article{lin2019riemannian,
  title={Riemannian geometry of symmetric positive definite matrices via Cholesky decomposition},
  author={Lin, Zhenhua},
  journal={SIAM Journal on Matrix Analysis and Applications},
  volume={40},
  number={4},
  pages={1353--1370},
  year={2019},
  publisher={SIAM}
}

@article{liverani1994generalization,
  title={Generalization of the Hilbert metric to the space of positive definite matrices.},
  author={Liverani, Carlangelo and Wojtkowski, Maciej P},
  journal={Pacific Journal of Mathematics},
  volume={166},
  number={2},
  year={1994}
}

@article{maheswaranathan2019universality,
  title={Universality and individuality in neural dynamics across large populations of recurrent networks},
  author={Maheswaranathan, Niru and Williams, Alex and Golub, Matthew and Ganguli, Surya and Sussillo, David},
  journal={Advances in neural information processing systems},
  volume={32},
  year={2019}
}

@article{mao2024training,
  title={The training process of many deep networks explores the same low-dimensional manifold},
  author={Mao, Jialin and Griniasty, Itay and Teoh, Han Kheng and Ramesh, Rahul and Yang, Rubing and Transtrum, Mark K and Sethna, James P and Chaudhari, Pratik},
  journal={Proceedings of the National Academy of Sciences},
  volume={121},
  number={12},
  pages={e2310002121},
  year={2024},
  publisher={National Academy of Sciences}
}

@article{mostajeran2024differential,
  title={Differential geometry with extreme eigenvalues in the positive semidefinite cone},
  author={Mostajeran, Cyrus and Da Costa, Natha{\"e}l and Van Goffrier, Graham and Sepulchre, Rodolphe},
  journal={SIAM Journal on Matrix Analysis and Applications},
  volume={45},
  number={2},
  pages={1089--1113},
  year={2024},
  publisher={SIAM}
}

@book{nussbaum1988hilbert,
  title={Hilbert's projective metric and iterated nonlinear maps},
  author={Nussbaum, Roger D},
  volume={1},
  year={1988},
  publisher={American Mathematical Soc.}
}

@article{ostrow2023beyond,
  title={Beyond geometry: Comparing the temporal structure of computation in neural circuits with dynamical similarity analysis},
  author={Ostrow, Mitchell and Eisen, Adam and Kozachkov, Leo and Fiete, Ila},
  journal={Advances in Neural Information Processing Systems},
  volume={36},
  pages={33824--33837},
  year={2023}
}

@article{paccolat2021geometric,
  title={Geometric compression of invariant manifolds in neural networks},
  author={Paccolat, Jonas and Petrini, Leonardo and Geiger, Mario and Tyloo, Kevin and Wyart, Matthieu},
  journal={Journal of Statistical Mechanics: Theory and Experiment},
  volume={2021},
  number={4},
  pages={044001},
  year={2021},
  publisher={IOP Publishing}
}

@article{park2023understanding,
  title={Understanding the latent space of diffusion models through the lens of riemannian geometry},
  author={Park, Yong-Hyun and Kwon, Mingi and Choi, Jaewoong and Jo, Junghyo and Uh, Youngjung},
  journal={Advances in Neural Information Processing Systems},
  volume={36},
  pages={24129--24142},
  year={2023}
}

@article{pellegrino2025rnns,
  title={RNNs perform task computations by dynamically warping neural representations},
  author={Pellegrino, Arthur and Chadwick, Angus},
  journal={Advances in Neural Information Processing Systems},
  volume={38},
  pages={3579--3612},
  year={2026}
}

@article{podell2023sdxl,
  title={Sdxl: Improving latent diffusion models for high-resolution image synthesis},
  author={Podell, Dustin and English, Zion and Lacey, Kyle and Blattmann, Andreas and Dockhorn, Tim and M{\"u}ller, Jonas and Penna, Joe and Rombach, Robin},
  journal={arXiv preprint arXiv:2307.01952},
  year={2023}
}

@inproceedings{he2015deep,
  title={Deep residual learning for image recognition. 2016 IEEE Conf},
  author={He, Kaiming and Zhang, Xiangyu and Ren, Shaoquing and Sun, Jian},
  booktitle={Comput. Vis. Pattern Recognit},
  pages={770--778},
  year={2015}
}

@article{poole2016exponential,
  title={Exponential expressivity in deep neural networks through transient chaos},
  author={Poole, Ben and Lahiri, Subhaneil and Raghu, Maithra and Sohl-Dickstein, Jascha and Ganguli, Surya},
  journal={Advances in neural information processing systems},
  volume={29},
  year={2016}
}

@article{raghu2017svcca,
  title={Svcca: Singular vector canonical correlation analysis for deep learning dynamics and interpretability},
  author={Raghu, Maithra and Gilmer, Justin and Yosinski, Jason and Sohl-Dickstein, Jascha},
  journal={Advances in neural information processing systems},
  volume={30},
  year={2017}
}

@article{ritter2025efficient,
  title={Efficient Working Memory Maintenance via High-Dimensional Rotational Dynamics},
  author={Ritter, Laura and Chadwick, Angus},
  journal={bioRxiv},
  pages={2025--09},
  year={2025},
  publisher={Cold Spring Harbor Laboratory}
}

@article{ryoo2025generalizable,
  title={Generalizable, real-time neural decoding with hybrid state-space models},
  author={Ryoo, Avery Hee-Woon and Krishna, Nanda H and Mao, Ximeng and Azabou, Mehdi and Dyer, Eva L and Perich, Matthew G and Lajoie, Guillaume},
  journal={arXiv preprint arXiv:2506.05320},
  year={2025}
}

@article{sagodi2025dynamical,
  title={Dynamical archetype analysis: Autonomous computation},
  author={Sagodi, Abel and Park, Il Memming},
  journal={arXiv preprint arXiv:2507.05505},
  year={2025}
}

@article{pellegrino2023low,
  title={Low tensor rank learning of neural dynamics},
  author={Pellegrino, Arthur and Cayco Gajic, N Alex and Chadwick, Angus},
  journal={Advances in Neural Information Processing Systems},
  volume={36},
  pages={11674--11702},
  year={2023}
}

@article{kingma2014adam,
  title={Adam: A method for stochastic optimization},
  author={Kingma, Diederik P and Ba, Jimmy},
  journal={arXiv preprint arXiv:1412.6980},
  year={2014}
}

@article{kidger2022neural,
  title={On neural differential equations},
  author={Kidger, Patrick},
  journal={arXiv preprint arXiv:2202.02435},
  year={2022}
}

@article{saxe2019mathematical,
  title={A mathematical theory of semantic development in deep neural networks},
  author={Saxe, Andrew M and McClelland, James L and Ganguli, Surya},
  journal={Proceedings of the National Academy of Sciences},
  volume={116},
  number={23},
  pages={11537--11546},
  year={2019},
  publisher={National Academy of Sciences}
}

@article{xie2022geometry,
  title={Geometry of sequence working memory in macaque prefrontal cortex},
  author={Xie, Yang and Hu, Peiyao and Li, Junru and Chen, Jingwen and Song, Weibin and Wang, Xiao-Jing and Yang, Tianming and Dehaene, Stanislas and Tang, Shiming and Min, Bin and others},
  journal={Science},
  volume={375},
  number={6581},
  pages={632--639},
  year={2022},
  publisher={American Association for the Advancement of Science}
}

@article{saxe2013exact,
  title={Exact solutions to the nonlinear dynamics of learning in deep linear neural networks},
  author={Saxe, Andrew M and McClelland, James L and Ganguli, Surya},
  journal={arXiv preprint arXiv:1312.6120},
  year={2013}
}

@article{shang2025unraveling,
  title={Unraveling the geometry of visual relational reasoning},
  author={Shang, Jiaqi and Kreiman, Gabriel and Sompolinsky, Haim},
  journal={ArXiv},
  pages={arXiv--2502},
  year={2025}
}

@inproceedings{shao2018riemannian,
  title={The riemannian geometry of deep generative models},
  author={Shao, Hang and Kumar, Abhishek and Thomas Fletcher, P},
  booktitle={Proceedings of the IEEE Conference on Computer Vision and Pattern Recognition Workshops},
  pages={315--323},
  year={2018}
}

@article{singer2012vector,
  title={Vector diffusion maps and the connection Laplacian},
  author={Singer, Amit and Wu, H-T},
  journal={Communications on pure and applied mathematics},
  volume={65},
  number={8},
  pages={1067--1144},
  year={2012},
  publisher={Wiley Online Library}
}

@article{tenenbaum2000global,
  title={A global geometric framework for nonlinear dimensionality reduction},
  author={Tenenbaum, Joshua B and Silva, Vin de and Langford, John C},
  journal={science},
  volume={290},
  number={5500},
  pages={2319--2323},
  year={2000},
  publisher={American Association for the Advancement of Science}
}

@article{wang2018towards,
  title={Towards understanding learning representations: To what extent do different neural networks learn the same representation},
  author={Wang, Liwei and Hu, Lunjia and Gu, Jiayuan and Hu, Zhiqiang and Wu, Yue and He, Kun and Hopcroft, John},
  journal={Advances in neural information processing systems},
  volume={31},
  year={2018}
}

@article{williams2021generalized,
  title={Generalized shape metrics on neural representations},
  author={Williams, Alex H and Kunz, Erin and Kornblith, Simon and Linderman, Scott},
  journal={Advances in neural information processing systems},
  volume={34},
  pages={4738--4750},
  year={2021}
}

@article{williams2024equivalence,
  title={Equivalence between representational similarity analysis, centered kernel alignment, and canonical correlations analysis},
  author={Williams, Alex H},
  journal={bioRxiv},
  pages={2024--10},
  year={2024},
  publisher={Cold Spring Harbor Laboratory}
}

@inproceedings{woodworth2020kernel,
  title={Kernel and rich regimes in overparametrized models},
  author={Woodworth, Blake and Gunasekar, Suriya and Lee, Jason D and Moroshko, Edward and Savarese, Pedro and Golan, Itay and Soudry, Daniel and Srebro, Nathan},
  booktitle={Conference on Learning Theory},
  pages={3635--3673},
  year={2020},
  organization={PMLR}
}

@article{wurgaft2026manifold,
  title={Manifold steering reveals the shared geometry of neural network representation and behavior},
  author={Wurgaft, Daniel and Rager, Can and Kowal, Matthew and Shyam, Vasudev and Feucht, Sheridan and Bhalla, Usha and Haklay, Tal and Bigelow, Eric and Sarfati, Raphael and McGrath, Thomas and others},
  journal={arXiv preprint arXiv:2605.05115},
  year={2026}
}

@inproceedings{zhangkoopstd,
  title={KoopSTD: Reliable Similarity Analysis between Dynamical Systems via Approximating Koopman Spectrum with Timescale Decoupling},
  author={Zhang, Shimin and Ye, Ziyuan and Yan, Yinsong and Song, Zeyang and Wu, Yujie and Wu, Jibin},
  year={2025},
  booktitle={Forty-second International Conference on Machine Learning}
}

@article{zhou2023comparing,
  title={Comparing neural models using their perceptual discriminability predictions},
  author={Zhou, Jing Yang and Chun, Chanwoo and Subramanian, Ajay and Simoncelli, Eero P},
  journal={bioRxiv},
  pages={2023--11},
  year={2023},
  publisher={Cold Spring Harbor Laboratory}
}
\end{refsegment}

\clearpage

\end{document}